%% file: main.tex
\documentclass[11pt,twoside]{article}
\input{setting.tex}
\DeclareMathSizes{9}{8}{7}{5}
\title{\textbf{When Does Gradient Descent with Logistic Loss \\
 Interpolate Using Deep Networks \\
 with Smoothed ReLU Activations? }}
\author{Niladri S. Chatterji \\ University of California, Berkeley \\ \textsf{chatterji@berkeley.edu} \and Philip M. Long \\ Google \\ \textsf{plong@google.com}  \and Peter L. Bartlett \\ University of California, Berkeley \& Google \\ \textsf{peter@berkeley.edu}}
\date{\today}
\begin{document}
\maketitle
\begin{abstract}
We establish conditions under which gradient descent applied to fixed-width deep networks drives the logistic loss to zero, and prove bounds on the rate of convergence. Our analysis applies for smoothed approximations to the ReLU, such as Swish and the Huberized ReLU, proposed in previous applied work. We provide two sufficient conditions for convergence.  
The first is simply a bound on the loss at initialization.  The second is a data separation condition used in prior analyses.
\end{abstract}

\section{Introduction}
Interest in 
the properties of interpolating deep learning models trained with first-order optimization methods is surging~\citep{zhang2016understanding,belkin2019reconciling}. One important question is to understand how gradient descent with appropriate random initialization routinely finds interpolating (near-zero training loss) solutions to these non-convex optimization problems. 

In this paper our focus is to understand when gradient descent drives the logistic loss to zero when
applied to fixed-width deep networks 
using
smooth approximations to the ReLU activation
function. We derive upper bounds on the rate of convergence 
under two 
conditions. 
The first result only requires that the initial loss is small, but does not require any assumption about the width of the network. It guarantees that if the initial loss is small then gradient descent drives the logistic loss down to zero. The second result is under a separation condition on the data. Under this assumption we demonstrate that the loss decreases adequately in the initial iterations such that the first result applies.

A few ideas that facilitate our analysis are as follows: under the first set of assumptions, when the loss is small, 
we show that the negative gradient aligns with the weights of the network. This lower bounds the norm of the gradient at the beginning of the gradient step and implies that the loss decreases quickly at the beginning of the step. 
We then show that the loss is smooth in the
neighborhood of the beginning of the step.
The smoothness of the loss combined with the lower bound on the norm of the gradient at the beginning of the step implies that the loss decreases throughout the gradient step when the step-size is small enough.

The second sufficient condition is when the data is separable by a margin using the features obtained by the gradient of the neural network at initialization (see Assumption~\ref{as:frac_margin}). This assumption has previously been studied by \citet{chen2021much}. 
Intuitively, it is weaker than an assumption that the training examples
are not too close, as we discuss after its definition.
Under this assumption we use a neural tangent kernel (NTK) analysis to show that the loss decreases sufficiently in the first stage of optimization such that we can invoke our first result to guarantee that the loss decreases thereafter in the second stage. To analyze this first stage we borrow ideas from \citep{allen2019convergence,zou2020gradient}, because the formulation of their results was most closely aligned with our needs. However we note that their results do not directly apply since they study networks with ReLU activations while we study smooth approximations to the ReLU. 
In addition to adapting their proofs to
our setting, we also
worked out some details in the original proofs.

Our first result could be viewed as a tool to establish convergence under a wide variety of conditions. Our second result is one example of how it may be applied. Other separation assumptions on the data like the ones studied by \citet{ji2019polylogarithmic,chen2021much,zou2020gradient}, could also be used in conjunction with our first result to establish convergence to zero training loss. 

Recently \citet{chatterji2020does} showed that gradient descent applied to two-layer neural networks drives the logistic loss to zero when the initial loss is small and the activation functions are Huberized ReLUs. Our work can be viewed as a generalization of their result to the case of deep networks. 

Previously, \citet{lyu2019gradient} studied the margin maximization of ReLU networks 
for the logistic loss. They also proved that gradient descent applied to deep networks drives the training logistic loss to zero. However, their result requires the neural network to be both positive homogeneous and smooth (see, for example, the proof of Lemma E.7 of their
paper), so that a substantially different analysis
was needed here. Their assumptions rule out the ReLU and close approximations to it like Swish~\citep{ramachandran2017searching} or the Huberized ReLU~\citep{tatro2020optimizing} that are widely used in practice. Their results do apply in case that the ReLU is raised to a power strictly greater than two. As far as we know, the analysis of the alignment
between the negative gradient and the weights
originated in their paper: in this paper, we establish such
alignment under weaker conditions.

Prior work has shown that gradient descent drives the squared loss of fixed-width deep networks to zero \citep{du2018gradient,du2019gradient,allen2019convergence,oymak2020towards}, using the NTK perspective
\citep{jacot2018neural,chizat2019lazy}. The logistic loss however is qualitatively different. Driving the logistic loss to zero requires the weights to 
go to infinity, far from their initial values. This means that a
Taylor approximation around the initial values cannot be applied. While the NTK framework has also been applied to analyze training with the
logistic loss, a typical result \citep{li2018learning,allen2019convergence,zou2020gradient}
is that after $\mathsf{poly}(1/\epsilon)$ 
updates, a network of size or width $\mathsf{poly}(1/\epsilon)$ 
achieves $\epsilon$ loss. Thus to guarantee loss very close to zero, these
analyses require larger and larger networks. The reason for this
appears to be that a key part of these analyses
is to show that a wider network can achieve a certain fixed loss by traveling
a shorter distance in parameter space.  Since, to drive the
logistic loss to zero with a fixed-width network, the parameters must travel 
an unbounded distance, it seems that the NTK approach cannot be applied to obtain the
results of this paper.

The remainder of the paper is organized as follows. In Section~\ref{s:prelim} we introduce notation and 
definitions. In Section~\ref{s:main_results} we present our main theorems.
We provide a proof of our first result, Theorem~\ref{t:main_theorem}, in Section~\ref{s:main_theorem}. We conclude with a discussion in Section~\ref{s:discussion}. Appendix~\ref{a:furtherrelatedwork} points to other related work. The proof of our second result, Theorem~\ref{t:theorem_frac_margin}, and other technical details, are 
presented in the remaining appendices.

\section{Preliminaries} \label{s:prelim}
This section includes notational conventions and a description of
the setting.

\subsection{Notation}
Given a vector $v$, let $\lv v \rv$ denote its Euclidean norm, $\lv v \rv_{p}$ denote its $\ell_{p}$-norm for any $p\ge 1$, $\lv v \rv_{0}$ denote the number of non-zero entries, and $\diag(v)$ denote a diagonal matrix with $v$ along the diagonal. We say a vector $v$ is $k$-sparse if $\lv v\rv_{0}\le k$. Given a matrix $M$, let $\lv M \rv$ denote its Frobenius norm, $\lv M \rv_{op}$ denote its operator norm and $\lv M\rv_{0}$ denote the number of non-zero entries in the matrix. Given either a matrix or a tensor we let $\vecrm(\cdot)$ be 
its 
vectorization. 
Given a tensor $T$, let $\lv T \rv = \lv \vecrm(T) \rv$; we
will sometimes call this the Frobenius norm of $T$. If, for matrices
$T_1,\ldots,T_{L+1}$ of different
shapes, we refer to them collectively
as $T$, we define 
$\lv T \rv$ analogously.
Given two tensors $A$ and $B$ let $A \cdot B$ denote the element-wise dot product $\vecrm(A)\cdot \vecrm(B)$. For any $k \in \N$, we denote the set $\{ 1,\ldots,k \}$ by $[k]$.  For a number $p$ of inputs, we denote
the set of unit-length vectors
in $\R^p$ by $\S^{p-1}$. We use the standard ``big Oh notation'' \citep[see, e.g.,][]{cormen2009introduction}. We will use 
$c,  c', c_1, \ldots$ to denote 
constants, which may take different values in
different contexts.

For a function $J$ of a tensor $V$, we denote
the gradient of $J$ at $V$ by $\nabla_V J(V)$,
and define $\Lip(\nabla_V J(V))$ to be the local
Lipschitz constant of $\nabla_V J(V)$, as
a function of $V$, with respect to the Euclidean
norm. That is
$$
\Lip(\nabla_V J(V)) = \limsup_{W \rightarrow V} 
     \frac{\lv \nabla_V J(V) - \nabla_W J(W) \rv}{\lv V - W \rv}.
$$

\subsection{The Setting}

We will analyze gradient descent applied
to minimize the training loss of a
multi-layer network.  

We assume that the number of inputs is equal to the number of hidden nodes per layer to simplify the presentation of our results. Our techniques can easily extend to the case where there are different numbers of hidden nodes in different layers. Let $p$ denote the number of inputs and the number of hidden nodes per layer, and let $L$ denote the number of hidden layers. 

We will denote the activation function by $\phi$. Given a vector $v$ let $\phi(v)$ denote a vector with the activation function applied to each coordinate. We study activation functions that are similar to the ReLU activation function but are smooth.
\begin{definition}\label{def:h_smooth_activation} A activation function $\phi$ is $h$-smoothly approximately ReLU if,
\begin{itemize}
    \item the function $\phi$ is differentiable;
    \item $\phi(0)= 0$;
    \item $\phi'$ is $\frac{1}{h}$-Lipschitz and $|\phi'(z)|\le 1$;
    \item for all  $z \in \R: |\phi'(z)z - \phi(z)| \le h/2$.
\end{itemize}
\end{definition}
It may aid intuition to note that,
for small $h$, the condition that 
$|\phi'(z)z - \phi(z)| \le h/2$ can be paraphrased
to say that a first-order Taylor approximation of
$\phi$ at $z$ is accurate at the origin.  
It is easy to verify the activation functions $\phi$ are contractive with respect to the Euclidean norm. That is, for any $v_1,v_2 \in \R^p$, $\lv \phi(v_1) - \phi(v_2)\rv \le \lv v_1 - v_2 \rv$. See Lemma~\ref{l:aux_contractive_maps}. 
Here are a couple of examples of activation functions that are $h$-smoothly approximately ReLU.
\begin{enumerate}
    \item Huberized ReLU \citep{tatro2020optimizing}: \begin{equation}
\label{e:helu}
\phi(z) := \left\{ \begin{array}{ll}
            0 & \mbox{ if $z < 0$,} \\
            \frac{z^2}{2 h} & \mbox{ if $z \in [0,h]$,} \\
            z - \frac{h}{2} & \mbox{otherwise.}
         \end{array}
           \right.
\end{equation}
\item Scaled Swish \citep{ramachandran2017searching}: $\phi(z) = \frac{z}{1.1\left(1+\exp\left(-2z/h\right)\right)}$. The scaling factor $1/1.1$ ensures that $|\phi'(z)| \le 1$.
\end{enumerate}


For $i \in \{1,\ldots, L\}$, let $V_i \in \mathbb{R}^{p \times p}$ be the weight matrix of the $i$th layer and let $V_{L+1} \in \mathbb{R}^{1 \times p }$ be the weight vector corresponding to the outer layer. Let $V = (V_1,\ldots,V_{L+1})$
consist of all of the trainable parameters in the network.
Let $f_{V}$ denote the function computed by the 
network, which maps $x$ to
$$
f_{V}(x) = V_{L+1}\phi\left(V_L \cdots \phi(V_1 x) \right).
$$
Consider a training set $(x_1,y_1),\ldots,(x_n,y_n) \in \S^{p-1} \times \{-1, 1 \}$. For any sample $s \in [n]$, define $u_{0,s}^V = x_{0,s}^V := x_s$ and for all $\ell \in [L]$, define
\begin{align*}
   u_{\ell,s}^{V} &:= V_{\ell} x_{\ell-1,s}^V \quad \text{and, } \quad x_{\ell,s}^{V} := \phi\left(V_{\ell} x_{\ell-1,s}^V\right),
\end{align*}
that is, $u_{\ell,s}^{V}$ refers to the pre-activation features in layer $\ell$, while $x_{\ell,s}^V$ corresponds to the features after applying the activation function in the $\ell$th layer. Also for any $\ell \in [L]$ and $s \in [n]$ let
\begin{align*}
    \Sigma_{\ell,s}^{V} := \diag\left(\phi'(u_{\ell,s}) \right) = \diag\left(\phi'\left(V_{\ell}x_{\ell-1,s}^V\right)\right).
\end{align*}

Define the training loss (empirical risk with respect to the logistic loss) $J$ by
\[
J(V) := \frac{1}{n}\sum_{s=1}^n \ln(1+\exp\left(-y_s f_{V}(x_s)\right)),
\]
and refer to loss on example $s$ by
$$
J(V; x_s,y_s) := \ln(1+\exp\left(-y_s f_{V}(x_s)\right)).
$$
The gradient of the loss evaluated at $V$ is
$$\nabla_{V} J(V) = \frac{1}{n} \sum_{s=1}^n \frac{-y_s \nabla_{V} f_{V}(x_s)}{1+\exp\left(y_s f_{V}(x_s)\right)},
$$
and the partial gradient of $f_V$ with respect to $V_{\ell}$ has the form \citep[see, e.g.,][]{zou2020gradient}
\begin{align}\startsubequation\subequationlabel{e:gradient_definitions}\tagsubequation\label{e:gradient_ell_inner_layers}
 \frac{\partial f_V(x_s)}{\partial V_\ell} &= 
        \left( \Sigma^V_{\ell,s} \prod_{j = \ell+1}^L 
             \left(  V_{j}^{\top} \Sigma^V_{j,s} \right)
           \right)
           V_{L+1}^{\top}
        x_{\ell-1,s}^{V\top}, \quad \mbox{when $\ell \in [L]$}, \\ \tagsubequation \label{e:gradient_outer_layer}
           \frac{\partial f_V(x_s)}{\partial V_{L+1}} & =  x_{L,s}^{V\top}.
\end{align}

We analyze the iterates of gradient descent $V^{(1)},V^{(2)},\ldots$ defined by
\[
    V^{(t+1)} := V^{(t)} - \alpha_t \nabla_V J\lvert_{V = V^{(t)}}
\]
in terms of the properties of $V^{(1)}$.
\begin{definition}
\label{d:L_tk}
For all 
iterates $t$,
define
$J_{ts} := J(V^{(t)}; x_s,y_s)$ and let $J_t := \frac{1}{n} \sum_{s=1}^n J_{ts}$. 
Additionally for all 
$t$,
define $\nabla J_t : = \nabla_{V} J|_{V = V^{(t)}}$.
\end{definition}

\section{Main Results}\label{s:main_results}
In this section we present our theorems and discuss their implications.
\subsection{A General Result} \label{ss:general_result}
Given the initial weight matrix $V^{(1)}$, width $p$, depth $L$, and training data $\{x_s,y_s\}_{s\in [n]}$, define $h_{\max}$, $\alpha_{\max}$ and $\widetilde{Q}$ below:
\begin{align}\startsubequation\subequationlabel{e:main_theorem_quantities}
\tagsubequation \label{e:h_max_def}       h_{\max} &:= \min\left\{\frac{L^{\frac{L}{2}-3}\log(1/J_1)}{24\sqrt{p}\lv V^{(1)}\rv^{L}},1\right\},\\
\tagsubequation\label{e:alpha_max_def} \alpha_{\max}(h)  &:= \min\left\{\frac{h}{1024\left( L+1\right)^2pJ_1\lv V^{(1)}\rv^{3L+5}}, \frac{(L+\frac{1}{2})\lv V^{(1)} \rv^2}{2L(L+\frac{3}{4})^2J_1  \log^{\frac{2}{L}}(1/J_1)}\right\}, \text{ and} 
 \\ \tagsubequation \label{eq:Q_def} \widetilde{Q}(\alpha)  &:= \frac{L(L+\frac{3}{4})^2\alpha J_1 \log^{\frac{2}{L}}(1/J_1)}{(L+\frac{1}{2})\lv V^{(1)} \rv^2}.
\end{align}
\begin{theorem} \label{t:main_theorem}
For any $L \ge 1$, for all $n \geq 3$, 
for all 
$p \geq 1$,
for any initial parameters $V^{(1)}$
and
dataset $(x_1,y_1),\ldots,(x_n,y_n) \in \S^{p-1} \times \{-1,1\}$, for any $h$-smoothly approximately ReLU activation function with $h<h_{\max}$, any positive $\alpha \le \alpha_{\max}(h)$ and positive $Q \le \widetilde{Q}(\alpha)$ the
following holds for all
$t \geq 1$.
If each step-size
$\alpha_t = \alpha $,
and if $J_1< 1/n^{1 + 24L}$
then, for all $t \geq 1$,
        \[
J_t
 \le 
    \frac{J_1}{ Q \cdot (t-1)+1}.
\]
\end{theorem}
We reiterate that this theorem makes no assumption about the 
width $p$ of the network 
and makes a very mild assumption on the number of samples required: $n\ge 3$. The only
other assumption is that the initial loss is less than $1/n^{1+24L}$. We pick the step-size to be a constant, which leads to a rate that scales with $1/t$. 

Next we provide an example where we show that it is possible to arrive at a small loss solution using gradient descent starting from randomly initialized weight matrices.

\subsection{Small Loss Guarantees Using NTK Techniques} \label{ss:ntk_init}

In this subsection assume that the entries of the initial weight matrices for the layers $\ell \in \{1,\ldots,L\}$ are drawn independently from $\cN\left(0,
2/p\right)$, and the entries of $V_{L+1}^{(1)}$ are drawn independently from $\cN\left(0,1\right)$. In this section we also specialize to the case where the activation function is the Huberized ReLU (see its definition in equation~\eqref{e:helu}). We make the following assumption on the training data.
\begin{restatable}{assumption}{fracmargin}
\label{as:frac_margin}With probability $1-\delta$ over the random initialization, there exists a collection of matrices $W^{\star}=(W_1^{\star},\ldots,W_{L+1}^{\star})$ with $\lv W^{\star} \rv = 1$, such that for all samples $s \in [n]$
\begin{align*}
    y_s \left(\nabla f_{V^{(1)}}(x_s)\cdot W^{\star}\right)\ge \sqrt{p}\gamma,
\end{align*}
for some $\gamma >0$.
\end{restatable}
The scaling factor $\sqrt{p}$ on the right hand side is to balance the scale of the norm of the gradient at initialization which will scale with $\sqrt{p}$ as well. This is because the entries of the final layer $V_{L+1}^{(1)}$ are drawn independently from $\cN(0,1)$. This assumption is inspired by Assumption~4.1 made by \citet{chen2021much}. 
This assumption can be seen to be implied by stronger conditions 
that simply require that the training examples are not too close,
as employed in \citep{allen2019convergence,zou2020gradient}.
Here is some rough intuition of
why.  
The components of $\nabla f_{V^{(1)}}(x_s)$ include 
values computed at the
last hidden layer when $x_s$ is processed using $V^{(1)}$ (that is, $\nabla_{V_{L+1}}f_{V^{(1)}}(x_s)=x_{L,s}^{V^{(1)}}$). For wide networks with Huberized ReLU activations,
if the values of $x_s$ in the training examples do not have duplicates,
their embeddings
into the last hidden layer of nodes 
are in 
general position with high probability.
In fact, the Gaussian Process analysis of infinitely wide deep
networks at initialization 
\citep{g.2018gaussian,matthews2018gaussian}
suggests that, for wide networks, the embeddings
will not even be close to 
failing to be in general position~\citep[see][]{DBLP:conf/alt/AgarwalAK21}.
If the width $p \gg n$, results from 
\citep{cover1965geometrical} show that they
will be linearly separable.
The anti-concentration conferred by the Gaussian initialization promotes
larger (though not necessarily constant) margins.
Assumption~\ref{as:frac_margin} is
more refined than a separation condition, since it captures a sense in which
the data is amenable to treatment with neural networks that enables us to
provide stronger guarantees in such cases. Furthermore, in Appendix~\ref{a:constant_gamma} 
we show that Assumption~\ref{as:frac_margin}
is satisfied with a constant margin $\gamma$ by
two-layer networks with Huberized ReLUs
for data satisfying a clustering condition. Finally, we note that we could also use other assumptions on the data that have been studied in the literature \citep[for example by,][]{ji2019polylogarithmic} to guarantee that the loss reduces below $1/n^{1+24L}$, as required to invoke Theorem~\ref{t:main_theorem}. However, we provide guarantees only under this assumption in the interest of simplicity.

Define 
\begin{align} \label{e:definition_rho_radius}
\rho := \frac{c_1}{\sqrt{p}\gamma} \mleft[\sqrt{\log\left(\frac{n}{\delta}\right)}+\log\mleft(6n^{(2+24L)}\mright)\mright],
\end{align}
where $c_1\ge 0$ is a large enough absolute constant. Also set the value of 
\begin{align} \label{e:huberized_relu_h_def}
    h=h_{\mathsf{NT}}:=  \frac{(1+24L)\log(n)}{6(6p)^{\frac{L+1}{2}}L^{3}}.
\end{align}
With these choices of $\rho$ and $h$ we are now ready to state our convergence result under Assumption~\ref{as:frac_margin}. The proof of this theorem is presented in Appendix~\ref{a:proof_of_ntk}. 
\begin{restatable}{theorem}{ntktheorem} \label{t:theorem_frac_margin} Consider a network with Huberized ReLU activations. There exists $r(n,L,\delta) = \poly\left(L,\log\left(\frac{n}{\delta}\right)\right)$ such that for any $L \ge 1$, $n\ge 3$, $\delta >0$, under Assumption~\ref{as:frac_margin} with $\gamma \in (0,1]$ if $h = h_{\mathsf{NT}}$ and $p\ge \frac{r(n,L,\delta)}{\gamma^2}$ then both of the following hold with probability at least $1-4\delta$ over the random initialization:
\begin{enumerate}
    \item For all $t\in  [T]$, set the step-size $\alpha_t =\alpha_{\mathsf{NT}}= 
    \Theta\mleft(\frac{1}{pL^5}\mright)$, where $T = \ceil{\frac{3(L+1)\rho^2 n^{2+24L}}{2\alpha_{\mathsf{NT}}}}$. Then 
    \begin{align*}
        \min_{t \in [T]} J_t < \frac{1}{n^{1+24L}}.
    \end{align*}
    \item Set $V^{(T+1)} = V^{(s)}$, where $s \in \argmin_{s \in [T]} J(V^{(s)})$, and for all $t\ge T+1$, set the step-size $\alpha_t = 
    \alpha_{\max}(h)$. Then for all $t \ge T+1$,
    \begin{align*}
        J_t \le O\left(\frac{L^{\frac{3L+11}{2}}(6p)^{2L+5}}{n^{1+24L}\cdot(t-T-1)}\right).
    \end{align*}
\end{enumerate}
\end{restatable}
We invite the reader to interpret the result of this theorem in two scenarios. The first is where the depth $L$ is a constant and the margin $\gamma \ge \left(p^{\omega} \poly\left(n,\log\left(\frac{1}{\delta}\right)\right)\right)^{-1}$, for some constant $\omega \in [0,\frac{1}{2})$. 
In this case the conditions of Theorem~\ref{t:theorem_frac_margin} are
satisfied for $p = \poly\left(n,\log\left(\frac{1}{\delta}\right)\right)$,
and, for such $p$, the rate of convergence in the second stage is 
\begin{align*}
   J_t\le  O\left(\frac{L^{\frac{3L+11}{2}}(6p)^{2L+5}}{n^{1+24L}\cdot(t-T-1)}\right) \le \frac{\poly\left(n,\log\left(\frac{1}{\delta}\right)\right)}{t}.
\end{align*}
Another scenario is where the margin $\gamma$ is at least
a constant. Here it suffices for the width $p \ge \poly\left(L,\log\left(\frac{n}{\delta}\right)\right)$. Thus if the number of samples  $n \ge \left[L^{\frac{3L+11}{2}}(6p)^{2L+5}\right]^{\frac{1}{1+24L}}$ then the rate of convergence in this second stage is
\begin{align*}
    J_t\le O\left(\frac{L^{\frac{3L+11}{2}}(6p)^{2L+5}}{n^{1+24L}\cdot(t-T-1)}\right) = O\left(\frac{1}{t-T-1}\right).
\end{align*}

\section{Proof of Theorem~\ref{t:main_theorem}}
\label{s:main_theorem}

In this section, we prove Theorem~\ref{t:main_theorem}. 

\subsection{Technical Tools} \label{ss:technical_tools}
In this subsection we assemble several technical tools required to prove Theorem~\ref{t:main_theorem}. 
Their proofs (which in turn depend on additional, more basic, lemmas)
can be found in Appendix~\ref{a:omitted_proofs_technical_tools}.

We start with the following lemma, which is a slight variant of a standard inequality, and provides a bound on the loss after a step of gradient descent when the loss function is locally smooth.
\begin{restatable}{lem}{pllemma}
\label{l:pl}
For $\alpha > 0$, let $V^{(t+1)} = V^{(t)} - \alpha \nabla J_t$.
If,
for all convex combinations $W$ of $V^{(t)}$
and $V^{(t+1)}$, 
we have
$\Lip(\nabla_W J(W)) \leq M$,
then
if $\alpha \leq \frac{1}{\left(L+\frac{1}{2}\right)M}$, we have
\[
J_{t+1} \leq J_t - \frac{\alpha L\lv \nabla J_t \rv^2}{L+\frac{1}{2}}.
\]
\end{restatable}
To apply Lemma~\ref{l:pl} we need to show that the loss $J$ is smooth near $J_t$. 
The following lemma establishes the smoothness of $J$, if the weights
are large enough.  
(We will be able to apply it, since the
weights must be fairly large to achieve small loss.)
\begin{restatable}{lem}{hessianlemma}
\label{l:L_smooth} 
If $h\le 1$, for any 
weights
$V$
such that
$\lv V\rv\ge \sqrt{L+1/2}$,
we have
$$
\Lip(\nabla_V J(V))
 \le
 \frac{256(L+1)\sqrt{p}\lv V \rv^{3L+5}J(V)}{h}.
$$
\end{restatable}

Next, we show that $J$ 
changes slowly in general, and especially slowly when it is
small. 
\begin{restatable}{lem}{gradientupper}
\label{l:gradient_norm.upper} For any weight matrix $V$
such that
$\lv V\rv\ge \sqrt{L+1/2}$ 
then $$\lv \nabla_{V} J(V) \rv \le  \sqrt{(L+1)p}\lv V \rv^{L+1}\min\{ J(V), 1\}.$$
\end{restatable}
The following lemma applies Lemma~\ref{l:pl} (along with Lemma~\ref{l:L_smooth}) to show that if the step-size at step $t$ is small enough then the loss decreases by an amount that is proportional to the squared norm of the gradient. 
\begin{restatable}{lem}{onesteplemma}
\label{l:L.one_step_improvement}
If $h\le 1$, $J_t <\frac{1}{n^{1+24L}}$, and $$\alpha J_t  \leq \frac{h}{1024\left( L+1\right)^2\sqrt{p}\lv V^{(t)}\rv^{3L+5}},$$ then
$$
J_{t+1} \leq J_t-\frac{\alpha L \lv \nabla J_t \rv^2}{L+\frac{1}{2}}.
$$
\end{restatable}

The next lemma establishes a lower bound on the norm of the gradient at any iteration in terms of the loss $J_t$ and the norm of the weight matrix $V^{(t)}$. 
\begin{restatable}{lem}{gradientlowerboundlemma}
\label{l:lower.bound.gradient} 
For all $L \in \N$ if $h \le h_{\max}$, $J_t < \frac{1}{n^{1+24L}}$, and $\lv V^{(t)}\rv^L \le \log(1/J_t) \frac{\lv V^{(1)}\rv^L}{\log(1/J_1)}$
then 
\begin{align} \label{e:lowerboundonthegradient_used}
    \lv \nabla J_t \rv \ge  \frac{(L+\frac{3}{4})J_t \log(1/J_t)}{ \lv V^{(t)} \rv}.
\end{align}
\end{restatable}
The lower bound on the gradient 
is proved
by showing that the alignment between the negative gradient $-\nabla J_t$ and 
$V^{(t)}$
is large when the loss is small. The proof proceeds by showing that when $h$ is sufficiently small and the norm of $V^{(t)}$ is not too large, then the inner product between $-\nabla J_t$ and $V^{(t)}$ can be lower bounded by a function of the loss $J_t$.

\subsection{The Proof}
\label{ss:the_proof}
As stated above, the proof goes through for any positive $h \le h_{\max}$,
step-size $\alpha \le \alpha_{\max}(h)$ and any $Q \le \widetilde{Q}(\alpha)$ (recall the definitions of $h_{\max},\alpha_{\max}$ and $\widetilde{Q}$ in equations~\eqref{e:h_max_def}-\eqref{eq:Q_def}). We will use the following multi-part inductive hypothesis:
\begin{enumerate}[({I}1)]
    \item $J_{t} \le \frac{J_1}{Q\cdot (t-1)+1};$
    \item $\frac{\log(1/J_{t})}{\lv V^{(t)}\rv^{L}} \ge 
           \frac{\log(1/J_1)}{\lv V^{(1)}\rv^L}$;
               \item $\alpha J_t \le \frac{h}{1024\left( L+1\right)^2p\lv V^{(t)}\rv^{3L+5}}$.
\end{enumerate}
The first part of the inductive hypothesis will be used to ensure that the loss decreases at the prescribed rate, the second part helps establish a lower bound on the norm of the gradient in light of Lemma~\ref{l:lower.bound.gradient} and the third part will ensure that the step-size is small enough to apply Lemma~\ref{l:L.one_step_improvement} and also allows us to make several useful approximations in our proofs.

The base case is trivially true for the first and second part of the inductive hypothesis. It is true for the third part 
since
the step-size $\alpha \le \alpha_{\max}(h)\le \frac{h}{1024\left( L+1\right)^2pJ_1\lv V^{(1)}\rv^{3L+5}}$. 
Now let us assume that the inductive hypothesis holds for a step $t \ge 1$ and
prove that it holds for the next step $t+1$.  We start with Part~I1.
\begin{lemma}
\label{l:L.inductive_step_constant_step_size} If the inductive hypothesis holds at step $t$, then
\[
J_{t+1} \leq 
\frac{J_1}{Qt + 1}.
\]
\end{lemma}
\begin{proof}
Since $\alpha J_t\le \frac{h}{1024\left( L+1\right)^2p\lv V^{(t)}\rv^{3L+5}}$ and $J_t \le J_1 < \frac{1}{n^{1+24L}}$, by invoking Lemma~\ref{l:L.one_step_improvement}, 
\begin{align*}
    J_{t+1} \le J_t - \frac{L\alpha}{(L+\frac{1}{2})} \lv \nabla J_t\rv^2.
\end{align*}
Additionally since $h \le h_{\max}$ and by Part~I2 of the inductive hypothesis $\lv V^{(t)}\rv^{L}\le \frac{\log(1/J_t) \lv V^{(1)}\rv^{L}}{\log(1/J_1)}$, we use the lower bound on the norm of the gradient established in Lemma~\ref{l:lower.bound.gradient} to get
\begin{align*}
    J_{t+1} &\le J_t - \frac{L(L+\frac{3}{4})^2 \alpha J_t^2 \log^2(1/J_t)}{(L+\frac{1}{2})\lv V^{(t)}\rv^2} \\
    &\overset{(i)}{\le} J_t\left(1 - \frac{L(L+\frac{3}{4})^2\alpha J_t \log^{2-\frac{2}{L}}(1/J_t)\log^{\frac{2}{L}}(1/J_1)}{(L+\frac{1}{2})\lv V^{(1)} \rv^2}\right)\\
    &\overset{(ii)}{\le} J_t\left(1 - \frac{L(L+\frac{3}{4})^2\alpha J_t \log^{\frac{2}{L}}(1/J_1)}{(L+\frac{1}{2})\lv V^{(1)} \rv^2}\right)\numberthis \label{e:lossdecreasemasterequation},
\end{align*}
where $(i)$ follows by Part~(I2) of the inductive hypothesis,
and $(ii)$ follows since $L\ge 1$ and $J_t \le J_1 < \frac{1}{n^{1+24L}}$, therefore $\log^{2-\frac{2}{L}}(1/J_t)\ge1$.

For any $z\ge 0$, the quadratic function 
$$z - z^2\frac{L(L+\frac{3}{4})^2\alpha  \log^{\frac{2}{L}}(1/J_1)}{(L+\frac{1}{2})\lv V^{(1)} \rv^2}$$
 is a monotonically increasing function in the interval
$$\left[0,\frac{(L+\frac{1}{2})\lv V^{(1)} \rv^2}{2L(L+\frac{3}{4})^2\alpha  \log^{\frac{2}{L}}(1/J_1)}\right].$$
Thus, because 
$J_t \leq \frac{J_1}{Q(t-1)+1}$, if $ \frac{J_1}{Q(t-1)+1} \leq \frac{(L+\frac{1}{2})\lv V^{(1)} \rv^2}{2L(L+\frac{3}{4})^2\alpha  \log^{\frac{2}{L}}(1/J_1)}$, the RHS of 
\eqref{e:lossdecreasemasterequation}
is bounded above by its value 
when $J_t = \frac{J_1}{Q(t-1)+1}$. But this is easy to check: by our choice of step-size $\alpha$ we have,
\begin{align*}
     &\alpha \le \alpha_{\max} \le  \frac{(L+\frac{1}{2})\lv V^{(1)} \rv^2}{2L(L+\frac{3}{4})^2J_1  \log^{\frac{2}{L}}(1/J_1)}\\
    & \Rightarrow J_1 \le \frac{(L+\frac{1}{2})\lv V^{(1)} \rv^2}{2L(L+\frac{3}{4})^2\alpha  \log^{\frac{2}{L}}(1/J_1)} \\
    &\Rightarrow \frac{J_1}{Q(t-1)+1} \le \frac{(L+\frac{1}{2})\lv V^{(1)} \rv^2}{2L(L+\frac{3}{4})^2\alpha  \log^{\frac{2}{L}}(1/J_1)} . 
\end{align*}
Bounding the RHS of inequality~\eqref{e:lossdecreasemasterequation}
by using the worst case that $J_t = \frac{J_1}{Q(t-1)+1}$, we get that
 \begin{align*}
     J_{t+1} & \le \frac{J_1}{Q(t-1)+1}\left(1 - \frac{J_1}{Q(t-1)+1}\frac{L(L+\frac{3}{4})^2\alpha  \log^{\frac{2}{L}}(1/J_1)}{(L+\frac{1}{2})\lv V^{(1)} \rv^2}\right)\\
     & = \frac{J_1}{Qt+1}\left(1+\frac{Q}{Q(t-1)+1}\right)\left(1 - \frac{Q}{Q(t-1)+1}\frac{L(L+\frac{3}{4})^2\alpha J_1 \log^{\frac{2}{L}}(1/J_1)}{Q(L+\frac{1}{2})\lv V^{(1)} \rv^2}\right)\\
     & \le \frac{J_1}{Qt+1}\left(1-\left(\frac{Q}{Q(t-1)+1}\right)^2\right) \hspace{0.6in} \left(\mbox{since $Q \le\widetilde{Q}(\alpha)= \frac{L(L+\frac{3}{4})^2\alpha J_1 \log^{\frac{2}{L}}(1/J_1)}{(L+\frac{1}{2})\lv V^{(1)} \rv^2}$}\right) \\
     & \le \frac{J_1}{Q t+1}.
 \end{align*}
  This establishes the desired upper bound on the loss at step $t+1$.
\end{proof}
In the next lemma we shall establish that the second part of the inductive hypothesis holds.
\begin{lemma} Under the setting of Theorem~\ref{t:main_theorem}, if the induction hypothesis holds at step $t$ then,
\label{l:vbound}
\begin{align*}
    \frac{\log\left(\frac{1}{J_{t+1}}\right)}{\lv V^{(t+1)}\rv^{L}} \ge \frac{\log\left(\frac{1}{J_{1}}\right)}{\lv V^{(1)}\rv^{L}}.
\end{align*}
\end{lemma}
\begin{proof}
We know from Lemma~\ref{l:L.one_step_improvement} that $$J_{t+1} \le J_t\left(1 - \frac{L\alpha \lv \nabla J_t \rv^2}{(L+\frac{1}{2})J_t}\right),$$ and by the triangle inequality $$\lv V^{(t+1)} \rv \le \lv V^{(t)} \rv + \alpha \lv \nabla J_t \rv,$$ hence
\begin{align*}
    \frac{\log\left(\frac{1}{J_{t+1}}\right)}{\lv V^{(t+1)}\rv^{L}}  \ge  \frac{\log\left(\frac{1}{J_t\left(1 - \frac{L\alpha}{(L+\frac{1}{2})J_t}\lv \nabla J_t \rv^2 \right)}\right)}{\left(\lv V^{(t)}\rv + \alpha \lv \nabla J_t\rv\right)^{L}} 
    & =  \frac{\log\left(\frac{1}{J_t}\right)+\log\left(\frac{1}{\left(1 - \frac{L\alpha}{(L+\frac{1}{2})J_t}\lv \nabla J_t \rv^2 \right)}\right)}{\left(\lv V^{(t)}\rv + \alpha \lv \nabla J_t\rv\right)^{L}} \\
    & =  \frac{\log\left(\frac{1}{J_t}\right)\left(1-\frac{\log\left(1 - \frac{L\alpha}{(L+\frac{1}{2})J_t}\lv \nabla J_t \rv^2 \right)}{\log\left(\frac{1}{J_t}\right)}\right)}{\lv V^{(t)}\rv^L\left(1 + \frac{\alpha \lv \nabla J_t\rv}{\lv V^{(t)}\rv}\right)^{L}} \\
    & \overset{(i)}{\ge} \frac{\log\left(\frac{1}{J_t}\right)}{\lv V^{(t)}\rv^L} 
    \left\{
    \frac{\left(1+\frac{L\alpha\lv \nabla J_t \rv^2}{(L+\frac{1}{2})J_t\log\left(\frac{1}{J_t}\right)}\right)}{\left(1 + \frac{\alpha \lv \nabla J_t\rv}{\lv V^{(t)}\rv}\right)^{L}}
    \right\} 
    \numberthis \label{eq:lower_bound_ratio}
\end{align*}
where $(i)$ follows since $\log(1-z) \le -z$ for all 
$z \in (0,1)$ and because
\begin{align*}
 & \frac{L\alpha}{(L+\frac{1}{2})J_t}\lv \nabla J_t \rv^2 \\
& \leq \frac{L\alpha}{(L+\frac{1}{2})}\left[(L+1)pJ_t\lv V^{(t)} \rv^{2(L+1)} \right]
  \hspace{0.2in}
  \mbox{(by Lemma~\ref{l:gradient_norm.upper})} \\
  & = \alpha J_t \left[\frac{L(L+1)p\lv V^{(t)} \rv^{2(L+1)}}{L+\frac{1}{2}} \right]\\
  & < \alpha J_t \left[ \frac{1024\left( L+1\right)^2p \lv V^{(t)}\rv^{3L+5}}{h} \right] \hspace{0.2in} \mbox{($\lv V^{(t)}\rv >1$ by Lemma~\ref{l:aux.lower.bound.norm.weight.vector}, and $h\le 1$)}\\
& \le 1 \hspace{1in}
  \mbox{(by Part~I3 of the IH)}.
\end{align*}

  We want to show that the term in curly brackets in inequality~\eqref{eq:lower_bound_ratio} is at least 1, that is, we want 
\begin{align}
    & 1+\frac{L\alpha\lv \nabla J_t \rv^2}{(L+\frac{1}{2})J_t\log\left(\frac{1}{J_t}\right)} \ge\left(1 + \frac{\alpha \lv \nabla J_t\rv}{\lv V^{(t)}\rv}\right)^{L}. \label{e:sufficient_condition_curly}
    \end{align}
    We will first show that this inequality holds in the case where $L>1$. To show this, note that 
    \begin{align*}
        \frac{\alpha \lv \nabla J_t \rv}{\lv V^{(t)}\rv} &\le 
        \alpha J_t\sqrt{(L+1)p} \lv V^{(t)}\rv^{L} \hspace{0.4in} \mbox{(by Lemma~\ref{l:gradient_norm.upper})} \\
        & < \frac{1}{L-1}\cdot\alpha J_t \left[ \frac{1024\left( L+1\right)^2p \lv V^{(t)}\rv^{3L+5}}{h} \right] \hspace{0.2in} \mbox{(since $\lv V^{(t)}\rv >1$ by Lemma~\ref{l:aux.lower.bound.norm.weight.vector})} \\
        &\le \frac{1}{L-1} \hspace{0.4in} \mbox{(by Part~I3 of the IH)}.
    \end{align*}
    For any positive $z< \frac{1}{L-1}$ we have the inequality that $(1+z)^{L}\le 1+\frac{Lz}{1-(L-1)z}$, therefore to show that inequality~\eqref{e:sufficient_condition_curly} holds it instead suffices to show that
    \begin{align*}
    &  1+\frac{L\alpha\lv \nabla J_t \rv^2}{(L+\frac{1}{2})J_t\log\left(\frac{1}{J_t}\right)} \ge 1 + \frac{L\alpha \lv \nabla J_t\rv}{\lv V^{(t)}\rv\left(1-\frac{(L-1)\alpha \lv \nabla J_t\rv}{\lv V^{(t)}\rv}\right)}\\
    & \Leftrightarrow \frac{\lv \nabla J_t \rv}{(L+\frac{1}{2})J_t\log\left(\frac{1}{J_t}\right)} 
     \ge \frac{1}{\lv V^{(t)}\rv\left(1-\frac{(L-1)\alpha \lv \nabla J_t\rv}{\lv V^{(t)}\rv}\right)}\\
    & 
    \Leftrightarrow
     \lv \nabla J_t \rv  \ge  \frac{(L+\frac{1}{2})J_t \log\left(\frac{1}{J_t}\right)}{\lv V^{(t)} \rv}+\frac{(L-1)\alpha\lv \nabla J_t\rv^2}{\lv V^{(t)}\rv}\\
    & \Leftarrow \lv \nabla J_t \rv  \ge  \frac{(L+\frac{1}{2})J_t \log\left(\frac{1}{J_t}\right)}{\lv V^{(t)} \rv}+\frac{(L-1)\alpha(L+1)p \lv V^{(t)}\rv^{2L+2}J_t^2}{\lv V^{(t)}\rv}\hspace{0.2in}\mbox{(by Lemma~\ref{l:gradient_norm.upper})}\\
    & \Leftarrow \lv \nabla J_t \rv  \ge  \frac{\left(L+\frac{1}{2}+\alpha J_t\frac{1024(L^2-1)p\lv V^{(t)}\rv^{2L+2}}{h} \times \frac{h}{1024\log\left(\frac{1}{J_t}\right)}\right)J_t \log\left(\frac{1}{J_t}\right)}{\lv V^{(t)} \rv}\\
    %
    & \Leftarrow \lv \nabla J_t \rv  \ge  \frac{(L+\frac{3}{4})J_t \log\left(\frac{1}{J_t}\right)}{\lv V^{(t)} \rv} ,
\end{align*} 
%
%
where the last implication follows from
Part~I3 of the IH, the fact that $h \leq 1$ and because $J_t \leq J_1 \leq 1/n^{1 + 24L}$ and
$n \geq 3$.
Now this last inequality holds again because of Lemma~\ref{l:lower.bound.gradient} that guarantees that $\lv \nabla J_t\rv \ge \frac{(L+\frac{3}{4})J_t\log(1/J_t)}{\lv V^{(t)} \rv}$. A similar argument can also be used in the case where $L=1$, without the use of the inequality that was used to upper bound $(1+z)^L$. Thus we have proved that the term in the curly brackets in inequality~\eqref{eq:lower_bound_ratio} is at least $1$ and hence
\begin{align*}
      \frac{\log\left(\frac{1}{J_{t+1}}\right)}{\lv V^{(t+1)}\rv^{L}} \ge   \frac{\log\left(\frac{1}{J_{t}}\right)}{\lv V^{(t)}\rv^{L}} \ge \frac{\log\left(\frac{1}{J_{1}}\right)}{\lv V^{(1)}\rv^L}.
\end{align*}
This proves that the ratio is 
bounded below
at step $t+1$ by its initial value and establishes our claim.
\end{proof}
Finally we ensure that the third part of the inductive hypothesis holds. This allows us to apply Lemma~\ref{l:L.one_step_improvement} in the next step $t+1$.
\begin{lemma}
\label{l:verify.valid.step_size} Under the setting of Theorem~\ref{t:main_theorem}, if the induction hypothesis holds at step $t$, then
\begin{align*}
    \alpha J_{t+1} \le \frac{h}{1024\left( L+1\right)^2p\lv V^{(t+1)}\rv^{3L+5}}.
\end{align*}
\end{lemma}
\begin{proof} 
We know by Lemma~\ref{l:vbound} that $\lv V^{(t+1)}\rv^{L} \le \frac{\log(1/J_{t+1})\lv V^{(1)}\rv^L}{\log(1/J_1)}$ so it instead suffices to prove that
\begin{align} \label{e:induction_part_3_lhs_maximize}
    \alpha J_{t+1}\log^{\frac{3L+5}{L}}\left(\frac{1}{J_{t+1}}\right) &\le \frac{h \log^{\frac{3L+5}{L}}(1/J_1)}{1024\left( L+1\right)^2p\lv V^{(1)}\rv^{3L+5}}.
\end{align}
Lemma~\ref{l:L.inductive_step_constant_step_size} establishes that $J_{t+1}\le J_1 < 1/n^{1+24L}$. The function $z\log^{\frac{3L+5}{L}}(1/z)$ is increasing over the interval $(0,\frac{1}{e^{\frac{3L+5}{L}}})$. Recall that $n\ge3$ therefore, $$J_{t+1}\le J_1 < \frac{1}{3^{1+24L}} < \frac{1}{e^{\frac{3L+5}{L}}}.$$
Thus, the LHS of \eqref{e:induction_part_3_lhs_maximize} is maximized at $J_1$
\begin{align*}
    \alpha J_{t+1}\log^{\frac{3L+5}{L}}\left(\frac{1}{J_{t+1}}\right) &\le \alpha J_{1}\log^{\frac{3L+5}{L}}\left(\frac{1}{J_{1}}\right) \le \frac{h \log^{\frac{3L+5}{L}}(1/J_1)}{1024\left( L+1\right)^2\sqrt{p}\lv V^{(1)}\rv^{3L+5}}
\end{align*}
where final inequality holds by choice of the step-size $\alpha$. This completes the proof.
\end{proof}

Combining the results of Lemmas \ref{l:L.inductive_step_constant_step_size}, \ref{l:vbound} and \ref{l:verify.valid.step_size} completes the proof of theorem.

\section{Discussion} \label{s:discussion}

We have shown that deep networks with smoothed ReLU
activations trained by gradient descent with logistic loss
achieve training loss approaching zero if the loss is
initially small enough. We also established conditions under
which this happens that formalize the idea that the
NTK features are useful. Our analysis applies in the
case of networks using the increasingly popular Swish
activation function.

While, to simplify our treatment, we concentrated on the
case that the number of hidden nodes in each layer
is equal to the number of inputs, our analysis 
should easily be adapted to the case of varying
numbers of hidden units.

Analysis of architectures such as Residual Networks
and Transformers would be a potentially interesting
next step.  

\subsection*{Acknowledgements}
We thank the anonymous reviewers for alerting us to a mistake in an earlier version of this paper. 

We gratefully acknowledge the support of the NSF through grants DMS-2031883 and DMS-2023505 and the Simons Foundation through award 814639.

\appendix
\newpage
\tableofcontents

\section{Additional Related Work} \label{a:furtherrelatedwork} 
 Building on the work of \citet{lyu2019gradient}, \citet{ji2020directional} study finite-width deep ReLU neural networks and show that starting from a small loss, gradient flow coupled with logistic loss leads to convergence of the directions of the parameter vectors. They also demonstrate alignment between the parameter vector directions and the negative gradient. However, they do not prove that the training loss converges to zero.

Using mean-field techniques \citet{chizat2020implicit}, building on \citep{chizat2018global,mei2019mean},
show that infinitely wide two-layer squared ReLU networks trained with gradient flow on the logistic loss leads to a max-margin classifier in a particular non-Hilbertian space of functions.
See also the videos in a talk about this work \citep{chizat20msri}. \citet{CCGZ20} analyzed regularized training with gradient flow on infinitely wide networks.  
When training is regularized, the weights also may travel far from their initial values.
Previously \citet{brutzkus2018sgd} studied finite-width two-layer leaky ReLU networks and showed that when the data is linearly separable, these networks can be trained up to zero-loss using stochastic gradient descent with the hinge loss.

Our study is motivated in part by the line of work that has emerged which emphasizes the need to understand the behavior of interpolating (zero training loss/error) classifiers and regressors. A number of recent papers have analyzed the properties of interpolating methods in linear regression \citep{hastie2019surprises,bartlett2020benign,muthukumar2020harmless,tsigler2020benign,bartlett2020failures}, linear classification \citep{montanari2019generalization,chatterji2020finite,liang2020precise,muthukumar2020classification,hsu2020proliferation}, kernel regression \citep{liang2020just,mei2019generalization,liang2020MultipleDescent} and simplicial nearest neighbor methods \citep{belkin2018overfitting}.

There are also many related papers that characterize the implicit bias of the solution obtained by first-order methods \citep{neyshabur2017implicit,soudry2018implicit,ji2018risk,gunasekar2018characterizing,gunasekar2018implicit,li2018algorithmic,arora2019implicit,ji2019gradient}.

Finally, we note that a number of other recent papers also theoretically study the optimization of neural networks including
\citep{andoni2014learning,li2017convergence,ZS0BD17,ZLWJ17a,ge2018learning,ZPS18,du2018gradient,safran2018spurious,zhang2019learning,arora2019fine,brutzkus2019larger,wei2019regularization,ji2019polylogarithmic,nitanda2019gradient,song2019quadratic,zou2019improved,DBLP:conf/colt/BreslerN20,DBLP:conf/nips/Daniely20,DBLP:conf/nips/DanielyM20}.

\section{Omitted Proofs from Section~\ref{ss:technical_tools}} \label{a:omitted_proofs_technical_tools}

In this section we present the proofs of Lemmas~\ref{l:pl}-\ref{l:lower.bound.gradient}.
\subsection{Additional Definitions}
\begin{definition}
\label{d:g_tk} For any weight matrix $V$, define
$
    g_s(V):= \frac{1}{1+\exp\left(y_s f_{V}(x_s)\right)}.
$
We will often use $g_s$ as shorthand for $g_s(V)$ when
$V$ can be determined from context.
Further, for all $t \in \{0,1,\ldots\}$, define
$g_{ts} := g_s(V^{(t)}).$
\end{definition}
Informally, $g_s(V)$ is the size of the contribution
of example $s$ to the gradient.

\begin{definition}
\label{d:x_tk} For all iterates $t$, all $\ell \in [L+1]$ and all $s\in [n]$, define $x^{(t)}_{\ell,s} := x_{\ell,s}^{V^{(t)}}$, $u_{\ell,s}^{(t)} := u_{\ell,s}^{V^{(t)}}$ and $\Sigma_{\ell,s}^{(t)}:= \Sigma_{\ell,s}^{V^{(t)}}$.
\end{definition}

\subsection{Basic Lemmas}
\label{ss:basic}

To prove Lemmas~\ref{l:pl}-\ref{l:lower.bound.gradient}, we will need some more
basic lemmas, which we first prove.

\begin{lemma}\label{l:relationsbetweengradientandloss}
For any $x \in \mathbb{R}^p$ and $y \in \{-1,1\}$ and any weight matrix $V$ we have the following:
\begin{enumerate}
    \item \label{i:gradientupperboundedbyloss}$$\frac{1}{1+\exp\left(yf_V(x)\right)} \le  \ln(1+\exp(-yf_V(x))) = J(V;x,y).$$
    \item \label{i:hessiangupperboundedbyloss}$$\frac{\exp\left(yf_V(x)\right)}{\left(1+\exp\left(yf_V(x)\right)\right)^2} \le \frac{1}{1+\exp\left(yf_V(x)\right)} \le J(V;x,y).$$
\end{enumerate}
\end{lemma}
\begin{proof}
Part~$1$ follows since for any $z \in \mathbb{R}$, we have the inequality $(1+\exp(z))^{-1}\le \ln(1+\exp(-z))$. 

Part~$2$ follows since for any $z \in \mathbb{R}^d$, we have the inequality \[
\exp(z)/\left(1+\exp(z)\right)^2 \le \left( 1+\exp(z)\right)^{-1}.
\]
\end{proof}

The following lemma is useful for establishing a relatively
simple lower bound on a sum of applications of a concave function.
\begin{lemma} \label{l:concave.min}
If $\psi:[0,M] \to \mathbb{R}$ is a concave function
with $\psi(0) = 0$.  Then the minimum of
$\sum_{i=1}^n \psi(z_i)$ subject to
$z_1,\ldots,z_n \geq 0$ and
$\sum_{i=1}^n z_i = M$ is $\psi(M)$.
\end{lemma}
\begin{proof}
Let $z_1,\ldots,z_n$ be any solution, and let $i$ be
the least index such that $z_i > 0$. 
Then, since
$\psi$ is concave and non-negative, we have that
\[
\psi(z_1 + z_i) + \psi(0) = \psi(z_1 + z_i) \leq \psi(z_1) + \psi(z_i).
\]
Thus, replacing $z_1$ with $z_1 + z_i$, and replacing $z_i$ with $0$,
produces a solution with one fewer nonzero entries that it at least as 
good.  Repeating this for each $i > 1$ implies that the solution with $z_1 = M$ and $z_2 = \ldots = z_n = 0$ is optimal.  
\end{proof}

The next lemma 
shows that large weights are needed to achieve small loss.
\begin{lemma}
\label{l:aux.lower.bound.norm.weight.vector} For any $L \in \N$ and any weight matrix $V$ if $J(V) \le \frac{2}{n^{1+24L}}$ then,
$\lv V \rv > \sqrt{L+1} \geq \sqrt{2}$.
\end{lemma}
\begin{proof}
Since $\phi$ is $1$-Lipschitz and $\phi(0) = 0$, for all $z$,
$|\phi(z)| \leq |z|$, and thus, given any sample $s$, 
\begin{align*}
    J(V;x_s,y_s) &= \log\left(1+\exp\left( -y_s V_{L+1}\phi(V_{L}\cdots \phi(V_1 x))\right) \right) \\
    & \ge \log\left(1+\exp\left( -\prod_{j=1}^{L+1} \lv V_j \rv_{op} \lv x_s \rv\right) \right) \\
    & \ge \log\left(1+\exp\left( -\prod_{j=1}^{L+1} \lv V_j \rv_{op} \right) \right) \hspace{0.5in} \mbox{(since $\lv x_s\rv = 1$)}\\
    & \ge \log\left(1+\exp\left( -\prod_{j=1}^{L+1} \lv V_j \rv \right) \right).
\end{align*} 
By the AM-GM inequality
\begin{align*}
 \left(\prod_{j=1}^{L+1} \lv V_j \rv^2   \right)^{\frac{1}{L+1}} \le \frac{\sum_{j=1}^{L+1}\lv V_j\rv^2}{L+1}
 %
 = \frac{\lv V \rv^2 }{L+1}.
\end{align*}
Therefore
\begin{align*}
   J(V;x_s,y_s) & \ge  \log\left(1+\exp\left( -\left(\frac{\lv V \rv}{\sqrt{L+1}}\right)^{L+1} \right) \right).
\end{align*}
Now we know that 
\begin{align*}
     \frac{2}{n^{1+24L}}>J(V) = \frac{1}{n}\sum_{s\in [n]} J(V;x_s,y_s) \ge \log\left(1+\exp\left( -\left(\frac{\lv V \rv}{\sqrt{L+1}}\right)^{L+1} \right) \right).
\end{align*}
Solving for $\lv V \rv$
leads to the implication
\begin{align*}
    \sqrt{L+1} \log^{\frac{1}{L+1}}\left( \frac{1}{\exp\left( \frac{2}{n^{1+24L}}\right)-1}\right)< \lv V \rv.
\end{align*}
Since for any $z \in [0,1]$, $\exp(z) \le 1+2z$  and $n\ge 3$, hence
\begin{align*}
    \lv V \rv >   \sqrt{L+1} \log^{\frac{1}{L+1}}\left( \frac{n^{1+24L}}{ 4}\right)
    \ge   \sqrt{L+1} \log^{\frac{1}{L+1}}\left( \frac{3^{1+24L}}{ 4}\right)
    &>   \sqrt{L+1} \log^{\frac{1}{L+1}}\left( 3^{23L}\right)\\
    &=   \sqrt{L+1} (23L)^{\frac{1}{L+1}}\log^{\frac{1}{L+1}}\left( 3\right) \\
    & > \sqrt{L+1}
    \ge \sqrt{2}.
\end{align*}
\end{proof}

\begin{lemma}
\label{l:aux.lower.bound.norm.weight.vector.part2} 
For any $L \in \N$:
\begin{enumerate}
    \item if $\lv V \rv > \sqrt{L+1}$, 
then $\max_{k \in [L]} \prod_{j=k+1}^{L+1} \lv V_j \rv 
    \le
     \left(\frac{\lv V \rv}{\sqrt{L}}\right)^{L}
     \le \lv V \rv^{L}$;
     \item if $\lv V \rv > 1$, 
then $\max_{k \in [L]} \prod_{j=k+1}^{L+1} \lv V_j \rv 
     \le \lv V \rv^{L}$.
\end{enumerate}
\end{lemma}
\begin{proof}
Let $\eta^2 = \lv V \rv^2$. Then for any $k \in [L]$, $$\prod_{j = k+1}^{L+1} \lv V_{j}\rv$$ is maximized subject to $\sum_{j = k+1}^{L+1} \lv V_j \rv^2 \le \eta^2$ when every 
%
$\lv V_j \rv^2 = \eta^2/(L-k+1)$;
this follows by the AM-GM inequality. 

Therefore we have
\begin{align*}
\max_{k \in [L]}\prod_{j=k+1}^{L+1}\lv V_j \rv_{op}\le 
     \max_{k \in [L]} 
     \left(\frac{\eta}{\sqrt{L-k+1}} \right)^{L-k+1}  .
\end{align*}
If $\lv V \rv \geq \sqrt{L+1}$,
\begin{align*}
     \max_{k \in [L]} 
     \left(\frac{\eta}{\sqrt{L-k+1}} \right)^{L-k+1} 
     \leq \left(\frac{\eta}{\sqrt{L}} \right)^{L} 
     & \le \eta^{L}.
\end{align*}
and if $\lv V \rv > 1$ then
\begin{align*}
    \max_{k \in [L]} 
     \left(\frac{\eta}{\sqrt{L-k+1}} \right)^{L-k+1} 
     \leq \eta^{L}.
\end{align*}
\end{proof}

The next lemma bounds the product of the operator norms of matrices in terms of a ``collective Frobenius norm''.
\begin{lemma}
\label{l:prod.by.sum}
For matrices $A_1,\ldots,A_{L+1}$
and $M_1,\ldots,M_{L+1}$, 
let $A = (A_1,\ldots,A_{L+1})$.  For all $i \in [L+1]$, $\lv M_i \rv_{op} \leq 1$.  
Then, for any nonempty $\cI \ss [L+1]$
\[
\prod_{i \in \cI}  \lv A_i \rv_{op} \lv M_i \rv_{op} \leq 
\max\left\{ 
\frac{\lv A \rv^{L+1}}{(L+1)^{\frac{L+1}{2}}},
 \lv A \rv \right\}.
 \]
\end{lemma}
\begin{proof}
We know that for all $i \in [L+1]$, $\lv A_i \rv_{op} \le \lv A_i \rv$, therefore, by the AM-GM inequality
\begin{align*}
    \prod_{i \in \cI} \left( \lv A_i \rv_{op}^2\lv M_i \rv_{op}^2 \right) \le \prod_{i\in \cI} \lv A_i \rv_{op}^2 \le \prod_{i \in \cI} \lv A_i \rv^2 &\le \left( \frac{\sum_{i \in \cI} \lv A_i \rv^2 }{|\cI|}\right)^{|\cI|}\\
    &\le \left( \frac{\lv A \rv^2 }{|\cI|}\right)^{|\cI|}\\
    &\le \max\left\{\frac{\lv A \rv^{2(L+1)}}{(L+1)^{L+1} },\lv A\rv^2\right\}.
\end{align*}
Taking square roots completes the proof.
\end{proof}
The next lemma bounds 
bounds how much perturbing the factors changes a 
product of matrices.
\begin{lemma}
\label{l:ABCD}
Let $A_1,\ldots,A_{L+1}$, $B_1,\ldots,B_{L+1}$, $M_1,\ldots,M_{L+1}$ and $N_1,\ldots,N_{L+1}$
be matrices, and 
let $A = (A_1,\ldots,A_{L+1})$
and $B = (B_1,\ldots,B_{L+1})$.
Assume
\begin{itemize}
    \item  $\lv A \rv \geq \sqrt{L+1/2}$,
    \item for all $i \in [L+1]$, $\lv M_i \rv_{op} \leq 1$
and $\lv N_i \rv_{op} \leq 1$ and
\item for all $i \in [L+1]$, $\lv M_i - N_i \rv_{op} \leq \kappa$,
\end{itemize}
then 
\[
\left\lv \prod_{i=1}^{L+1} (A_i M_i) - \prod_{i=1}^{L+1} (B_i N_i) \right\rv_{op}
 \leq \frac{3}{2}(\lv A \rv + \lv A - B \rv)^{L+1}
         \left(
         \kappa\lv A \rv
        + 
          \left\lv A- B  \right\rv
        \right).
\]
\end{lemma}
\begin{proof}
By the triangle inequality
\begin{align}
\nonumber
 & \left\lv\prod_{i=1}^{L+1} (A_i M_i) -  \prod_{i=1}^{L+1} (B_iN_i) \right\rv_{op} \\
 \nonumber
 & = \left\lv
            \sum_{j=1}^{L+1} \left( 
                    \left( \prod_{i=1}^j A_i M_i \right) \left( \prod_{i=j+1}^{L+1} B_i N_i \right)
                    - \left( \prod_{i=1}^{j-1} A_i M_i \right) \left( \prod_{i=j}^{L+1} B_iN_i \right)
                      \right)\right\rv_{op} \\
                      \nonumber
 & \le \sum_{j=1}^{L+1} \left\lv 
                    \left( \prod_{i=1}^j A_i M_i \right) \left( \prod_{i=j+1}^{L+1} B_i N_i \right)
                    - \left( \prod_{i=1}^{j-1} A_i M_i \right) \left( \prod_{i=j}^{L+1} B_i N_i \right)
                      \right\rv_{op} \\
                      & = \sum_{j=1}^{L+1} \left\lv
                    \left(A_jM_j - B_j N_j \right)\left( \prod_{i=1}^{j-1} A_i M_i \right)\left(\prod_{i=j+1}^{L+1} B_i N_i \right)\right\rv_{op}
                      \nonumber
                      \\
                      \label{e:by.partial.prods}
   & \leq \sum_{j=1}^{L+1} \left\lv A_jM_j - B_jN_j \right\rv_{op}
                    \left\lv
                       \left( \prod_{i=1}^{j-1} A_i M_i \right)
                       \left( \prod_{i=j+1}^{L+1} B_i N_i \right)\right\rv_{op}.
\end{align}
For some $j$,
consider 
$T(j) \eqdef (A_1,\ldots,A_{j-1}, B_{j+1}, B_{L+1})$.
By the triangle inequality,
\[
\lv T(j) \rv \leq \lv A \rv + \lv A - B \rv.
\]
Thus, Lemma~\ref{l:prod.by.sum} implies
\begin{align*}
\left\lv \left( \prod_{i=1}^{j-1} A_i M_i\right)
                       \left( \prod_{i=j+1}^{ L+1} B_i N_i \right)
                       \right\rv_{op}
                       & \le \left( \prod_{i=1}^{j-1} \lv  A_i M_i \rv_{op}\right)
                       \left( \prod_{i=j+1}^{ L+1} \lv B_i N_i\rv_{op} \right)\\
                       & \le \left( \prod_{i=1}^{j-1} \lv  A_i  \rv_{op}\right)
                       \left( \prod_{i=j+1}^{ L+1} \lv B_i \rv_{op} \right)\\
&                       \leq \max\left\{\frac{(\lv A \rv + \lv A - B \rv)^{L+1}}{(L+1)^{\frac{L+1}{2}}},\lv A\rv+\lv A-B\rv\right\}.
\end{align*}
 Returning to \eqref{e:by.partial.prods},
\begin{align*}
\nonumber
 & \left\lv \prod_{i=1}^{L+1} A_iM_i - \prod_{i=1}^{L+1} B_iN_i \right\rv_{op} \\
 \nonumber
   & \leq \max\left\{\frac{(\lv A \rv + \lv A - B \rv)^{L+1}}{(L+1)^{\frac{L+1}{2}}},\lv A\rv+\lv A-B\rv\right\}
   \sum_{j=1}^{L+1}
     \left\lv A_j M_j - B_j N_j \right\rv_{op}
         \\
     & = \max\left\{\frac{(\lv A \rv + \lv A - B \rv)^{L+1}}{(L+1)^{\frac{L+1}{2}}},\lv A\rv+\lv A-B\rv\right\} \sum_{j=1}^{L+1} \left\lv A_j M_j - A_j N_j + A_j N_j - B_j N_j \right\rv_{op}
         \\
     & \leq \max\left\{\frac{(\lv A \rv + \lv A - B \rv)^{L+1}}{(L+1)^{\frac{L+1}{2}}},\lv A\rv+\lv A-B\rv\right\} \sum_{j=1}^{L+1} 
      \left(
        \left\lv A_j (M_j -  N_j) \right\rv_{op} + 
          \left\lv A_j- B_j  \right\rv_{op}
        \right)
         \\
         & \leq \max\left\{\frac{(\lv A \rv + \lv A - B \rv)^{L+1}}{(L+1)^{\frac{L+1}{2}}},\lv A\rv+\lv A-B\rv\right\}
         \left[
         \sum_{j=1}^{L+1}
        \left\lv A_j \right\rv_{op} \left\lv M_j -  N_j \right\rv_{op}
        + 
        \sum_{j=1}^{L+1}
          \left\lv A_j- B_j  \right\rv_{op}
        \right]
         \\
         & \leq \max\left\{\frac{(\lv A \rv + \lv A - B \rv)^{L+1}}{(L+1)^{\frac{L+1}{2}}},\lv A\rv+\lv A-B\rv\right\}
         \left[
         \kappa\sum_{j=1}^{L+1}
        \left\lv A_j \right\rv 
        + 
        \sum_{j=1}^{L+1}
          \left\lv A_j- B_j  \right\rv
        \right]
         \\
         & \leq \max\left\{\frac{(\lv A \rv + \lv A - B \rv)^{L+1}}{(L+1)^{\frac{L+1}{2}}},\lv A\rv+\lv A-B\rv\right\}
         \left[\sqrt{L+1}\left(
         \kappa\lv A \rv
        + 
          \left\lv A- B  \right\rv
\right)        \right]
         \\
         & = \max\left\{\frac{(\lv A \rv + \lv A - B \rv)^{L+1}}{(L+1)^{\frac{L}{2}}},\sqrt{L+1}\left(\lv A\rv+\lv A-B\rv\right)\right\}
         \left(
         \kappa\lv A \rv
        + 
          \left\lv A- B  \right\rv
        \right)\\&  \le \frac{3}{2}(\lv A \rv + \lv A - B \rv)^{L+1}
         \left(
         \kappa\lv A \rv
        + 
          \left\lv A- B  \right\rv
        \right),
\end{align*}
where the last inequality holds since $\lv A\rv\ge \sqrt{L+1/2}$.
This completes the proof.
\end{proof}

The next lemma shows that $h$-smoothly approximately ReLU activations are contractive maps.
\begin{lemma}
\label{l:aux_contractive_maps} Given an $h$-smoothly approximately ReLU activation $\phi$, for any $v_1,v_2 \in \R^p$ we have $\lv \phi(v_1)-\phi(v_2)\rv\le \lv v_1 - v_2\rv$. That is, $\phi$ is a contractive map with respect to the Euclidean norm.
\end{lemma}
\begin{proof}Let $(v)_j$ denote the $j$th coordinate of a vector $v$. For each $j \in [p]$, by 
the mean value theorem
for some $\tilde{v}_j \in [(v_2)_j,(v_1)_j]$
\begin{align*}
    (\phi(v_1) - \phi(v_2))_j & = \phi'(\tilde{v}_j)  (v_1-v_2)_j .
\end{align*}
Thus,
\begin{align*}
    \lv \phi(v_1) - \phi(v_2) \rv^2 = \sum_{j\in [p]} \left(\phi'(\tilde{v}_j)  (v_1-v_2)_j \right)^2 = \sum_{j\in [p]} \left(\phi'(\tilde{v}_j) \right)^2 \left(v_1-v_2 \right)_j^2 &\overset{(i)}{\le} \sum_{j \in [p]}\left(v_1-v_2 \right)_j^2\\ &=\lv v_1 - v_2 \rv^2,
\end{align*}
where $(i)$ follows because $|\phi'(z)|\le 1$ for all $z \in \R$ for $h$-smoothly approximately ReLU activations. Taking square roots completes the proof.
\end{proof}

\subsection{Proof of Lemma~\ref{l:pl}}
\label{a:pl}
\pllemma*
\begin{proof}
Along the line segment joining $V^{(t)}$ to $V^{(t+1)}$, the function $J(\cdot)$ is $M$-smooth, therefore by using a standard argument~\citep[see, e.g.,][Lemma~3.4]{bubeck2015convex} we get that
\begin{align*}
    J_{t+1} &\le J_t + \nabla J_t \cdot (V^{(t+1)}-V^{(t)}) + \frac{M}{2} \lv V^{(t+1)} - V^{(t)} \rv^2\\
    & = J_t - \alpha \lv \nabla J_t\rv^2 + \frac{\alpha^2 M}{2} \lv \nabla J_t\rv^2 \\
    & = J_t - \alpha \left(1-\frac{\alpha M}{2}\right)\lv \nabla J_t\rv^2\\
    & \le J_t - \frac{L}{L+\frac{1}{2}}\alpha \lv \nabla J_t\rv^2.
\end{align*}
This completes the proof.
\end{proof}



\subsection{Proof of Lemma~\ref{l:L_smooth}} \label{a:smoothness_lemma}

The proof of Lemma~\ref{l:L_smooth} is built up in
stages, through a series of lemmas.

The first lemma bounds the norm of the difference between the pre-activation ($u_{j,s}^V$) and post-activation features ($x_{j,s}^V$) at any layer $j$, when the weight matrix of a single layer is swapped. It also provides a bound on the norm of the pre-activation and post-activation features at any layer in terms of the norm of the weight matrix.
\begin{lemma}
\label{l:u.x.bounds}
Consider 
$V = (V_1,\ldots,V_{L+1})$ and $W = (W_1,\ldots,W_{L+1})$, and $\ell \in [L+1]$. Suppose that $V_j = W_j$ for all $j \neq \ell$, and $\lv V \rv, \lv W \rv> \sqrt{L+1/2}$. Then,
for all examples $s$ and all layers $j$,
\begin{enumerate}
\item $\lv u_{j,s}^V \rv \le \lv V \rv^{L+1}$; 
\item $\lv x_{j,s}^V \rv \le \lv V \rv^{L+1}$;
    \item $\lv u_{j,s}^V - u_{j,s}^W \rv \leq 
\lv V_{\ell} - W_{\ell} \rv_{op} \lv V \rv^{L+1}$; and
\item $\lv x_{j,s}^V - x_{j,s}^W \rv \leq 
\lv V_{\ell} - W_{\ell} \rv_{op} \lv V \rv^{L+1}$.
\end{enumerate}
\end{lemma}
\begin{proof} 
\textit{Proof of Parts 1 and 2:} For any sample $s$ and layer $j$ we have
\begin{align*}
    \lv u_{j,s}^V\rv = \lv V_{j}\phi(u_{j-1,s}^V)\rv 
    \le \lv V_{j}\rv_{op}\lv\phi(u_{j-1,s}^V)\rv 
    \overset{(i)}{\le} \lv V_{j}\rv_{op}\lv u_{j-1,s}^V\rv 
    &\le \prod_{k=1}^{j}\lv V_{k}\rv_{op}\lv x_s\rv \\
    &\overset{(ii)}{\le} \prod_{k=1}^{j}\lv V_{k}\rv_{op} \\
    &\overset{(iii)}{\le} \lv V \rv^{L+1},
\end{align*}
where $(i)$ follows since $\phi$ is contractive
(Lemma~\ref{l:aux_contractive_maps}), $(ii)$ is because $\lv x_s \rv = 1$ and $(iii)$ is by Lemma~\ref{l:prod.by.sum}. This completes the proof of Part~1 of this lemma. Again since $\phi$ is contractive, $\lv x_{j,s}^V \rv = \lv \phi(u_{j,s}^V)\rv \le \lv V\rv^{L+1}$, which establishes the second part of the lemma.

\textit{Proof of Parts 3 and 4:}
For any $j < \ell$, 
$u_{j,s}^V = u_{j,s}^W$ and
$x_{j,s}^V = x_{j,s}^W$, since $V_j = W_j$ for all $j \neq \ell$. For $j=\ell$ we have
\begin{align*}
 \lv u_{\ell,s}^V - u_{\ell,s}^W \rv
 = \lv V_{\ell} x_{\ell-1,s}^V - W_{\ell}  x_{\ell-1,s}^W \rv 
 = \lv (V_{\ell}  - W_{\ell})  x_{\ell-1,s}^V \rv 
& \leq \lv V_{\ell} - W_{\ell} \rv_{op} \lv x^V_{\ell-1,s} \rv \\
& = \lv V_{\ell} - W_{\ell} \rv_{op} \lv \phi\left(V_{\ell-1}x^V_{\ell-2,s}\right) \rv \\
& \overset{(i)}{\leq} \lv V_{\ell} - W_{\ell} \rv_{op} \prod_{k < \ell} \lv V_{k} \rv_{op}\\
& \leq \lv V_{\ell} - W_{\ell} \rv_{op} \prod_{k < \ell} \lv V_{k} \rv\\
& \overset{(ii)}{\leq} \lv V_{\ell} - W_{\ell} \rv_{op} \lv V \rv^{L+1}
\end{align*}
where $(i)$ follows since $\phi$ is a contractive map (Lemma~\ref{l:aux_contractive_maps}) and because $\lv x_s\rv=1$, and (ii) follows by applying Lemma~\ref{l:prod.by.sum}. Since $\phi$ is contractive we also have that
\[
\lv x_{\ell,s}^V - x_{\ell,s}^W \rv
 \leq \lv V_{\ell} - W_{\ell} \rv_{op} \lv V \rv^{L+1}.
\]
When $j > \ell$, it is possible to establish our claim by mirroring the argument in the $j =\ell$ case which completes the proof of the last two parts.

\end{proof}
The next lemma upper bounds difference between the $\Sigma_{j,s}^V$ and $\Sigma_{j,s}^W$, when the weight matrices differ in a single layer.
\begin{lemma}
\label{l:phi.prime.close}
Consider $V = (V_1,\ldots,V_L)$ and $W = (W_1,\ldots,W_L)$,
and $\ell \in [L]$. Suppose that $V_j = W_j$ for all $j \neq \ell$, and $\lv V \rv, \lv W \rv > \sqrt{L+1/2}$. Then,
for all examples $s$ and all layers $j$,
\[
\lv \Sigma_{j,s}^V - \Sigma_{j,s}^W \rv_{op}
   \leq \frac{\lv V_{\ell} - W_{\ell} \rv_{op} \lv V \rv^{L+1}}{h}.
\]
\end{lemma}
\begin{proof}
For any $j\in [L]$ and any $s\in [n]$, $\Sigma_{j,s}^V$ and $\Sigma_{j,s}^W$ are both diagonal matrices, and hence
\begin{align*}
\lv \Sigma_{j,s}^V - \Sigma_{j,s}^W \rv_{op}
 & = \lv \phi'(u_{j,s}^V) - \phi'(u_{j,s}^W) \rv_{\infty} \\
 & \leq \frac{\lv u_j^V - u_j^W \rv_{\infty} }{h} \hspace{1in}\mbox{(since $\phi'$ is $(1/h)$-Lipschitz)} \\
 & \leq \frac{\lv V_{\ell} - W_{\ell} \rv_{op} \lv V \rv^{L+1}}{h},
\end{align*}
by Lemma~\ref{l:u.x.bounds}.
\end{proof}

The following lemma bounds the difference between $g_s(V)$ and $g_s(W)$ for any sample $s$ when the weight matrices $V$ and $W$ differ in a single layer.
\begin{lemma}
\label{l:g_lipschitz}Consider $V = (V_1,\ldots,V_{L+1})$ and $W = (W_1,\ldots,W_{L+1})$,
and $\ell \in [L+1]$. Suppose that $V_j = W_j$ for all $j \neq \ell$, with $ \lv V_{\ell}-W_{\ell} \rv_{op} \lv V \rv^{L+1}
         \leq 1$, and $\lv V \rv, \lv W \rv > \sqrt{L+1/2}$. Also suppose that,
for all examples $s$, for all convex combinations $\tW$ of $V$ and
$W$, we have
$J_s(\tW) \leq 2 J_s(V)$.  Then
\begin{align*}
    \left\lvert g_s(V) - g_s(W) \right\rvert & \le  2J_s(V) \lv V_{\ell}-W_{\ell}\rv_{op}\lv V \rv^{L+1}.
\end{align*}
\end{lemma}
\begin{proof}
  By Taylor's theorem applied to the function $1/(1+\exp(z))$ we can bound
 \begin{align*}
   & \lvert g_s(V) - g_s(W) \rvert \\
    & = \left\lvert \frac{1}{1+\exp\left( y_s f_{V}(x_s)\right)} - \frac{1}{1+\exp\left( y_s f_{W}(x_s)\right)} \right\rvert \\
 & \le  \underbrace{\frac{\exp\left( y_s f_{V}(x_s)\right)}{\left(1+\exp\left( y_s f_{V}(x_s)\right)\right)^2} | y_s f_{W}(x_s) - y_s f_{V}(x_s)|}_{=: \Xi_1} \\& + \underbrace{\frac{\left( y_s f_{W}(x_s) - y_s f_{V}(x_s)\right)^2}{2} \max_{\tW \in [V,W]} \left|\frac{2\exp(2y_s f_{\tW}(x_s))}{(\exp(y_s f_{\tW}(x_s))+1)^3}-\frac{\exp(y_s f_{\tW}(x_s))}{(\exp(y_s f_{\tW}(x_s))+1)^2} \right|}_{=:\Xi_2}.
 \numberthis \label{e:taylors_theorem_for_g}
 \end{align*}
 The first term $\Xi_1$ can be bounded as
 \begin{align*}
     \Xi_1 
     & = \frac{1}{1+\exp\left( y_s f_{V}(x_s)\right)}\frac{\exp\left( y_s f_{V}(x_s)\right)}{1+\exp\left( y_s f_{V}(x_s)\right)} \left\lvert y_s f_{W}(x_s) - y_s f_{V}(x_s)\right\rvert \\
      & \le \frac{1}{1+\exp\left( y_s f_{V}(x_s)\right)} \left\lvert  f_{W}(x_s) -  f_{V}(x_s)\right\rvert \\
    & = g_s(V) \left\lv  u_{L+1,s}^W - u_{L+1,s}^V\right\rv  \overset{(i)}\le g_s(V)  \lv V_{\ell}-W_{\ell} \rv_{op} \lv V \rv^{L+1}
       \overset{(ii)}{\le} J_s(V) \lv V_{\ell}-W_{\ell} \rv_{op} \lv V \rv^{L+1},
 \end{align*}
where $(i)$ follows by applying Lemma~\ref{l:u.x.bounds} and $(ii)$ follows since $g_s(V) \le J_s(V)$ by Lemma~\ref{l:relationsbetweengradientandloss}.
The second term $\Xi_2$
\begin{align*}
\Xi_2 
    & =\frac{\left( y_s f_{W}(x_s) - y_s f_{V}(x_s)\right)^2}{2} \max_{\tW \in [V,W]} \left|\frac{2\exp(2y_s f_{\tW}(x_s))}{(\exp(y_s f_{\tW}(x_s))+1)^3}-\frac{\exp(y_s f_{\tW}(x_s))}{(\exp(y_s f_{\tW}(x_s))+1)^2} \right|\\
    & \overset{(i)}\le \frac{\left( f_{W}(x_s) -  f_{V}(x_s)\right)^2}{2} \max_{\tW \in [V,W]} \log(1+\exp(-y_s f_{\tW}(x_s))) \\
    & = \frac{\left( f_{W}(x_s) -  f_{V}(x_s)\right)^2}{2} \max_{\tW \in [V,W]}J_s(\tW) \\
    & \overset{(ii)}\le J_s(V)\left( f_{W}(x_s) -  f_{V}(x_s)\right)^2\\
    & = J_s(V)\left( u_{L+1,s}^V-  u_{L+1,s}^W\right)^2 \overset{(iii)}{\le} J_s(V) \lv V_{\ell} - W_{\ell} \rv_{op}^2\lv V \rv^{2(L+1)}  \overset{(iv)}{\le} J_s(V)\lv V_{\ell} - W_{\ell} \rv_{op}\lv V \rv^{L+1},
\end{align*}
where $(i)$ follows since for every $z \in \mathbb{R}$
\begin{align*}
   \left| \frac{2\exp(2z)}{(\exp(z)+1)^3}-\frac{\exp(z)}{(\exp(z)+1)^2} \right| &\le \log(1+\exp(-z)),
\end{align*}
$(ii)$ is by our assumption that for any $\tW \in [V,W]$, $J_s(\tW) \le 2J_s(V)$, $(iii)$ follows by invoking Lemma~\ref{l:u.x.bounds} and finally $(iv)$ is by the assumption that $\lv V_{\ell}-W_{\ell}\rv_{op}\lv V \rv^{L+1}\le 1$. By using our bounds on $\Xi_1$ and $\Xi_2$ in conjunction with inequality~\eqref{e:taylors_theorem_for_g} we obtain the bound
\begin{align*}
    \lvert g_s(V) - g_s(W)\rvert \le 2J_s(V) \lv V_{\ell}-W_{\ell}\rv_{op}\lv V \rv^{L+1}
\end{align*}
completing the proof.
\end{proof}
By using Lemmas~\ref{l:u.x.bounds}, \ref{l:phi.prime.close} and \ref{l:g_lipschitz} we will now bound the norm of the difference of the gradients of the loss at $V$ and $W$, when these weight matrices differ in a single layer. 
\begin{lemma}
\label{l:nabla.smooth.one.layer}
Let $h \le 1$, and consider $V = (V_1,\ldots,V_{L+1})$ and $W = (W_1,\ldots,W_{L+1})$,
and
%
$\ell \in [L+1]$. 
Suppose that $V_j = W_j$ for all $j \neq \ell$, and 
\begin{itemize}
    \item $\lv V_{\ell}-W_{\ell} \rv_{op} 
          \leq 1$;
         \item $\lv V-W \rv\le \frac{\lv V \rv}{2(L+1)}$;
         \item $\lv V \rv > \sqrt{L+1/2}$ and $\lv W \rv > \sqrt{L+1/2}$;
         \item for all $s$ and all convex combinations $\tW$ of $V$ and
$W$,  $J_s(\tW) \leq 2 J_s(V)$.
\end{itemize} Then,
\[
\lv \nabla_V J_s(V) - \nabla_W J_s(W)  \rv
   \leq \frac{64 \sqrt{(L+1)p}J_s(V)\lv V \rv^{3L+5}\lv V_{\ell}-W_{\ell}\rv}{h}   .
\]
\end{lemma}
\begin{proof}
We can decompose $\lv \nabla_V J_s(W) - \nabla_W J_s(V)  \rv^2$ into contributions from different layers as follows:
\begin{align}
\lv \nabla_V J_s(W) - \nabla_W J_s(V)  \rv^2
 = \sum_{k=1}^{L+1}
 \lv \nabla_{V_k} J_s(W) - \nabla_{W_k} J_s(V)  \rv^2 \label{e:decomposition_of_grad_in_terms_of_layers}
\end{align}
First 
we seek a bound on $\lv \nabla_{V_k} J_s(V) - \nabla_{W_k} J_s(W) \rv_{op}$ when $k \in [L]$.
(We will handle the
output layer separately.)
We have 
\begin{align*}
& \lv \nabla_{V_k} J_s(V) - \nabla_{W_k} J_s(W)  \rv_{op} \\
   & = \left\lv g_s(W)
        \left( \Sigma^W_{k,s} \prod_{j = k+1}^L 
             \left(  W_{j}^{\top} \Sigma^W_{j,s} \right)
           \right)
           W_{L+1}^{\top}
        x_{k-1,s}^{W\top}  - g_s(V)
        \left( \Sigma^V_{k,s} \prod_{j = k+1}^L 
             \left(  V_{j}^{\top} \Sigma^V_{j,s} \right)
           \right)
           V_{L+1}^{\top}
        x_{k-1,s}^{V\top}\right\rv_{op} \\ 
        & = \left\lv g_s(W)
        \left( \Sigma^W_{k,s} \prod_{j = k+1}^L 
             \left(  W_{j}^{\top} \Sigma^W_{j,s} \right)-\Sigma^V_{k,s} \prod_{j = k+1}^L 
             \left(  V_{j}^{\top} \Sigma^V_{j,s} \right)+\Sigma^V_{k,s} \prod_{j = k+1}^L 
             \left(  V_{j}^{\top} \Sigma^V_{j,s} \right)
           \right)
           W_{L+1}^{\top}
        x_{k-1,s}^{W\top} \right. \\ &\qquad \qquad  \left. - g_s(V)
        \left( \Sigma^V_{k,s} \prod_{j = k+1}^L 
             \left(  V_{j}^{\top} \Sigma^V_{j,s} \right)
           \right)
           V_{L+1}^{\top}
        x_{k-1,s}^{V\top}\right\rv_{op} \\ 
        & = \left\lv g_s(W)
        \left( \Sigma^W_{k,s} \prod_{j = k+1}^L 
             \left(  W_{j}^{\top} \Sigma^W_{j,s} \right)-\Sigma^V_{k,s} \prod_{j = k+1}^L 
             \left(  V_{j}^{\top} \Sigma^V_{j,s} \right) \right)
           W_{L+1}^{\top}
        x_{k-1,s}^{W\top} \right. \\ 
        &\qquad \qquad 
        +g_s(W)
             \Sigma^V_{k,s} \prod_{j = k+1}^L 
             \left(  V_{j}^{\top} \Sigma^V_{j,s} \right)
           W_{L+1}^{\top}
        x_{k-1,s}^{W\top}  \\ 
        &\qquad \qquad  \left. - g_s(V)
        \left( \Sigma^V_{k,s} \prod_{j = k+1}^L 
             \left(  V_{j}^{\top} \Sigma^V_{j,s} \right)
           \right)
           V_{L+1}^{\top}
        x_{k-1,s}^{V\top}\right\rv_{op} \\  
        & = \left\lv g_s(W)
        \left( \Sigma^W_{k,s} \prod_{j = k+1}^L 
             \left(  W_{j}^{\top} \Sigma^W_{j,s} \right)-\Sigma^V_{k,s} \prod_{j = k+1}^L 
             \left(  V_{j}^{\top} \Sigma^V_{j,s} \right) \right)
           W_{L+1}^{\top}
        x_{k-1,s}^{W\top} \right. \\ 
        &\qquad \qquad 
        +g_s(W)
             \Sigma^V_{k,s} \prod_{j = k+1}^L 
             \left(  V_{j}^{\top} \Sigma^V_{j,s} \right)
           W_{L+1}^{\top}
        (x_{k-1,s}^{W\top} - x_{k-1,s}^{V\top}+x_{k-1,s}^{V\top} )  \\ 
        &\qquad \qquad  \left. - g_s(V)
        \left( \Sigma^V_{k,s} \prod_{j = k+1}^L 
             \left(  V_{j}^{\top} \Sigma^V_{j,s} \right)
           \right)
           V_{L+1}^{\top}
        x_{k-1,s}^{V\top}\right\rv_{op} \\ 
        & = \left\lv g_s(W)
        \left( \Sigma^W_{k,s} \prod_{j = k+1}^L 
             \left(  W_{j}^{\top} \Sigma^W_{j,s} \right)-\Sigma^V_{k,s} \prod_{j = k+1}^L 
             \left(  V_{j}^{\top} \Sigma^V_{j,s} \right) \right)
           W_{L+1}^{\top}
        x_{k-1,s}^{W\top} \right. \\ 
        &\qquad \qquad 
        +g_s(W)
             \Sigma^V_{k,s} \prod_{j = k+1}^L 
             \left(  V_{j}^{\top} \Sigma^V_{j,s} \right)
           W_{L+1}^{\top}
        (x_{k-1,s}^{W\top} - x_{k-1,s}^{V\top})  \\ 
        &\qquad \qquad 
        +g_s(W)
             \Sigma^V_{k,s} \prod_{j = k+1}^L 
             \left(  V_{j}^{\top} \Sigma^V_{j,s} \right)
           W_{L+1}^{\top}
        x_{k-1,s}^{V\top}  \\ 
        &\qquad \qquad  \left. - g_s(V)
        \left( \Sigma^V_{k,s} \prod_{j = k+1}^L 
             \left(  V_{j}^{\top} \Sigma^V_{j,s} \right)
           \right)
           V_{L+1}^{\top}
        x_{k-1,s}^{V\top}\right\rv_{op}.
\end{align*}
Applying the triangle inequality
\begin{align*}
    & \lv \nabla_{V_k} J_s(V) - \nabla_{W_k} J_s(W)  \rv_{op} \\
    & \le  \underbrace{\left\lv g_s(W)
        \left( \Sigma^V_{k,s} \prod_{j = k+1}^L 
             \left(  V_{j}^{\top} \Sigma^V_{j,s} \right)
           \right)
           W_{L+1}^{\top}
        (x_{k-1,s}^{W\top}-x_{k-1,s}^{V\top})  \right\rv_{op}}_{=: \Xi_1} \\
        &\quad + \underbrace{\left\lv 
        g_s(W)\left( \Sigma^V_{k,s} \prod_{j = k+1}^L 
             \left(  V_{j}^{\top} \Sigma^V_{j,s} \right)
           \right)
          W_{L+1}^{\top}
        x_{k-1,s}^{V\top}-g_s(V)\left( \Sigma^V_{k,s} \prod_{j = k+1}^L 
             \left(  V_{j}^{\top} \Sigma^V_{j,s} \right)
           \right)V_{L+1}^{\top}x_{k-1,s}^{V\top} \right\rv_{op} }_{=:\Xi_2}\\
        & \quad + \underbrace{\left\lv g_s(W)
        \left( \Sigma^W_{k,s} \prod_{j = k+1}^L 
             \left(  W_{j}^{\top} \Sigma^W_{j,s} \right)-\Sigma^V_{k,s} \prod_{j = k+1}^L 
             \left(  V_{j}^{\top} \Sigma^V_{j,s} \right)
           \right)
           W_{L+1}^{\top}
        x_{k-1,s}^{W\top} \right\rv_{op}}_{=:\Xi_3} . \numberthis \label{e:def_of_xi_1_2_3}
\end{align*}
We will control each of these three terms separately in lemmas below. First in Lemma~\ref{l:bound_on_xi_1} we establish that \begin{align*}
    \Xi_1 &\le  4J_s(V)\lv V_{\ell}-W_{\ell}\rv_{op} \lv V \rv^{2(L+1)},
\end{align*}
then in Lemma~\ref{l:bound_on_xi_2} we prove that
\begin{align*}
    \Xi_2 \le 4 J_s(V) \lv V \rv^{3(L+1)}\lv V_{\ell} - W_{\ell}\rv_{op},
\end{align*}
and in Lemma~\ref{l:bound_on_xi_3} we establish that
\begin{align*}
    \Xi_3 &\le \frac{56 J_s(V)\lv V \rv^{3L+5}\lv V_{\ell}-W_{\ell}\rv}{h}.
\end{align*}
These three bound combined with the decomposition in \eqref{e:def_of_xi_1_2_3} tells us that for any $k \in [L]$
\begin{align*}
     \lv \nabla_{V_k} J_s(V) - \nabla_{W_k} J_s(W)  \rv_{op} &\le 4J_s(W)\lv V_{\ell}-W_{\ell}\rv_{op} \lv V \rv^{2(L+1)}+
   4 J_s(V) \lv V \rv^{3(L+1)}\lv V_{\ell} - W_{\ell}\rv_{op}
     \\&\qquad \qquad+\frac{56 J_s(V)\lv V \rv^{3L+5}\lv V_{\ell}-W_{\ell}\rv}{h}\\
    &\le \frac{64 J_s(V)\lv V \rv^{3L+5}\lv V_{\ell}-W_{\ell}\rv}{h},
\end{align*}
where the previous inequality follows since $h<1$ and $\lv V \rv > 1$. Since $V_{k}$ and $W_k$ are a $p \times p$-dimensional matrices, we find that
\begin{align*}
    \lv \nabla_{V_k} J_s(V) - \nabla_{W_k} J_s(W)  \rv &\le \sqrt{p} \lv \nabla_{V_k} J_s(V) - \nabla_{W_k} J_s(W)  \rv_{op} \\
     &\le 
     \frac{64 \sqrt{p}J_s(V)\lv V \rv^{3L+5}\lv V_{\ell}-W_{\ell}\rv}{h} \numberthis \label{e:bound_on_gradient_k_less_than_L+1}.
\end{align*}
For the final layer we know that
\begin{align*}
    \lv \nabla_{V_{L+1}} J_s(V) - \nabla_{W_{L+1}} J_s(W)  \rv &= \lv g_s(V)x_{L,s}^V - g_s(W)x_{L,s}^W \rv \\
    &= \lv (g_s(V)-g_s(W)+g_s(W))x_{L,s}^V - g_s(W)x_{L,s}^W \rv \\
    &\le \lvert g_s(V)-g_s(W)\rvert \lv x_{L,s}^V\rv + g_s(W)\lv x_{L,s}^V - x_{L,s}^W \rv \\
    &\overset{(i)}{\le} 2J_s(V)\lv V_{\ell} - W_{\ell}\rv_{op}\lv V\rv^{2(L+1)}  + g_s(W)\lv V_{\ell} -W_{\ell}\rv_{op}\lv V\rv^{L+1} \\
    &\overset{(ii)}{\le} 2J_s(V)\lv V_{\ell} - W_{\ell}\rv_{op}\lv V\rv^{2(L+1)}  + 2J_s(V)\lv V_{\ell} -W_{\ell}\rv_{op}\lv V\rv^{L+1} \\
    &\overset{(iii)}\le 4J_s(V)\lv V_{\ell} - W_{\ell} \rv_{op}\lv V \rv^{2(L+1)}, \numberthis \label{e:bound_on_gradient_k_equal_to_L+1}
\end{align*}where $(i)$ follows by invoking Lemma~\ref{l:g_lipschitz} and Lemma~\ref{l:u.x.bounds}, $(ii)$ follows since $g_s(W)\le J_s(W)$ by Lemma~\ref{l:relationsbetweengradientandloss} and because by assumption $J_s(W)\le 2J_s(V)$, and 
$(iii)$ follows since $\lv V \rv >1$. This previous inequality along with \eqref{e:bound_on_gradient_k_less_than_L+1} and \eqref{e:decomposition_of_grad_in_terms_of_layers} yield
\begin{align*}
&\lv \nabla_V J_s(W) - \nabla_W J_s(V)  \rv^2 \\ &\qquad \qquad \qquad \le  L\left(\frac{64 \sqrt{p}J_s(V)\lv V \rv^{3L+5}\lv V_{\ell}-W_{\ell}\rv}{h}   \right)^2 +\left(4J_s(V) \lv V_{\ell} - W_{\ell} \rv_{op}\lv V \rv^{2(L+1)}\right)^2 \\
&\qquad \qquad \qquad \le (L+1)\left(\frac{64 \sqrt{p}J_s(V)\lv V \rv^{3L+5}\lv V_{\ell}-W_{\ell}\rv}{h}   \right)^2.
\end{align*}
Taking square roots completes the proof.
\end{proof}
As promised in the proof of Lemma~\ref{l:nabla.smooth.one.layer} we now bound $\Xi_1$.
\begin{lemma} Borrowing 
the setting and notation of Lemma~\ref{l:nabla.smooth.one.layer}, if $\Xi_1$ is as defined in \eqref{e:def_of_xi_1_2_3}, we have
\label{l:bound_on_xi_1}
\begin{align*}
    \Xi_1 &\le 4J_s(V)\lv V_{\ell}-W_{\ell}\rv_{op} \lv V \rv^{2(L+1)}.
\end{align*}
\end{lemma}
\begin{proof} Unpacking  using the definition of $\Xi_1$
\begin{align*}
    \Xi_1 & = \left\lv g_s(W)
        \left( \Sigma^V_{k,s} \prod_{j = k+1}^L 
             \left(  V_{j}^{\top} \Sigma^V_{j,s} \right)
           \right)
           W_{L+1}^{\top}
        (x_{k-1,s}^{W\top}-x_{k-1,s}^{V\top})  \right\rv_{op}\\
    & = \left\lv g_s(W)
        \left( \Sigma^V_{k,s} \prod_{j = k+1}^L 
             \left(  V_{j}^{\top} \Sigma^V_{j,s} \right)
           \right)
           (W_{L+1}^{\top}-V_{L+1}^{\top}+V_{L+1}^{\top})
        (x_{k-1,s}^{W\top}-x_{k-1,s}^{V\top})  \right\rv_{op}\\
        & \le g_s(W)\left\lv 
         \Sigma^V_{k,s}\right\rv_{op} 
            \left(
             \prod_{j = k+1}^L 
             \left\lv V_{j}^{\top}\right\rv_{op} \left\lv\Sigma^V_{j,s} \right\rv_{op}
             \right)
           \left\lv W_{L+1}^{\top}-V_{L+1}^{\top}+V_{L+1}^{\top}\right\rv_{op}
        \left\lv x_{k-1,s}^{W\top}-x_{k-1,s}^{V\top} \right\rv\\
        & \overset{(i)}{\le} g_s(W) 
            \left(
             \prod_{j = k+1}^L 
             \left\lv V_{j}^{\top}\right\rv_{op} 
             \right)
           \left\lv W_{L+1}^{\top}-V_{L+1}^{\top}+V_{L+1}^{\top}\right\rv_{op}
        \left\lv x_{k-1,s}^{W\top}-x_{k-1,s}^{V\top} \right\rv\\
        & \le g_s(W) 
             \left(
              \prod_{j = k+1}^L 
             \left\lv V_{j}^{\top}\right\rv_{op} 
             \right)
           \left(\left\lv W_{L+1}^{\top}-V_{L+1}^{\top}\right\rv_{op}+\left\lv V_{L+1}^{\top}\right\rv_{op}\right)
        \left\lv x_{k-1,s}^{W\top}-x_{k-1,s}^{V\top} \right\rv\\
        & \overset{(ii)}{\le} g_s(W)\lv V_{\ell}-W_{\ell}\rv_{op} \lv V \rv^{L+1}
           \left(
          \prod_{j = k+1}^L 
             \left\lv V_{j}^{\top}\right\rv_{op} 
             \right)
           \left(\left\lv W_{L+1}^{\top}-V_{L+1}^{\top}\right\rv_{op}+\left\lv V_{L+1}^{\top}\right\rv_{op}\right)\\
           & =g_s(W)\lv V_{\ell}-W_{\ell}\rv_{op} \lv V \rv^{L+1}
           \left(\left\lv W_{\ell}-V_{\ell}\right\rv_{op}\prod_{j = k+1}^L 
             \left\lv V_{j}^{\top}\right\rv_{op} +\prod_{j = k+1}^{L+1}
             \left\lv V_{j}^{\top}\right\rv_{op} \right)\\
                     & \overset{(iii)}{\le} g_s(W)\lv V_{\ell}-W_{\ell}\rv_{op} \lv V \rv^{L+1}
           \left(\left\lv W_{\ell}-V_{\ell}\right\rv_{op}\lv V\rv^{L+1} +\lv V\rv^{L+1} \right)\\
           & \overset{(iv)}{\le} 2J_s(V)\lv V_{\ell}-W_{\ell}\rv_{op} \lv V \rv^{2(L+1)}
           \left(\left\lv W_{\ell}-V_{\ell}\right\rv_{op} +1\right)\\
           & \overset{(v)}{\le} 4J_s(V)\lv V_{\ell}-W_{\ell}\rv_{op} \lv V \rv^{2(L+1)},
\end{align*}
where $(i)$ follows since $\lv \Sigma_{k,s}^V \rv_{op}\le 1$, $(ii)$ follows from invoking Lemma~\ref{l:u.x.bounds}, $(iii)$ is by
Lemma~\ref{l:prod.by.sum}, $(iv)$ follows since $g_s(W)\le J_s(W)$ by Lemma~\ref{l:relationsbetweengradientandloss} and because by assumption $J_s(W)\le 2J_s(V)$. Finally $(v)$ follows since 
we have assumed that
$\lv V_{\ell}-W_{\ell} \rv_{op}\le 1$.
\end{proof}
We continue and now bound $\Xi_2$ which as defined in the proof of Lemma~\ref{l:nabla.smooth.one.layer}.
\begin{lemma} Borrowing the setting and notation of Lemma~\ref{l:nabla.smooth.one.layer}, if $\Xi_2$ is as defined in \eqref{e:def_of_xi_1_2_3} then
\label{l:bound_on_xi_2}
\begin{align*}
    \Xi_2 &\le 4 J_s(V) \lv V \rv^{3(L+1)}\lv V_{\ell} - W_{\ell}\rv_{op}.
\end{align*}
\end{lemma}
\begin{proof}  Unpacking the term $\Xi_2$
\begin{align*}
    &\Xi_2 \\& = \left\lv 
        g_s(W)\left( \Sigma^V_{k,s} \prod_{j = k+1}^L 
             \left(  V_{j}^{\top} \Sigma^V_{j,s} \right)
           \right)
          W_{L+1}^{\top}
        x_{k-1,s}^{V\top}-g_s(V)\left( \Sigma^V_{k,s} \prod_{j = k+1}^L 
             \left(  V_{j}^{\top} \Sigma^V_{j,s} \right)
           \right)V_{L+1}^{\top}x_{k-1,s}^{V\top} \right\rv_{op} \\
           &=\left\lv 
        g_s(W)\left( \Sigma^V_{k,s} \prod_{j = k+1}^L 
             \left(  V_{j}^{\top} \Sigma^V_{j,s} \right)
           \right)
          (W_{L+1}^{\top}-V_{L+1}^{\top}+V_{L+1}^{\top})
        x_{k-1,s}^{V\top}\right.\\&\qquad \left.-g_s(V)\left( \Sigma^V_{k,s} \prod_{j = k+1}^L 
             \left(  V_{j}^{\top} \Sigma^V_{j,s} \right)
           \right)V_{L+1}^{\top}x_{k-1,s}^{V \top} \right\rv_{op}\\
           &\le\left\lv 
        g_s(W)\left( \Sigma^V_{k,s} \prod_{j = k+1}^L 
             \left(  V_{j}^{\top} \Sigma^V_{j,s} \right)
           \right)
          (W_{L+1}^{\top}-V_{L+1}^{\top})
        x_{k-1,s}^{V\top}\right\rv_{op}\\
        & \qquad
        +\left\lv g_s(W)\left( \Sigma^V_{k,s} \prod_{j = k+1}^L 
             \left(  V_{j}^{\top} \Sigma^V_{j,s} \right)
           \right)
          V_{L+1}^{\top}
        x_{k-1,s}^{V\top}-g_s(V)\left( \Sigma^V_{k,s} \prod_{j = k+1}^L 
             \left(  V_{j}^{\top} \Sigma^V_{j,s} \right)
           \right)V_{L+1}^{\top}x_{k-1,s}^{V \top} \right\rv_{op}\\
           &=\underbrace{\left\lv 
        g_s(W)\left( \Sigma^V_{k,s} \prod_{j = k+1}^L 
             \left(  V_{j}^{\top} \Sigma^V_{j,s} \right)
           \right)
          (W_{L+1}^{\top}-V_{L+1}^{\top})
        x_{k-1,s}^{V\top}\right\rv_{op}}_{\spadesuit_2}\\
        &\qquad \qquad +\underbrace{\left\lv(g_s(W)-g_s(V))\left( \Sigma^V_{k,s} \prod_{j = k+1}^L 
             \left(  V_{j}^{\top} \Sigma^V_{j,s} \right)
           \right)V_{L+1}^{\top}x_{k-1,s}^{V \top} \right\rv_{op}}_{=:\varheart_2}. \numberthis \label{e:xi_2_decomp}
\end{align*}
The first term
\begin{align*}
    \spadesuit_2 &= \left\lv 
        g_s(W)\left( \Sigma^V_{k,s} \prod_{j = k+1}^L 
             \left(  V_{j}^{\top} \Sigma^V_{j,s} \right)
           \right)
          (W_{L+1}^{\top}-V_{L+1}^{\top})
        x_{k-1,s}^{V\top}\right\rv_{op}\\
        & \le g_s(W)\left\lv 
        \Sigma^V_{k,s}\right\rv_{op} \left(\prod_{j = k+1}^L 
             \lv  V_{j}^{\top}\rv_{op}\lv \Sigma^V_{j,s} \rv_{op}\right)
        \lv  W_{L+1}^{\top}-V_{L+1}^{\top}\rv_{op}
        \lv x_{k-1,s}^{V\top}\rv\\
        & \overset{(i)}{\le} g_s(W) \left(\prod_{j = k+1}^L 
             \lv  V_{j}^{\top}\rv_{op}\right)
        \lv  W_{L+1}^{\top}-V_{L+1}^{\top}\rv_{op}
        \lv x_{k-1,s}^{V\top}\rv\\ 
        & \overset{(ii)}{\le} g_s(W) \lv V\rv^{L+1} \lv  W_{L+1}^{\top}-V_{L+1}^{\top}\rv_{op}
        \lv x_{k-1,s}^{V\top}\rv\\
        & \le g_s(W)\lv V\rv^{L+1}\lv  V_{\ell}-W_{\ell}\rv_{op}
        \lv x_{k-1,s}^{V\top}\rv\\
        & \overset{(iii)}{\le} g_s(W)\lv V\rv^{L+1}\lv  V_{\ell}-W_{\ell}\rv_{op}
        \lv V \rv^{L+1} \\
        & = g_s(W)\lv V\rv^{2(L+1)}\lv  V_{\ell}-W_{\ell}\rv_{op}
        \\
        & \overset{(iv)}{\le} 2J_s(V)\lv V\rv^{2(L+1)}\lv  V_{\ell}-W_{\ell}\rv_{op}
\end{align*}
where $(i)$ follows since $\lv \Sigma_{k,s}^V \rv_{op}\le 1$, $(ii)$ is by invoking Lemma~\ref{l:prod.by.sum}, $(iii)$ follows due to Lemma~\ref{l:u.x.bounds}, and $(iv)$ is because $g_s(W)\le J_s(W)$ by Lemma~\ref{l:relationsbetweengradientandloss} and by the assumption $J_s(W)\le 2J_s(V)$.

Moving on to $\varheart_2$,
\begin{align*}
    \varheart_2 & = \left\lv(g_s(W)-g_s(V))\left( \Sigma^V_{k,s} \prod_{j = k+1}^L 
             \left(  V_{j}^{\top} \Sigma^V_{j,s} \right)
           \right)V_{L+1}^{\top}x_{k-1,s}^{V \top} \right\rv_{op}\\
           &\le \lvert g_s(W)-g_s(V) \rvert \lv \Sigma^V_{k,s} \rv_{op} \left(\prod_{j=k+1}^{L}  \lv V_{j}^{\top}\rv_{op} \lv \Sigma^V_{j,s}\rv_{op}\right)\lv V_{L+1}^{\top} \rv_{op} \lv x_{k-1,s}^{V\top}\rv \\
           &\le \lvert g_s(W)-g_s(V) \rvert
           \left( \prod_{j=k+1}^{L+1}  \lv V_{j}^{\top}\rv_{op}  \right) \lv x_{k-1,s}^{V\top}\rv \\
           &\overset{(i)}{\le} \lvert g_s(W)-g_s(V) \rvert \lv V\rv^{L+1} \lv x_{k-1,s}^{V\top}\rv \\
           &\overset{(ii)}{\le} \lvert g_s(W)-g_s(V) \rvert \lv V\rv^{2(L+1)}  \\
           &\overset{(iii)}{\le} 2J_s(V) \lv V_{\ell}-W_{\ell}\rv_{op}\lv V\rv^{3(L+1)},
\end{align*}
where $(i)$ follows by Lemma~\ref{l:prod.by.sum}, $(ii)$ is by Lemma~\ref{l:u.x.bounds} and $(iii)$ is by invoking Lemma~\ref{l:g_lipschitz}.
Combining the bounds on $\spadesuit_2$
and $\varheart_2$ along with \eqref{e:xi_2_decomp} we find that
\begin{align*}
    \Xi_2 & \le 2 J_s(V) \lv V \rv^{2(L+1)} \lv V_{\ell} - W_{\ell}\rv_{op} + 2J_s(V) \lv V \rv^{3(L+1)}\lv V_{\ell} - W_{\ell}\rv_{op} \\&\le 4 J_s(V) \lv V \rv^{3(L+1)}\lv V_{\ell} - W_{\ell}\rv_{op},
\end{align*}
where the previous inequality follows since $\lv V \rv > 1$.
\end{proof}
Finally we bound $\Xi_3$ which as defined in the proof of Lemma~\ref{l:nabla.smooth.one.layer}.
\begin{lemma} Borrowing the setting and notation of Lemma~\ref{l:nabla.smooth.one.layer}, if $\Xi_3$ is as defined in \eqref{e:def_of_xi_1_2_3} then
\begin{align*}
    \Xi_3 &\le \frac{56 J_s(V)\lv V \rv^{3L+5}\lv V_{\ell}-W_{\ell}\rv}{h}.
\end{align*}
\label{l:bound_on_xi_3}
\end{lemma}
\begin{proof}  Since $g_s(W) \le J_s(W)$ and $J_s(W)\le 2J_s(V)$ (by assumption) we have that
\begin{align*}
    \Xi_3 & = \left\lv g_s(W)
        \left( \Sigma^W_{k,s} \prod_{j = k+1}^L 
             \left(  W_{j}^{\top} \Sigma^W_{j,s} \right)-\Sigma^V_{k,s} \prod_{j = k+1}^L 
             \left(  V_{j}^{\top} \Sigma^V_{j,s} \right)
           \right)
           W_{L+1}^{\top}
        x_{k-1,s}^{W\top} \right\rv_{op}\\
        & \le J_s(W) \left\lv
        \left( \Sigma^W_{k,s} \prod_{j = k+1}^L 
             \left(  W_{j}^{\top} \Sigma^W_{j,s} \right)-\Sigma^V_{k,s} \prod_{j = k+1}^L 
             \left(  V_{j}^{\top} \Sigma^V_{j,s} \right)
           \right)
           W_{L+1}^{\top}
        x_{k-1,s}^{W\top} \right\rv_{op}\\
        & \le 2J_s(V) \lv W_{L+1}^{\top}\rv_{op} \lv x_{k-1,s}^{W \top} \rv \left\lv \Sigma^W_{k,s} \prod_{j = k+1}^L 
             \left(  W_{j}^{\top} \Sigma^W_{j,s} \right)-\Sigma^V_{k,s} \prod_{j = k+1}^L 
             \left(  V_{j}^{\top} \Sigma^V_{j,s} \right)\right\rv_{op} \\
        & = 2J_s(V) \lv W_{L+1}^{\top}\rv_{op} \lv x_{k-1,s}^{W \top} \rv \left\lv \Sigma_{k,s}^W \prod_{j = k+1}^L 
             \left(  W_{j}^{\top} \Sigma^W_{j,s} \right)-
     (\Sigma^V_{k,s}-\Sigma^{W}_{k,s}+\Sigma^{W}_{k,s}) 
             \prod_{j = k+1}^L 
             \left(  V_{j}^{\top} \Sigma^V_{j,s} \right)\right\rv_{op} \\
             & \le \underbrace{2J_s(V) \lv W_{L+1}^{\top}\rv_{op} \lv x_{k-1,s}^{W \top} \rv \left\lv \Sigma_{k,s}^W \left(\prod_{j = k+1}^L 
             \left(  W_{j}^{\top} \Sigma^W_{j,s} \right)-\prod_{j = k+1}^L 
             \left(  V_{j}^{\top} \Sigma^V_{j,s} \right)\right)\right\rv_{op}}_{=:\spadesuit_3}\\&\qquad +\underbrace{2J_s(V) \lv W_{L+1}^{\top}\rv_{op} \lv x_{k-1,s}^{W \top} \rv \left\lv(\Sigma^V_{k,s}-\Sigma^{W}_{k,s}) \prod_{j = k+1}^L 
             \left(  V_{j}^{\top} \Sigma^V_{j,s} \right)\right\rv_{op}}_{=:\clubsuit_3}.\numberthis
          \label{e:xi_3_decomp}
\end{align*}

Before we bound $\spadesuit_3$ and $\clubsuit_3$, let us establish a few useful bounds. First note that for any layer $j$ by Lemma~\ref{l:phi.prime.close}
\begin{align}
    \lv \Sigma_{j,s}^V - \Sigma_{j,s}^W \rv_{op}\le \frac{\lv V_{\ell}-W_{\ell}\rv_{op}\lv V \rv^{L+1}}{h}. \label{e:sigma_diff_bound}
\end{align}
Also we know that 
\begin{align*}
    \lv x_{k-1,s}^{W\top} \rv  \le \lv x_{k-1,s}^{V\top} \rv + \lv x_{k-1,s}^{W\top} -x_{k-1,s}^{V\top}\rv
    & \le \lv V \rv^{L+1} + \lv V \rv^{L+1} \lv V_{\ell}-W_{\ell}\rv_{op} \\
    & \le \lv V \rv^{L+1}\left(1+ \lv V_{\ell}-W_{\ell}\rv_{op} \right)\\
    & \le 2\lv V \rv^{L+1}. \numberthis \label{e:x_w_bound}
\end{align*}
Finally, 
\begin{align*}
    \lv W_{L+1} \rv_{op} \le \lv V_{L+1} \rv_{op} +\lv V_{L+1}-W_{L+1} \rv_{op} &\le \lv V \rv + \lv V_{\ell}-W_{\ell} \rv_{op} \\
    &\le 2\lv V \rv, \numberthis \label{e:w_L+1:bound}
\end{align*}
where the last inequality follows by our assumptions that 
$\lv V_{\ell}-W_{\ell}\rv_{op}\le 1$
and $\lv V\rv > 1$.

With these bounds in place we are ready to bound $\spadesuit_3$:
\begin{align*}
    \spadesuit_3 & = 2J_s(V) \lv W_{L+1}^{\top}\rv_{op} \lv x_{k-1,s}^{W \top} \rv \left\lv \Sigma_{k,s}^W \left(\prod_{j = k+1}^L 
             \left(  W_{j}^{\top} \Sigma^W_{j,s} \right)-\prod_{j = k+1}^L 
             \left(  V_{j}^{\top} \Sigma^V_{j,s} \right)\right)\right\rv_{op} \\
             & \le 2J_s(V) \lv W_{L+1}^{\top}\rv_{op} \lv x_{k-1,s}^{W \top} \rv \left\lv \Sigma_{k,s}^W \right\rv\left\lv\prod_{j = k+1}^L 
             \left(  W_{j}^{\top} \Sigma^W_{j,s} \right)-\prod_{j = k+1}^L 
             \left(  V_{j}^{\top} \Sigma^V_{j,s} \right)\right\rv_{op} \\
             & \overset{(i)}{\le} 8J_s(V) \lv V\rv^{L+2} \left\lv\prod_{j = k+1}^L 
             \left(  W_{j}^{\top} \Sigma^W_{j,s} \right)-\prod_{j = k+1}^L 
             \left(  V_{j}^{\top} \Sigma^V_{j,s} \right)\right\rv_{op}\\
             & \overset{(ii)}{\le} 12J_s(V) \lv V\rv^{L+2}\left(\lv V \rv + \lv V-W \rv \right)^{L+1}\left(\left(\frac{\lv V_{\ell}-W_{\ell}\rv_{op}\lv V \rv^{L+1}}{h} \right)\lv V \rv + \lv V- W \rv \right)\\
    & = \frac{12J_s(V) \lv V\rv^{3L+5} \lv V - W \rv}{h}\left(1+ \frac{\lv V-W \rv}{\lv V\rv} \right)^{L+1}\left(\frac{\lv V_{\ell}-W_{\ell}\rv_{op}}{\lv V-W\rv} + \frac{h}{\lv V\rv^{L+2}} \right)\\
             & \overset{(iii)}{\le} \frac{24J_s(V) \lv V\rv^{3L+5}\lv V-W\rv}{h}\left(1+ \frac{\lv V-W \rv}{\lv V\rv} \right)^{L+1}\\
             & \overset{(iv)}{\le} \frac{24J_s(V) \lv V\rv^{3L+5}\lv V-W\rv}{h}\left(1+ \frac{2(L+1)\lv V-W \rv}{\lv V\rv} \right)\\
             & \overset{(v)}{\le} \frac{48J_s(V) \lv V\rv^{3L+5}\lv V-W\rv}{h}\numberthis\label{e:spadebound}
\end{align*}
where $(i)$ follows by using the bounds in \eqref{e:x_w_bound} and \eqref{e:w_L+1:bound}, $(ii)$ follows by invoking Lemma~\ref{l:ABCD} and using \eqref{e:sigma_diff_bound}, $(iii)$ follows since $h\le 1$ and $\lv V\rv \ge 1$ by assumption, and therefore
$$\frac{\lv V_{\ell}-W_{\ell}\rv_{op}}{\lv V-W\rv}+\frac{h}{\lv V\rv^{L+2}}\le 2,$$
inequality~$(iv)$ follows since for any $0<z< \frac{1}{L}$, $(1+z)^{L+1}\le 1+2(L+1)z$ and because by assumption $\lv V - W \rv \le \lv V \rv/(2(L+1))$, and finally $(v)$ is again because $\lv V - W \rv \le \lv V \rv/(2(L+1))$.

Let's turn our attention to $\clubsuit_3$.
\begin{align*}
    \clubsuit_3 & = 2J_s(V) \lv W_{L+1}^{\top}\rv_{op} \lv x_{k-1,s}^{W \top} \rv \left\lv(\Sigma^W_{k,s}-\Sigma^{V}_{k,s}) \prod_{j = k+1}^L 
             \left(  V_{j}^{\top} \Sigma^V_{j,s} \right)\right\rv_{op}\\
             & \overset{(i)}{\le} 8J_s(V) \lv V \rv^{L+2}\left\lv(\Sigma^W_{k,s}-\Sigma^{V}_{k,s}) \prod_{j = k+1}^L 
             \left(  V_{j}^{\top} \Sigma^V_{j,s} \right)\right\rv_{op}\\
             & \le 8J_s(V) \lv V \rv^{L+2}\left\lv \Sigma^W_{k,s}-\Sigma^{V}_{k,s}\right\rv_{op} \prod_{j = k+1}^L 
             \lv  V_{j}^{\top}\rv_{op}\lv \Sigma^V_{j,s} \rv_{op}\\
             & \overset{(ii)}{\le} 8J_s(V) \lv V \rv^{2L+3}\left\lv \Sigma^W_{k,s}-\Sigma^{V}_{k,s}\right\rv_{op} \\
             & \overset{(iii)}{\le} \frac{8J_s(V) \lv V \rv^{3L+4}\left\lv V_{\ell}-W_{\ell}\right\rv_{op}}{h}, \numberthis \label{e:clubbound}
\end{align*}
where $(i)$ follows from the bounds in \eqref{e:x_w_bound} and \eqref{e:w_L+1:bound}, $(ii)$ follows by invoking Lemma~\ref{l:prod.by.sum} and $(iii)$ is by inequality~\eqref{e:sigma_diff_bound}.

By combining the bounds in \eqref{e:spadebound} and \eqref{e:clubbound} we have a bound on $\Xi_3$. 
\begin{align*}
    \Xi_3 &\le \frac{48J_s(V) \lv V\rv^{3L+5}}{h}+\frac{8J_s(V) \lv V \rv^{3L+4}\left\lv V_{\ell}-W_{\ell}\right\rv_{op}}{h}\\
    &\le \frac{56 J_s(V)\lv V \rv^{3L+5}\lv V-W\rv}{h}\\
    &= \frac{56 J_s(V)\lv V \rv^{3L+5}\lv V_{\ell}-W_{\ell}\rv}{h},
\end{align*}
which completes the proof.
\end{proof}

Lemma~\ref{l:nabla.smooth.one.layer} provides a bound on the norm of the difference between $\nabla_V J_s(V)$ and $\nabla_V J_s(W)$, when the weight matrices $V$ and $W$ differ only at a single layer. 
We next invoke Lemma~\ref{l:nabla.smooth.one.layer} $(L+1)$ times to bound the norm of the difference between the gradients of the loss at $V$ and $W$ when they potentially differ in all of the layers.
\begin{lemma}
\label{l:nabla.smooth}
Let $h \le 1$, and consider 
$V = (V_1,\ldots,V_{L+1})$ and $W = (W_1,\ldots,W_{L+1})$,
such that the following are satisfied for all $j \in [L+1]$:
\begin{itemize}
\item $\lv V-W \rv \le \frac{\lv V\rv}{6L+10}$;
\item $\lv V \rv > \sqrt{L+1/2}$ and $\lv W \rv > \sqrt{L+1/2}$.
\end{itemize}
For every $j \in \{0,\ldots,L+1\}$ define 
$T(j):= (W_1,W_2,\ldots,W_{j},V_{j+1},\ldots,V_{L+1})$.
Suppose that for all $j\in [L+1]$, for all examples $s$, and for all convex combinations $\tW$ of $T(j)$ and
$T(j+1)$, $J_s(\tW) \leq 2 J_s(T(j))\le 4J_s(V)$.
Then
\[
\lv \nabla_V J(V) - \nabla_W J(W)  \rv
   \leq \frac{256(L+1)\sqrt{p}J(V)\lv V \rv^{3L+5}\lv V-W\rv}{h}.
\]
\end{lemma}
\begin{proof}
We may transform $V$ into $W$ by swapping one layer at
a time. For any $s\in [n]$ Lemma~\ref{l:nabla.smooth.one.layer} bounds the norm of difference in each swap, thus,
\begin{align*}
    \lv \nabla_{V}J_s(V) - \nabla_W J_s(W) \rv &= \left\lv \sum_{k=0}^{L} \left(\nabla_{T(k)}J_s(T(k))-\nabla_{T(k+1)}J_s(T(k+1)) \right)\right\rv\\
    &\le \sum_{k=0}^{L}\left\lv  \nabla_{T(k)}J_s(T(k))-\nabla_{T(k+1)}J_s(T(k+1)) \right\rv\\
    &\le \sum_{k=1}^{L+1}\frac{64 \sqrt{(L+1)p}J_s(T(k))\lv T(k)\rv^{3L+5}\lv V_{k}-W_{k}\rv}{h}\\
    &= \frac{64\sqrt{(L+1)p}}{h}\sum_{k=1}^{L+1}J_s(T(k))\lv T(k)\rv^{3L+5}\lv V_{k}-W_{k}\rv\\
    &\le \frac{128\sqrt{(L+1)p}J_s(V)}{h}\sum_{k=1}^{L+1}\lv T(k)\rv^{3L+5}\lv V_{k}-W_{k}\rv, \numberthis \label{e:total_grad_bound_midway}
\end{align*}
where the final inequality follows from the assumption that $J_s(T(k))\le 2J_s(V)$.
For any $k \in [L+1]$ 
\begin{align*}
    \lv T(k)\rv^{3L+5}  = \lv V \rv^{3L+5}\left( \frac{\lv T(k)\rv}{\lv V \rv}\right)^{3L+5}
    & = \lv V \rv^{3L+5}\left( \frac{\lv T(k)-V+V\rv}{\lv V \rv}\right)^{3L+5}\\
    & \le \lv V \rv^{3L+5}\left( 1+\frac{\lv T(k)-V\rv}{\lv V \rv}\right)^{3L+5}\\
    & \le \lv V \rv^{3L+5}\left( 1+\frac{\lv W-V\rv}{\lv V \rv}\right)^{3L+5}\\
    & \overset{(i)}{\le} \lv V \rv^{3L+5}\left( 1+\frac{(3L+5)\lv W-V\rv}{\lv V \rv}\right)\\
    & \overset{(ii)}{\le} 2\lv V \rv^{3L+5},
\end{align*}
where $(i)$ follows since for any non-negative $z < \frac{1}{3L+5}$, $(1+z)^{3L+5}\le 1+(6L+10)z$ and because by assumption $\lv V-W \rv/\lv V\rv \le \frac{1}{6L+10}$, and $(ii)$ again follows by our assumption that $\lv V-W \rv/\lv V\rv \le \frac{1}{6L+10}$. Using this bound in inequality~\eqref{e:total_grad_bound_midway}
\begin{align*}
    \lv \nabla_{V}J_s(V) - \nabla_W J_s(W) \rv &\le \frac{256\sqrt{(L+1)p}J_s(V)\lv V \rv^{3L+5}}{h}\sum_{k=1}^{L+1}\lv V_{k}-W_{k}\rv\\
     &\le \frac{256\sqrt{(L+1)p}J_s(V)\lv V \rv^{3L+5}}{h}\left(\sqrt{L+1} \lv V- W \rv \right)\\
     &= \frac{256(L+1)\sqrt{p}J_s(V)\lv V \rv^{3L+5}\lv V-W\rv}{h}.
\end{align*}
Thus,
\begin{align*}
     \lv \nabla_{V}J(V) - \nabla_W J(W) \rv & = \left\lv \frac{1}{n}\sum_{s\in [n]} \nabla_{V}J_s(V) - \nabla_W J_s(W) \right\rv\\
      & \le \frac{1}{n}\sum_{s\in [n]}\left\lv  \nabla_{V}J_s(V) - \nabla_W J_s(W) \right\rv\\
      & \le \frac{1}{n}\sum_{s\in [n]}\frac{256(L+1)\sqrt{p}J_s(V)\lv V \rv^{3L+5}\lv V-W\rv}{h} \\
      &= \frac{256(L+1)\sqrt{p}J(V)\lv V \rv^{3L+5}\lv V-W\rv}{h}
\end{align*}
completing the proof.
\end{proof}

\hessianlemma*
\begin{proof}
Since the function $J(\cdot)$ is continuous, for all 
close enough $W$
the assumptions of Lemma~\ref{l:nabla.smooth} are satisfied.
\end{proof}


\subsection{Proof of Lemma~\ref{l:gradient_norm.upper}}
\label{a:gradient_norm.upper}
\gradientupper*
\begin{proof}
For any $\ell \in [L]$ the formula for the gradient of the loss with respect to $V_{\ell}$  is given by (see equation~\eqref{e:gradient_ell_inner_layers})
\begin{align*}
    \frac{\partial J(V;x_s,y_s)}{\partial V_{\ell}} & = g_s(V)\left( \Sigma^V_{\ell,s} \prod_{j = \ell+1}^L 
             \left(  V_{j}^{\top} \Sigma^V_{j,s} \right)
           \right)
           V_{L+1}^{\top}
        x_{\ell-1,s}^{V\top},
\end{align*}
therefore its operator norm 
\begin{align*}
    \left\lv  \frac{\partial J(V;x_s,y_s)}{\partial V_{\ell}}\right\rv_{op} & = g_s(V)\left\lv\left( \Sigma^V_{\ell,s} \prod_{j = \ell+1}^L 
             \left(  V_{j}^{\top} \Sigma^V_{j,s} \right)
           \right)
           V_{L+1}^{\top}
        x_{\ell-1,s}^{V\top} \right\rv_{op}\\
        & \le g_s(V) \lv \Sigma^V_{\ell,s}\rv_{op} \left(\prod_{j = \ell+1}^L 
            \lv V_{j}^{\top}\rv_{op} \lv\Sigma^V_{j,s} \rv_{op}\right)
           \lv V_{L+1}^{\top}\rv_{op}
        \lv x_{\ell-1,s}^{V\top} \rv\\
        & \le g_s(V)  \left(\prod_{j = \ell+1}^{L+1} 
            \lv V_{j}\rv_{op}\right) 
        \lv x^V_{\ell-1,s} \rv \numberthis \label{e:gradient_upper_bound_midway}
\end{align*}
where the last step follows since $\lv \Sigma^V_{j,s}\rv_{op} \le \max_{z}|\phi'(z)| < 1$. By its definition 
\begin{align*}
    \lv x_{\ell-1,s}^V \rv  = \left\lv \phi\left( V_{\ell-1}\phi\left( \cdots \phi(V_1 x_s)\right)\right)\right\rv 
    & \overset{(i)}{\le} \left\lv  V_{\ell-1}\phi\left( \cdots \phi(V_1 x_s)\right)\right\rv \\
    & \le \left\lv  V_{\ell-1}\right\rv_{op}\left\lv\phi\left( \cdots \phi(V_1 x)\right)\right\rv \\
    & \le \left(\prod_{j=1}^{\ell-1}\left\lv  V_{j}\right\rv_{op}\right) \lv x_s \rv \overset{(ii)}{\le}  \left(\prod_{j=1}^{\ell-1}\left\lv  V_{j}\right\rv_{op}\right)  ,
\end{align*}
where $(i)$ follows since $\phi$ is contractive 
(Lemma~\ref{l:aux_contractive_maps})
and $(ii)$ is because $\lv x_s \rv = 1$. Along with inequality~\eqref{e:gradient_upper_bound_midway} this implies
\begin{align*}
      \left\lv  \frac{\partial J(V;x_s,y_s)}{\partial V_{\ell}}\right\rv_{op} & \le g_s(V) \prod_{j \neq \ell} \lv V_{j} \rv_{op}\le g_s(V)\prod_{j \neq \ell} \lv V_{j} \rv \le g_s(V) \lv V \rv^{L+1},
\end{align*}
where the last inequality follows from Lemma~\ref{l:prod.by.sum}. Therefore we have
\begin{align*}
      \left\lv  \frac{\partial J(V)}{\partial V_{\ell}}\right\rv_{op}  = \left\lv  \frac{1}{n} \sum_{s\in [n]}\frac{\partial J(V;x_s,y_s)}{\partial V_{\ell}}\right\rv_{op} \le \frac{1}{n}\sum_{s \in [n]} \left\lv  \frac{\partial J(V;x_s,y_s)}{\partial V_{\ell}}\right\rv_{op} & \le \frac{\lv V \rv^{L+1}}{n}\sum_{s\in [n]} g_s(V).
\end{align*}
We know that $g_s(V) \le J_s(V)$ by Lemma~\ref{l:relationsbetweengradientandloss} and also that $g_s(V)<1$. Therefore,
\begin{align*}
     \left\lv  \frac{\partial J(V)}{\partial V_{\ell}}\right\rv_{op} & \le \frac{\lv V \rv^{L+1}}{n} \min\left\{\sum_{s}J_s(V),n \right\}  \le \lv V \rv^{L+1} \min\left\{ J(V),1\right\}.
\end{align*}
Given that $V_{\ell}$ is a $p \times p$ matrix we infer
\begin{align} \label{e:gradient_upper_ell_less_than_Lplusone}
    \left\lv  \frac{\partial J(V)}{\partial V_{\ell}}\right\rv & \le \sqrt{p}\left\lv  \frac{\partial J(V)}{\partial V_{\ell}}\right\rv_{op} \le \sqrt{p}\lv V \rv^{L+1} \min\left\{ J(V),1\right\}.
\end{align}
When $\ell = L+1$
\begin{align*}
     \frac{\partial J(V;x_s,y_s)}{\partial V_{L+1}} & = g_s(V)
        x_{L,s}^{V\top},
\end{align*}
by using the same chain of logic as in the case of $\ell < L+1$ we can obtain the bound
\begin{align*}
 \left\lv  \frac{\partial J(V)}{\partial V_{\ell}}\right\rv & \le    \sqrt{p}\lv V \rv^{L+1} \min\left\{ J(V),1\right\}.
\end{align*}
Summing up over all layers
\begin{align*}
    \lv \nabla J(V) \rv^2 & = \sum_{\ell = 1}^{L+1} \left\lv  \frac{\partial J(V)}{\partial V_{\ell}}\right\rv^2  \le (L+1)p \lv V \rv^{2(L+1)}\left( \min\left\{ J(V),1\right\}\right)^2,
\end{align*}
hence, taking squaring roots completes the proof.
\end{proof}

\subsection{Proof of Lemma~\ref{l:L.one_step_improvement}}\label{a:L.one_step_improvement}
\onesteplemma*
\begin{proof} Since, by assumption, $J_t <\frac{1}{n^{1+24L}}$,
Lemma~\ref{l:aux.lower.bound.norm.weight.vector}
implies $\lv V^{(t)}\rv > \sqrt{L+1}$. We would like to apply Lemmas~\ref{l:L_smooth} and \ref{l:gradient_norm.upper}. To apply these lemmas, we first bound the norm of all convex combinations of $V^{(t)}$ and $V^{(t+1)}$ from above and below. 
Consider $W =\eta V^{(t)}+(1-\eta)V^{(t+1)}= V^{(t)}-(1-\eta)\alpha\nabla J_t$
for any $\eta \in [0,1]$. An upper bound on the norm raised to the $3L+5$th power is
\begin{align*}
    \lv W \rv^{3L+5} = \lv V^{(t)}-(1-\eta)\alpha \nabla J_t \rv^{3L+5} 
    &=\lv V^{(t)}\rv^{3L+5} \left(\frac{\lv V^{(t)}-(1-\eta)\alpha \nabla J_t \rv}{\lv V^{(t)}\rv} \right)^{3L+5}\\
    &\le\lv V^{(t)}\rv^{3L+5} \left(\frac{\lv V^{(t)}\rv+ \alpha \lv\nabla J_t \rv}{\lv V^{(t)}\rv} \right)^{3L+5}\\
    &=\lv V^{(t)}\rv^{3L+5} \left(1+\frac{\alpha \lv\nabla J_t \rv}{\lv V^{(t)}\rv} \right)^{3L+5}\\
    &\overset{(i)}{\le}\lv V^{(t)}\rv^{3L+5} \left(1+\frac{\alpha (\sqrt{(L+1)p} J_t \lv V^{(t)}\rv^{L+1})}{\lv V^{(t)}\rv} \right)^{3L+5}\\
    &=\lv V^{(t)}\rv^{3L+5} \left(1+\alpha (\sqrt{(L+1)p} J_t \lv V^{(t)}\rv^{L}) \right)^{3L+5}\\
    &\overset{(ii)}{\le} \lv V^{(t)}\rv^{3L+5}\left(1+(6L+10)(\sqrt{(L+1)p} \alpha  J_t \lv V^{(t)}\rv^{L} \right) \\
    &\le 2\lv V^{(t)}\rv^{3L+5} \label{e:bound_on_L_th_power_vt+1} \numberthis
\end{align*}
where $(i)$ follows by invoking Lemma~\ref{l:gradient_norm.upper} and $(ii)$ follows since for any $0<z<1/(3L+5)$, $(1+z)^{3L+5} \le 1+(6L+10)z$ and because the step-size $\alpha$ is chosen such that \begin{align*}\alpha (\sqrt{(L+1)p} J_t \lv V^{(t)}\rv^{L})& \le \frac{h}{1024 (L+1)^2 \sqrt{p}J_t\lv V^{(t)}\rv^{3L+5}} \cdot\sqrt{(L+1)p} J_t \lv V^{(t)}\rv^{L}\\
& = \frac{h}{1024 (L+1)^{3/2} \lv V^{(t)}\rv^{2L+5}} \le\frac{1}{6L+10}.\end{align*} Thus, we have shown that the norm of $W \in [V^{(t)},V^{(t+1)}]$ raised to the $3L+5$th power is bounded by $2\lv V^{(t)}\rv^{3L+5}$. Next we lower bound the norm of $W$,
\begin{align*}
    \lv W \rv = \lv V^{(t)} -(1-\eta)\alpha \nabla J_t \rv &\ge \lv V^{(t)}\rv\left(1-\frac{\alpha \lv \nabla J_t\rv}{\lv V^{(t)}\rv}\right)\\ &\overset{(i)}{\ge} \lv V^{(t)}\rv\left(1-\alpha \sqrt{(L+1)p}J_t \lv V^{(t)}\rv^L \right) \\
    &\overset{(ii)}{>} \sqrt{L+1}\left(1-\alpha \sqrt{(L+1)p}J_t \lv V^{(t)}\rv^L \right) \\
    &\overset{(iii)}{>} \sqrt{L+1}\left(1-\frac{1}{6L+10} \right)\\
    &> \sqrt{L + 1/2},
\end{align*}
where $(i)$ follows by again invoking Lemma~\ref{l:gradient_norm.upper}, $(ii)$ is by Lemma~\ref{l:aux.lower.bound.norm.weight.vector} that guarantees that 
$\lv V^{(t)}\rv >  \sqrt{L+1}$ 
since $J_t < \frac{1}{n^{1+24L}}$ and $(iii)$ is by the logic above that guarantees that 
\begin{align*}
\alpha\left(\sqrt{(L+1)p}J_t \lv V^{(t)}\rv^L\right)\le \frac{1}{6L+10}.
\end{align*}
Thus we have also shown that $\lv W \rv >\sqrt{L + 1/2}$ for any $W \in [V^{(t)},V^{(t+1)}]$.

In order to apply Lemma~\ref{l:pl} (that shows that the loss decreases along a gradient step when the loss is smooth along the path), we would like to bound
$\Lip(\nabla_W J(W))$
for all convex combinations $W$
of $V^{(t)}$ and $V^{(t+1)}$.  For 
$N = \ceil{\frac{2\sqrt{(L+1)p}\lv V^{(t)}\rv^{L+1}\lv V^{(t+1)}-V^{(t)}\rv}{J_t}}$, 
(similarly to the proof of Lemma E.8
of \citep{lyu2019gradient})
we will prove the following by induction
\begin{quote}
For all $s \in \{ 0,\ldots,N\}$, for all $\eta \in [0,s/N]$,
for $W = \eta V^{(t+1)} + (1 - \eta) V^{(t)}$,
$\Lip(\nabla_W J(W))
\leq \frac{1024(L+1)\sqrt{p}J_t\lv V^{(t)} \rv^{3L+5}}{h}$.
\end{quote}
The base case, where $s = 0$, follows directly from Lemma~\ref{l:L_smooth}.  Now,
assume that the inductive hypothesis holds from some $s$, and,
for $\eta \in (s/N,(s+1)/N]$, consider
$W = \eta V^{(t+1)} + (1 - \eta) V^{(t)}$.
Let $\tW = (s/N) V^{(t+1)} + (1 - s/N) V^{(t)}$.
Since the step-size $\alpha$ is small enough, applying Lemma~\ref{l:pl} along with the inductive hypothesis yields
$J(\tW) \leq J_t$.
Applying Lemma~\ref{l:gradient_norm.upper} (which provides a bound on the Lipschitz constant of $J$)
\begin{align*}
J(W) & \leq J(\tW) + (\sqrt{(L+1)p} \max_{\bar{W}\in [W,\tW]}\lv \bar{W}\rv^{L+1} )\lv W- \tW  \rv \\
 & \overset{(i)}{\leq} J(\tW) + (2\sqrt{(L+1)p}) \lv V^{(t)}\rv^{L+1}\lv W- \tW  \rv \\
    &\le J(\tW) + \frac{(2\sqrt{(L+1)p}) \lv V^{(t)}\rv^{L+1}\lv V^{(t+1)}-V^{(t)}\rv}{N} \\
    &= J(\tW) + J_t \\ &\le 2J_t,
\end{align*}
where $(i)$ follows since $\max_{\bar{W}\in [W,\tW]}\lv \bar{W}\rv^{L+1} \le 2 \lv V^{(t)}\rv^{L+1}$ by using the same logic used to arrive at inequality~\eqref{e:bound_on_L_th_power_vt+1}.
Applying 
Lemmas~\ref{l:aux.lower.bound.norm.weight.vector}
and~\ref{l:L_smooth}, 
this implies
that for any $W \in [V^{(t)},V^{(t+1)}]$
$$
\Lip(\nabla_W J(W))
\leq \frac{256(L+1)\sqrt{p}J(W)\lv W \rv^{3L+5}}{h}\le
\frac{1024(L+1)\sqrt{p}J_t\lv V^{(t)} \rv^{3L+5}}{h},$$ completing the proof of the inductive step.

So, now we know that, for all convex combinations
$W$ of $V^{(t)}$ and $V^{(t+1)}$, 
$
\Lip(\nabla_W J(W))
\leq \frac{1024(L+1)\sqrt{p}J_t\lv V^{(t)} \rv^{3L+5}}{h}$. By our choice of step size $\alpha<\frac{1}{L+\frac{1}{2}}\cdot \frac{h}{1024(L+1)\sqrt{p}J_t\lv V^{(t)} \rv^{3L+5}}$,
so by applying
Lemma~\ref{l:pl}, we have that
\begin{align*}
J_{t+1} 
& \le J_t - \frac{L}{L+\frac{1}{2}}\alpha  \lv \nabla J_t \rv^2
\end{align*}
which is the desired result.
\end{proof}

\subsection{Proof of Lemma~\ref{l:lower.bound.gradient}}\label{a:gradient_lower_bound_lemma}
\gradientlowerboundlemma*
\begin{proof}We have
\begin{align} \nonumber
    \lv \nabla J_t \rv = \sup_{a: \lv a \rv =1} \left(\nabla J_t \cdot a\right) &\ge (\nabla J_t) \cdot \left( \frac{-V^{(t)}}{\lv V^{(t)}\rv}\right)\\ &= \frac{1}{\lv V^{(t)}\rv} \sum_{\ell \in [L+1]} \nabla_{V_{\ell}} J_t \cdot \left(-V^{(t)}_{\ell}\right) .\label{e:lower_bound_norm}
\end{align}

Note that by definition,
\begin{align} \label{e:lower_bound_decomposition}
    \nabla_{V_{\ell}} J_t \cdot \left(-V^{(t)}_{\ell}\right) = \frac{1}{n}\sum_{s\in [n]} \nabla_{V_{\ell}} J_{ts} \cdot \left(-V^{(t)}_{\ell}\right). 
\end{align}

Consider two cases. 

\paragraph{Case 1:} (When $\ell = L+1$) In this case, for any $s \in [n]$ by the formula for the gradient in \eqref{e:gradient_outer_layer} we have
\begin{align*}
    \nabla_{V_{L+1}} J_{ts} \cdot \left(-V_{L+1}^{(t)}\right) & = g_{ts}y_s V_{L+1}^{(t)} x_{L,s}^{(t)} = g_{ts}y_s f_{V^{(t)}}(x_s)
\end{align*}
and therefore
\begin{align}
     \nabla_{V_{L+1}} J_t \cdot \left(-V_{L+1}^{(t)}\right) & = \frac{1}{n}\sum_{s\in [n]} g_{ts} y_s f_{V^{(t)}}(x_s). \label{e:lower_bound_norm_final_layer}
\end{align}

\textbf{Case 2:} (When $\ell \in [L]$)  Below we will prove the claim (in Lemma~\ref{l:inner_product_lower_bound_middle_layers}) that for any $\ell \in [L]$
\begin{align} \label{e:claim_inner_product_lower_bound_middle_layers}
    \nabla_{V_{\ell}^{(t)}} J_t\cdot \left(-V_{\ell}^{(t)}\right) & \ge \frac{1}{n}\sum_{s\in [n]} g_{ts} \left[y_s f_{V^{(t)}}(x_s) - \frac{\sqrt{p}h \lv V^{(t)}\rv^{L}}{2L^{\frac{L}{2}-1}}  \right].
\end{align}

By combining this with the results of inequalities \eqref{e:lower_bound_norm} and  \eqref{e:lower_bound_norm_final_layer} \begin{align*}
    \lv \nabla J_t \rv & \ge \frac{L+1}{n \lv V^{(t)} \rv} \sum_{s\in [n]}g_{ts} y_s f_{V^{(t)}}(x_s) - \frac{L\sqrt{p}h\lv V^{(t)} \rv^{L}}{2L^{\frac{L}{2}-1}\lv V^{(t)} \rv}\left[\frac{1}{n}\sum_{s\in [n]}g_{ts}\right]\\
    & \overset{(i)}{\ge} \frac{L+1}{n\lv V^{(t)} \rv} \sum_{s\in [n]}g_{ts} y_s f_{V^{(t)}}(x_s) - \frac{L\sqrt{p}h\lv V^{(t)} \rv^{L}}{2L^{\frac{L}{2}-1}\lv V^{(t)} \rv}J_t\\
    & \overset{(ii)}{\ge} \frac{L+1}{n\lv V^{(t)} \rv} \sum_{s\in [n]}g_{ts} y_s f_{V^{(t)}}(x_s) - \frac{L\sqrt{p}h\lv V^{(1)} \rv^{L} \log(1/J_t)}{2\log(1/J_1)L^{\frac{L}{2}-1}\lv V^{(t)} \rv}J_t \numberthis \label{e:lower_bound_midway_part_2}
\end{align*}
where $(i)$ follows because $g_{ts}\le J_{ts}$ by Lemma~\ref{l:relationsbetweengradientandloss} and $(ii)$ follows by our assumption on $\lv V^{(t)}\rv$.

For every sample $s$, $J_{ts} = \log\left(1+\exp\left(-y_s f_{V^{(t)}} (x_s)\right)\right)$ which implies
\begin{align*}
    y_s f_{V^{(t)}}(x_s) = \log\left( \frac{1}{\exp(J_{ts})-1}\right) \quad \text{and} \quad g_{ts} = \frac{1}{1+\exp(y_s f_{V^{(t)}}(x_s))} = 1-\exp\left( -J_{ts}\right).
\end{align*}
Plugging this into inequality~\eqref{e:lower_bound_midway_part_2} we derive,
\begin{align*}
     \lv \nabla J_t \rv & \ge \frac{L+1}{n\lv V \rv} \sum_{s=1}^n \left(1-\exp\left( -J_{ts}\right)\right)\log\left( \frac{1}{\exp(J_{ts})-1}\right) - \frac{L\sqrt{p}h\lv V^{(1)} \rv^{L}}{2\log(1/J_1)L^{\frac{L}{2}-1}\lv
     V^{(t)} \rv}J_t \log(1/J_t).
\end{align*}
Observe that the function $\left(1-\exp(-z)\right)\log\left(\frac{1}{\exp(z)-1}\right)$ 
is continuous and 
concave (when the inputs lie between $0$ and $1$) with 
\begin{align*}
  \lim_{z \to 0^{+}}(1-\exp(-z))\log\left(\frac{1}{\exp(z)-1}\right)~=~0.  
\end{align*}
Also recall that $\sum_{s} J_{ts} = J_tn$ and that $J_{ts}\le J_tn\le 1/n^{24L}\le 1$. Therefore 
applying Lemma~\ref{l:concave.min}
to the
function $\psi$ with $\psi(0) = 0$ and 
$\psi(z)=\left(1-\exp(-z)\right)\log\left(\frac{1}{\exp(z)-1}\right)$
for $z > 0$, 
we get that
\begin{align}
\nonumber \lv \nabla J_t \rv
 & \ge \frac{L+1}{\lv V^{(t)} \rv}\left[ \frac{ 1 - \exp(-J_{t}n)}{n} 
      \log\left( \frac{1}{\exp(J_{t} n) - 1} \right)\right]\\&\qquad\qquad -\frac{L\sqrt{p}h\lv V^{(1)} \rv^{L}}{2\log(1/J_1)L^{\frac{L}{2}-1}\lv V^{(t)} \rv}J_t \log(1/J_t). \label{e:pretaylor_lowerbound}
\end{align}
We know that for any $z \in [0,1]$
\begin{align*}
\exp(z) \le 1+2z \quad \mbox{and} \quad  \exp(-z) \le 1-z+z^2.
\end{align*}
Since $J_t < \frac{1}{n^{1 + 24L}}$ and $n \geq 3$, these bounds on the exponential function combined with inequality~\eqref{e:pretaylor_lowerbound} yields
\begin{align*}
   & \lv \nabla J_t \rv\\
 & \ge \frac{L+1}{\lv V^{(t)} \rv}\left[ (J_t - nJ_t^2) 
      \log\left( \frac{1}{2J_{t} n} \right)\right]-\frac{L\sqrt{p}h\lv V^{(1)} \rv^{L}}{2\log(1/J_1)L^{\frac{L}{2}-1}\lv V^{(t)} \rv}J_t \log(1/J_t) \\
      & = \frac{(L+1) J_t \log(1/J_t)}{\lv V^{(t)}\rv}\left[1-nJ_t-\frac{\log(2n)}{\log(1/J_t)}-\frac{\sqrt{p}h\lv V^{(1)} \rv^{L}}{2\log(1/J_1)L^{\frac{L}{2}-2}}\right] \\
      & = \frac{(L+3/4) J_t \log(1/J_t)}{\lv V^{(t)}\rv}\left[1+\frac{1}{4(L+\frac{3}{4})}\right]\left[1-nJ_t-\frac{\log(2n)}{\log(1/J_t)}-\frac{\sqrt{p}h\lv V^{(1)} \rv^{L}}{2\log(1/J_1)L^{\frac{L}{2}-2}}\right]. \label{e:lower_bound_norm_midway} \numberthis
\end{align*}
By the choice of $h\le h_{\max}$ we have
\begin{align*}
    \frac{\sqrt{p}h\lv V^{(1)} \rv^{L}}{2\log(1/J_1)L^{\frac{L}{2}-2}} \le \frac{1}{48L}.
\end{align*}
Next, since $J_t < \frac{1}{n^{1+24L}}$ and $n\ge 3$
\begin{align*}
    \frac{\log(2n)}{\log(1/J_t)} & \le \frac{\log(2)+\log(n)}{(1+24L)\log(n)}\le \frac{1}{12L}
\end{align*}
and 
\begin{align*}
    nJ_t < \frac{1}{3^{24L}} \le \frac{1}{48L}.
\end{align*}
Therefore, using these three bounds in conjunction with inequality~\eqref{e:lower_bound_norm_midway} yields the bound
\begin{align*}
    \lv \nabla J_t \rv &\ge \frac{(L+3/4) J_t \log(1/J_t)}{\lv V^{(t)}\rv}\left[1+\frac{1}{4(L+\frac{3}{4})}\right]\left[1-\frac{1}{8L}\right] 
      \ge \frac{(L+\frac{3}{4}) J_t \log(1/J_t)}{\lv V^{(t)}\rv},
\end{align*}
which establishes the desired bound.
\end{proof}

As promised above we now lower bound the inner product between the gradient of the loss with respect to $V_{\ell}^{(t)}$ and the weight matrix for any $\ell \in [L]$. 

\begin{lemma}
\label{l:inner_product_lower_bound_middle_layers}Under the conditions of Lemma~\ref{l:lower.bound.gradient} and borrowing all notation from 
its proof,
for all $\ell \in [L]$
\begin{align*}
    \nabla_{V_{\ell}} J_t \cdot \left(-V_{\ell}^{(t)}\right) & \ge \frac{1}{n}\sum_{s\in [n]} g_{ts} \left[y_s f_{V^{(t)}}(x_s) - \frac{\sqrt{p}h \lv V^{(t)}\rv^{L}}{2L^{\frac{L}{2}-1}}  \right].
\end{align*}
\end{lemma}
\begin{proof}To ease notation, let us drop the $(t)$ in the superscript and refer to $V^{(t)}$ as $V$. Recall that for any matrices $A$ and $B$, $A\cdot B=\vecrm(A)\cdot \vecrm(B) = \Tr(A^{\top}B)$. Also recall the formula for the gradient of the loss in \eqref{e:gradient_ell_inner_layers}, therefore, for any $s\in [n]$
\begin{align*}
    &\nabla_{V_{\ell}} J_{ts}\cdot \left(-V_{\ell}\right) \\
    &=-\Tr\left(V_{\ell}^{\top}\nabla_{V_{\ell}} J_{ts} \right)\\
    &= g_{ts}y_s \Tr\left(
    V_{\ell}^{\top}
        \left( \Sigma_{\ell,s} \prod_{j = \ell+1}^L 
               V_{j}^{\top} \Sigma_{j,s} 
           \right)
           V_{L+1}^{\top}
        x_{\ell-1,s}^{\top} \right)\\
        &= g_{ts}y_s \Tr\left(
          \prod_{j = \ell}^L 
             \left(  V_{j}^{\top} \Sigma_{j,s} 
           \right)
           V_{L+1}^{\top}
        x_{\ell-1,s}^{\top} \right)\\
  & \overset{(i)}{=} g_{ts}y_s \Tr\left(
        x_{\ell-1,s}^{\top}\prod_{j = \ell}^L 
             \left(  V_{j}^{\top} \Sigma_{j,s} 
           \right)
           V_{L+1}^{\top}
       \right)\\
         & = g_{ts}y_s 
        x_{\ell-1,s}^{\top}\prod_{j = \ell}^L 
             \left(  V_{j}^{\top} \Sigma_{j,s} 
           \right)
           V_{L+1}^{\top} \\
        &\overset{(ii)}{=}g_{ts} y_s  x_{L,s}^{\top}V_{L+1}^{\top}
+       g_{ts} y_s \sum_{k=\ell}^{L-1}\left(
        \left(x_{k-1,s}^{\top}V_k^{\top}\Sigma_{k,s}-x_{k,s}^{\top}\right)\prod_{j = k+1}^L 
             \left(  V_{j}^{\top} \Sigma_{j,s} 
           \right)
           V_{L+1}^{\top}
       \right)\\
       & \qquad \qquad  +       g_{ts} y_s 
        \left(x_{L-1,s}^{\top}V_L^{\top}\Sigma_{L,s}-x_{L,s}^{\top}\right)
           V_{L+1}^{\top}\\
        &\overset{(iii)}{=}g_{ts} y_s f_{V^{(t)}}(x_s)
+        g_{ts}y_s \sum_{k=\ell}^{L-1}\left(
        \left(x_{k-1,s}^{\top}V_k^{\top}\Sigma_{k,s}-x_{k,s}^{\top}\right)\prod_{j = k+1}^L 
             \left(  V_{j}^{\top} \Sigma_{j,s} 
           \right)
           V_{L+1}^{\top}
       \right)\\& \qquad \qquad  +       g_{ts} y_s 
        \left(x_{L-1,s}^{\top}V_L^{\top}\Sigma_{L,s}-x_{L,s}^{\top}\right)
           V_{L+1}^{\top},
\end{align*}
where $(i)$ follows by the cyclic property of the trace, and $(ii)$ follows since the second term and third term in the equation form a telescoping sum, and $(iii)$ is because $f_{V^{(t)}}(x_s)= V_{L+1}x_{L,s}$ by definition. By the property of $h$-smoothly approximately ReLU activations, for any $z \in \mathbb{R}$ we know that $|\phi'(z)z-\phi(z)| \le \frac{h}{2}$. Therefore for any $k \in [L]$, $\lv x_{k-1,s}^{\top} V_{k}^{\top}\Sigma_{k,s} - x_{k,s}^{\top}\rv_{\infty} \le \frac{h}{2}$ and hence $\lv x_{k-1,s}^{\top} V_{k}^{\top}\Sigma_{k,s} - x_{k,s}^{\top}\rv \le \frac{\sqrt{p}h}{2}$. Continuing from the previous displayed equation, by applying the Cauchy-Schwarz inequality we find
\begin{align*}
     &\nabla_{V_{\ell}} J_{ts} \cdot \left(-V_{\ell}\right) \\
  &    \ge g_{ts} y_s f_{V^{(t)}}(x_s)
  -        g_{ts} \sum_{k=\ell}^{L-1}
        \lv x_{k-1,s}^{\top}V_k^{\top}\Sigma_{k,s}-x_{k,s}^{\top}\rv\prod_{j = k+1}^L 
             \left\lv  V_{j}^{\top} \right\rv_{op} \left\lv \Sigma_{j,s} \right\rv_{op}
           \left\lv V_{L+1}^{\top}\right\rv\\
          & \qquad \qquad \qquad \qquad \qquad  -        g_{ts} 
        \lv x_{L-1,s}^{\top}V_L^{\top}\Sigma_{L,s}-x_{L,s}^{\top}\rv
           \left\lv V_{L+1}^{\top}\right\rv\\
           &    \ge g_{ts} y_s f_{V^{(t)}}(x_s)
  -     \frac{ \sqrt{p}h  g_{ts}}{2} \sum_{k=\ell}^L
        \prod_{j = k+1}^{L+1}
             \left\lv  V_{j} \right\rv_{op} \left\lv \Sigma_{j,s} \right\rv_{op}\\
             &    \overset{(i)}{\ge} g_{ts} y_s f_{V^{(t)}}(x_s)
  -     \frac{ \sqrt{p}h  g_{ts}}{2} \sum_{k=\ell}^L
        \prod_{j = k+1}^{L+1}
             \left\lv  V_{j} \right\rv_{op}\\
             &  \ge g_{ts} y_s f_{V^{(t)}}(x_s)
  -     \frac{ \sqrt{p} L h  g_{ts}}{2} \max_{k \in [L]}
        \prod_{j = k+1}^{L+1}
             \left\lv  V_{j} \right\rv_{op}\\
                        &  \overset{(ii)}\ge g_{ts} y_s f_{V^{(t)}}(x_s)
  -     \frac{ \sqrt{p} L h  g_{ts}}{2} \max_{k \in [L]}
        \prod_{j = k+1}^{L+1}
             \left\lv  V_{j} \right\rv\\
                        &  \overset{(iii)}\ge g_{ts} y_s f_{V^{(t)}}(x_s)
  -    \frac{\sqrt{p}h \lv V\rv^{L}g_{ts}}{2L^{\frac{L}{2}-1}}
             \numberthis \label{e:lower_bound_midway}
\end{align*}
where $(i)$ follows since $\phi'\le 1$ and therefore $\lv \Sigma_{j,s}\rv_{op}\le 1$, $(ii)$ follows since for any matrix $M$, $\lv M \rv_{op}\le \lv M \rv$ and inequality~$(iii)$ follows by invoking Lemma~\ref{l:aux.lower.bound.norm.weight.vector.part2} since we know that $\lv V\rv > \sqrt{L+1}$ by Lemma~\ref{l:aux.lower.bound.norm.weight.vector}. The previous display along with the decomposition in equation~\eqref{e:lower_bound_decomposition} yields
\begin{align*}
     \nabla_{V_{\ell}} J_t \cdot \left(-V_{\ell}\right) & \ge \frac{1}{n}\sum_{s\in [n]} g_{ts} \left[y_s f_{V^{(t)}}(x_s) - \frac{\sqrt{p}h \lv V\rv^{L}}{2L^{\frac{L}{2}-1}}  \right] 
\end{align*}
which completes our proof of this claim.
\end{proof}


Now that we have proved all the lemmas stated in
Section~\ref{ss:technical_tools}, the reader can next jump to Section~\ref{ss:the_proof}.
\section{An Example Where the Margin in Assumption~\ref{as:frac_margin} is Constant} \label{a:constant_gamma}
In this section we provide an example where the margin $\gamma$ in Assumption~\ref{as:frac_margin} is constant. Consider a two-layer Huberized ReLU network. In this section we always let $\phi$ denote the Huberized ReLU activation (see its definition in equation~\eqref{e:helu}). Since here we are only concerned with the properties of the network at initialization,
let $V^{(1)}$ be denoted simply by $V$. The first layer $V_1 \in \R^{p \times p}$ has its entries drawn independently from $\cN\left(0,\frac{2}{p}\right)$ and $V_2 \in \R^{1\times p}$ has its entries drawn independently from $\cN(0,1)$.  


Let $V_{1,i}$ denote the $i$th row of $V_1$ and let $V_{2,i}$ denote the $i$th coordinate of $V_2$. The network computed by these weights is $f_{V}(x) = V_{2}\phi(V_{1}x)$.

Consider data in which examples of each class are clustered.
There is a unit vector $\mu \in \S^{p-1}$ such that, for all $s$ with
$y_s = 1$, $\lv x_s - \mu \rv \leq r$, and, for all $s$
with $y_s = -1$, $\lv x_s - (-\mu) \rv \leq r$.  
Let
us say that such data is {\em $r$-clustered}.
(Recall
that $\lv x_s \rv = 1$ for all $s$.)
\begin{proposition}\label{p:constant_margin}
For any $\delta >0$, suppose that $h \le \frac{\sqrt{\pi}}{2p}$, $r \le \min\left\{\frac{1}{16},\frac{\sqrt{p}h}{c'\sqrt{\log\left(\frac{3pn}{\delta}\right)}}\right\}$, and $p\ge \log^{c'}(n/\delta)$ for a large enough constant $c'>0$. If the data $r$-clustered then,
with probability $1-\delta$ there exists $W^{\star}=(W_1^{\star},W_{2}^{\star})$ with $\lv W^{\star}\rv =1$ such that
\begin{align*}
   \text{for all }\;s\in [n],\qquad  y_s \left(\nabla_V f_V(x_s) \cdot W^{\star} \right) \ge c\sqrt{p}
\end{align*}
where $c$ is a positive absolute constant.
\end{proposition}
\begin{proof} 
Define a set
\begin{align*}
    \cS := \left\{ i \in [p]: \frac{1}{2}\le\left| V_{2,i} \right|\le 2\right\},
\end{align*}
and also define
\begin{align*}
    \cS_{+} := \left\{i \in \cS: V_{1,i}\cdot \mu \ge 4h\right\} \quad  \text{ and} \quad
    \cS_{-} := \left\{i \in \cS: -V_{1,i} \cdot\mu \ge 4h\right\}.
\end{align*}

Consider an event $\cE_{\mathsf{margin}}$ such that all of the following simultaneously occur:
\begin{enumerate}[(a)]
    \item $p\left(\frac{1}{4}-o_p(1)\right)\le |\cS_{+}| \le p\left(\frac{1}{2}+o_p(1)\right)$;
    \item $p\left(\frac{1}{4}-o_p(1)\right)\le |\cS_{-}| \le p\left(\frac{1}{2}+o_p(1)\right)$;
    \item for all $s \in [n]$ and $i \in [p]$, $\lvert V_{1,i}\cdot (x_s - y_s\mu)\rvert \le 2h$.
\end{enumerate}
Using simple concentration arguments in Lemma~\ref{l:margin_good_event} below we will show that $\Pr\left[\cE_{\mathsf{margin}}\right]\ge 1-\delta$.
Let us assume that the event $\cE_{\mathsf{margin}}$ holds for the remainder of the proof. 

The gradient of $f$ with respect to $V_{1,i}$ is
\begin{align*}
    \nabla_{V_{1,i}} f_V(x) =  x \left(V_{2,i}\phi'(V_{1,i}\cdot x)\right).
\end{align*}

Consider a sample with index $s$ with $y_s = 1$. For any $i \in \cS_{+}$ 
\begin{align*}
    \sign(V_{2,i})\left(\mu\cdot \nabla_{V_{1,i}} f_V(x_s)\right) &=  \mu\cdot x_s \left(|V_{2,i}|\phi'(V_{1,i}\cdot x_s)\right) \\
    &= 
       \left(|V_{2,i}|\phi'(V_{1,i} \cdot x_s)\right)+ \mu \cdot (x_s - \mu) \left(|V_{2,i}|\phi'(V_{1,i} \cdot x_s)\right)\\
    &\overset{(i)}{\ge} \frac{1}{2} \phi'(V_{1,i}\cdot \mu + V_{1,i}\cdot (x_s-\mu)) - 2\mu\cdot (x_s -\mu) \phi'(V_{1,i}\cdot x_s) \\
    &\overset{(ii)}{\ge} \frac{1}{2} \phi'(V_{1,i}\cdot \mu + V_{1,i}\cdot (x_s-\mu)) - \frac{1}{8}\\
    &\overset{(iii)}{\ge} \frac{ \phi'(2h)}{2} - \frac{1}{8}\overset{(iv)}{=} \frac{1}{2}   - \frac{1}{8}= \frac{3}{8} \numberthis \label{e:margin_lower_bound_good_node}
\end{align*}
where $(i)$ follows since $\frac{1}{2}\le |V_{2,i}|\le 2$ when $i \in \cS_+$. Inequality~$(ii)$ follows since $\phi'$ is bounded by $1$ and because $\lv x_s -y_s\mu\rv\le r \le 1/16$. Inequality~$(iii)$ follows since $i \in \cS_{+}$ and therefore $(V_{1,i})\cdot \mu \ge 4h$, under event $\cE_{\mathsf{margin}}$, $(V_{1,i})\cdot (x_s-\mu)\ge -2h$, and since $\phi'$ is a monotonically increasing function. Equation~$(iv)$ follows since $\phi'(2h)=1$. 
On the other hand,
for any $i \in \cS_{-}$:
\begin{align*}
     \sign(V_{2,i})\mu\cdot \nabla_{V_{1,i}} f_V(x_s) &= \mu\cdot x_s \left(|V_{2,i}|\phi'(V_{1,i}\cdot x_s)\right)\\
     &= 
     |V_{2,i}|\phi'(V_{1,i} \cdot x_s)+ \mu \cdot (x_s - \mu) \left(|V_{2,i}|\phi'(V_{1,i} \cdot x_s)\right)\\
     &\overset{(i)}{\ge}
     - 2\mu\cdot (x_s -\mu) \phi'(V_{1,i}\cdot x_s) \overset{(ii)}{\ge} - \frac{1}{8}
    \label{e:margin_lower_bound_bad_node} \numberthis
\end{align*}
where 
$(i)$ follows since $|V_{2,i}|\le 2$ when $i \in \cS_-$ and $\phi'$ is always non-negative.
Inequality~$(ii)$ again follows since $\phi'$ is bounded by $1$ and because $\lv x_s -y_s\mu\rv\le r \le 1/16$. 

Similarly we can also show that for a sample $s$ with $y_s = -1$, for any $i \in \cS_{-}$
\begin{align}
    \sign(V_{2,i})\mu\cdot \nabla_{V_{1,i}} f_V(x_s) \le \frac{-3 }{8} \label{e:margin_lower_bound_good_node_neg}
\end{align}
and for any $i \in \cS_{+}$
\begin{align}
    \sign(V_{2,i})\mu\cdot \nabla_{V_{1,i}} f_V(x_s) \le \frac{1}{8}. \label{e:margin_lower_bound_bad_node_neg}
\end{align}
With these calculations in place let us construct $W^{\star} = (W_{1}^{\star},W_2^{\star})$ where, $W_{1}^{\star}\in \R^{p \times p}$, $W_{2}^{\star} \in \R^{1\times p}$ and $\lv W^{\star}\rv=1$. Set $W_2^{\star} =0$. For all $i \in \cS_{+}\cup \cS_{-}$ set
\begin{align*}
    W_{1,i}^{\star} = \sign(V_{2,i})\mu\frac{1}{\sqrt{|\cS_{+}|+|\cS_{-}|}}
\end{align*}
and for all $i \not\in \cS_{+}\cup \cS_{-}$, set $W_{1,i}^{\star}=0$. We can easily check that $\lv W^{\star}\rv = 1$ (since $\lv \mu\rv =1$). Thus, for any sample $s$ with $y_s =1$  
\begin{align*}
   & y_s \left(\nabla_V f_V(x_s) \cdot W^{\star} \right)\\ &= \nabla_{V_1} f_V(x_s) \cdot W_1^{\star} \\
    & = \sum_{i\in \cS_{+}}\nabla_{V_{1,i}} f_V(x_s) \cdot W_{1,i}^{\star}+\sum_{i\in \cS_{-}}\nabla_{V_{1,i}} f_V(x_s) \cdot W_{1,i}^{\star}\\
    &=\frac{1}{ \sqrt{|\cS_{+}|+|\cS_{-}|}}\left[ \sum_{i\in \cS_{+}}\sign(V_{2,i})\mu\cdot \nabla_{V_{1,i}} f_V(x_s)  +\sum_{i\in \cS_{-}}\sign(V_{2,i})\mu\cdot\nabla_{V_{1,i}} f_V(x_s) \right] \\
    &\overset{(i)}{\ge} \frac{1}{ \sqrt{|\cS_{+}|+|\cS_{-}|}}\left[\frac{3|\cS_{+}|}{8}-\frac{|\cS_{-}|}{8}\right] \\
    &=\frac{1}{ 8\sqrt{|\cS_{+}|+|\cS_{-}|}}\left[3|\cS_{+}|-|\cS_{-}|\right]\\
    &\overset{(ii)}{\ge} \frac{1}{ 8\sqrt{p\left(1+o_p(1)\right)}}\left[3p\left(\frac{1}{4}-o_p(1)\right)-p\left(\frac{1}{2}+o_p(1)\right)\right] \ge c\sqrt{p}
\end{align*}
where $(i)$ follows by using inequalities \eqref{e:margin_lower_bound_good_node} and \eqref{e:margin_lower_bound_bad_node} and $(ii)$ follows by Parts~(a) and (b) of the event $\cE_{\mathsf{margin}}$. The final inequality follows since we assume that $p$ is greater than a constant. This shows that it is possible to achieve a margin of $c\sqrt{p}$ on the positive examples. By mirroring the logic above and using inequalities~\eqref{e:margin_lower_bound_good_node_neg} and \eqref{e:margin_lower_bound_bad_node_neg} we can show that a margin of $c\sqrt{p}$ can also be attained on the negative examples. This completes our proof.
\end{proof}
As promised we now show that the event $\cE_{\mathsf{margin}}$ defined above occurs with probability at least $1-\delta$.
\begin{lemma}
\label{l:margin_good_event}
For the event $\cE_{\mathsf{margin}}$ be 
defined in the proof of 
Proposition~\ref{p:constant_margin} 
above,
\begin{align*}
    \Pr\left[\cE_{\mathsf{margin}}\right] \ge 1-\delta.
\end{align*}
\end{lemma}
\begin{proof}We shall show that each of the three sub-events in the definition of the event $\cE_{\mathsf{margin}}$ occur with probability at least $1-\delta/3$. Then a union bound establishes the statement of the lemma.

\textit{Proof of Part~(a):} Recall the definition of the set $\cS$
\begin{align*}
    \cS := \left\{ i \in [p]: \frac{1}{2}\le\left| V_{2,i} \right|\le 2\right\},
\end{align*}
and also the definition of the set $\cS_{+}$
\begin{align*}
 \cS_{+} := \left\{i \in \cS: V_{1,i}\cdot \mu \ge 4h\right\}.  
\end{align*}
We will first derive a high probability bound the size of the set $\cS$, and then use this bound to control the size of $\cS_{+}$. A trivial upper bound is $|\cS| \le p$. Let us derive a lower bound on its size. Define the random variable $\zeta_i = \mathbb{I}\left[\frac{1}{2}\le |V_{2,i}| \le 2 \right]$. It is easy to check that $|\cS| = \sum_{i \in [p]}\zeta_i$. The expected value of this random variable
\begin{align*}
     \E\left[\zeta_i\right]  = 1- \Pr\left[|V_{2,i}|\le \frac{1}{2}\right]-\Pr\left[|V_{2,i}|\ge 2\right] 
& \overset{(i)}{\ge}1-  \frac{1/2-(-1/2)}{\sqrt{2\pi}}-\Pr\left[|V_{2,i}|\ge 2\right] \\
& \overset{(ii)}{\ge}  1-\frac{1}{\sqrt{2\pi}}-\frac{\exp\left(-2\right)}{\sqrt{2\pi}} > \frac{1}{2}
\end{align*}
where $(i)$ follows since $V_{2,i} \sim \cN(0,1)$ so its density is upper bounded bounded by $1/\sqrt{2\pi}$, and $(ii)$ follows by a Mill's ratio bound to upper bound $\Pr\left[|V_{2,i}|\ge z\right]\le 2\times \frac{\exp(-z^2/2)}{\sqrt{2\pi}z}$. 
 A Hoeffding bound (see Theorem~\ref{thm:hoeffding}) implies that for any $\eta \ge 0$
\begin{align*}
    \Pr\left[|\cS|\ge p\mathbb{E}\left[\zeta_i\right] -\frac{\eta p}{2}\right] \ge 1-\exp\left( -c_1\eta^2p\right).
\end{align*}
Setting $\eta = 1/p^{1/4}$ we get
\begin{align}
    \Pr\left[|\cS|\ge p\left(\frac{1}{2}-\frac{1}{p^{1/4}} \right)\right] \ge 1-\exp\left( -c_1\sqrt{p}\right). \label{e:good_event_size_of_S}
\end{align}
We now will bound $|\cS_{+}|$ conditioned on the event in the previous display: $p\left(\frac{1}{2}-\frac{1}{p^{1/4}}\right)\le|\cS|\le p$.

For each $i \in \cS$, the random variable $V_{1,i} \cdot \mu \sim \cN\left(0,\frac{2}{p}\right)$ since each entry of $V_{1,i}$ is drawn independently from $\cN\left(0,\frac{2}{p}\right)$ and because $\lv \mu \rv =1$. Define a random variable $\xi_i:=\mathbb{I}\left[V_{1,i}\cdot \mu \ge 4h\right]$. It is easy to check that $|\cS_{+}|=\sum_{i\in {\cS}} \xi_i$. The expected value of $\xi_i$
\begin{align*}
\left\lvert \E\left[\xi_i\right]-\frac{1}{2}\right\rvert  =  \left\lvert \Pr\left[V_{1,i}\cdot \mu \ge 4h\right]-\frac{1}{2}\right\rvert & = \left\lvert\Pr\left[V_{1,i}\cdot \mu \ge 0\right]-\Pr\left[V_{1,i}\cdot \mu \in [0,4h)\right]-\frac{1}{2}\right\rvert\\
    &=\Pr\left[V_{1,i}\cdot \mu \in [0,4h)\right]\overset{(i)}{\le} \frac{4h\sqrt{p}}{\sqrt{2\pi}\sqrt{2}} \overset{(ii)}{\le} \frac{1}{\sqrt{p}}
\end{align*}
where $(i)$ follows since the density of this Gaussian is upper bounded by $\frac{1}{\sqrt{2\pi}\times\left(\frac{\sqrt{2}}{\sqrt{p}}\right)}$ and $(ii)$ is by the assumption that $h \le\frac{\sqrt{\pi}}{2p}$. Thus we have shown that $\frac{1}{2}-\frac{1}{\sqrt{p}}\le \E\left[\xi_i\right]\le \frac{1}{2}+\frac{1}{\sqrt{p}}$. Again a Hoeffding bound (see Theorem~\ref{thm:hoeffding}) implies that for any $\eta \ge 0$
\begin{align*}
    \Pr\left[\Big|\sum_{i\in \cS}\xi_i-|\cS|\E\left[\xi_i\right]\Big| \le \eta p \;\; \bigg| \;\; p\left(\frac{1}{2}-\frac{1}{p^{1/4}}\right)\le  |\cS|\le p\right] \ge 1-2\exp\left( -c_2\eta^2p\right).
\end{align*}
By setting $\eta = 1/p^{1/4}$ we get that
\begin{align}
    \Pr\left[\Big|\left|\cS_{+}\right|-|\cS|\E\left[\xi_i\right]\Big| \le p^{3/4}\; \; \bigg| \;\; p\left(\frac{1}{2}-\frac{1}{p^{1/4}}\right)\le  |\cS|\le p\right]  \ge 1-2\exp\left( -c_2\sqrt{p}\right). \label{e:size_of_set_s_plus}
\end{align}
By a union bound over the events in \eqref{e:good_event_size_of_S} and \eqref{e:size_of_set_s_plus} we get that
\begin{align*}
    \Pr\left[p\left(\frac{1}{4}-o_p(1)\right)\le |\cS_{+}|\le p\left(\frac{1}{2}+o_p(1)\right)\right] \ge 1-\exp\left( -c_1\sqrt{p}\right)-2\exp\left( -c_2\sqrt{p}\right).
\end{align*}
By assumption $p \ge \log^{c'}(n/\delta)$ for a large enough constant $c'$, thus
\begin{align*}
        \Pr\left[p\left(\frac{1}{4}-o_{p}(1)\right)\le |\cS_{+}|\le p\left(\frac{1}{2}+o_{p}(1)\right)\right] \ge 1-\delta/3
\end{align*}
which completes our proof of the first part.

\textit{Proof of Part~(b):} The proof of this second part follows by exactly the same logic as Part~(a).

\textit{Proof of Part~(c):} Fix any $i \in [p]$ and $s \in [n]$. Recall that $V_{1,i} \sim \cN\left(0,\frac{2}{p}I\right)$ and by assumption $\lv x_s - y_s\mu\rv \le r$. Thus the random variable $V_{1,i} \cdot (x_s - y_s\mu)$ is a zero-mean Gaussian random variable with variance at most $\frac{2r^2}{p}$. A standard Gaussian concentration bound implies that
\begin{align}\label{e:conc_tight_control_norm_v_2_i}
    \Pr\left[\left| V_{1,i}\cdot \left(x_s - y_s\mu\right) \right| \le 2h\right] \ge 1-2\exp\left(-\frac{c_2ph^2}{r^2}\right).
\end{align}
By a union bound over all $i \in [p]$ and all $s \in [n]$ we get
\begin{align*}
    \Pr\left[\exists\; i\in [p], s\in [n] \;: \;\left|(V_{1,i})\cdot (x_s -y_s \mu)\right| \le 2h \right] \ge 1- 2np\exp\left(-\frac{c_2p h^2}{r^2}\right)\ge 1-\frac{\delta}{3}
\end{align*}
where the last inequality follows since $r^2 \le \frac{ph^2}{(c')^2\log\left(\frac{3pn}{\delta}\right)}$ and because $p\ge \log^{c'}(n/\delta)$ for a large enough constant $c'>0$. This completes our proof.
\end{proof}

\section{Omitted Proofs from Section~\ref{ss:ntk_init}} \label{a:proof_of_ntk}

In this section we prove Theorem~\ref{t:theorem_frac_margin}. We largely follow the high-level analysis strategy presented in \citep{chen2021much} to prove that, with high probability, if the the width of the network is large enough then gradient descent drives down the loss to at most $\frac{1}{n^{1+24L}}$ under Assumption~\ref{as:frac_margin}. After that we use our general result, Theorem~\ref{t:main_theorem}, to prove that gradient descent continues to reduce the loss beyond this point. We begin by introducing some definitions that are useful in our proofs in this section. All the results in this section are specialized to the case of the Huberized ReLU activation function (see its definition in equation~\eqref{e:helu}).
\subsection{Additional Definitions and Notation}
Following \citet{chen2021much}, we define the Neural Tangent random features (henceforth NT) function class. These definitions depend on the initial weights $V^{(1)}$ and radii $\tau,\rho>0$. We shall choose the value of these radii in terms of problem parameters in the sequel. Define a ball around the initial 
parameters.
\begin{definition}\label{def:ball_around_init}For any $V^{(1)}$ and $\rho>0$ define a ball around this weight matrix as
\begin{align*}
    \cB(V^{(1)},\rho):= \left\{V: \max_{\ell \in [L+1]} \lv V_{\ell}-V^{(1)}_{\ell} \rv\le \rho \right\}.
\end{align*}
\end{definition}
We then define the neural tangent kernel function class.
\begin{definition}
\label{def:neural_tangent_kernel} Given initial weights $V^{(1)}$, define the function $$F_{V^{(1)},V}(x):= f_{V^{(1)}}(x) + \left( \nabla f_{V^{(1)}}(x)\right) \cdot (V-V^{(1)}),$$ then the NT function class with radius $\rho>0$ is as follows
\begin{align*}
    \cF(V^{(1)},\rho):= \left\{F_{V^{(1)},V}(x): V \in \cB(V^{(1)},\rho) \right\}.
\end{align*}
\end{definition}

We continue to define the minimal error achievable by any function in this NT function class.
\begin{definition}
\label{minimal_NT_error}For any $V^{(1)}$ and any $\rho>0$ define
\begin{align*}
    \epsilon_{\mathsf{NT}}(V^{(1)},\rho) := 
   \min_{V \in  \cB(V^{(1)},\rho)} 
    \frac{1}{n}\sum_{s=1}^n \log(1+\exp(-y_s F_{V^{(1)},V}(x_s))),
\end{align*}
that is, it is the 
minimal training loss achievable by functions in the NT function class centered at $V^{(1)}$. Also 
let
$V^{\star}(V^{(0)},\rho) \in \cB(V^{(0)},\rho)$ be an arbitrary minimizer:
\begin{align*}
    V^{\star} \in \argmin_{V \in  \cB(V^{(1)},\rho)} \frac{1}{n}\sum_{s=1}^n \log(1+\exp(-y_s F_{V^{(1)},V}(x_s))).
\end{align*}

\end{definition}
We will be concerned with the maximum approximation error of this tangent kernel around a ball of the initial weight matrix. 
\begin{definition}
\label{app_error}For any $V^{(1)}$ and any $\tau>0$ define
\begin{align*}
    \epsilon_{\mathsf{app}}(V^{(1)},\tau) := \sup_{s\in [n]}\sup_{\hV,\tV \in \cB(V^{(1)},\tau)}\left\lvert f_{\hV}(x_s)-f_{\tV}(x_s)-\nabla f_{\tV}(x_s)\cdot\left(\hV-\tV \right)\right\rvert.
\end{align*}
\end{definition}

Finally we define the maximum norm of the gradient with respect to the weights of any layer.
\begin{definition}
\label{max_gradient}For any initial weights $V^{(1)}$ and any $\tau>0$ define 
\begin{align*}
    \Gamma(V^{(1)},\tau) := \sup_{s\in [n]}\sup_{\ell \in [L+1]}\sup_{V \in \cB(V^{(1)},\tau)}\lv \nabla_{V_{\ell}}f_{V}(x_s)\rv.
\end{align*}
\end{definition}

\subsection{Technical Tools Required for the Neural Tangent Kernel Proofs} \label{ss:neural_tangent_tech_tools}
We borrow \citep[][Lemma~5.1]{chen2021much} that bounds the average empirical risk in the first $T$ iterations when the iterates remain in a ball around the initial weight matrix. We have translated the lemma into our notation.
\begin{lemma}\label{l:average_loss_decrease}
Set the step-size $\alpha_t =\alpha= O\left(\frac{1}{L\Gamma(V^{(1)},\tau)^2}\right)$ for all $t\in [T]$. Suppose that given an initialization $V^{(1)}$ and radius $\rho>0$ we pick $\tau>0$ such that $V^{\star} \in \cB(V^{(1)},\tau)$ and $V^{(t)}\in \cB(V^{(1)},\tau)$ for all $t \in [T]$,
and that $\epsilon_{\mathsf{app}}(V^{(1)},\tau) < 3/8$. 
Then 
\begin{align*}
    \frac{1}{T} \sum_{t=1}^{T} J(V^{(t)}) \le \frac{\lv V^{(1)}-V^{\star} \rv^2-\lv V^{(T+1)}-V^{\star} \rv^2+2T\alpha \epsilon_{\mathsf{NT}}(V^{(1)},\rho)}{T\alpha \left(\frac{3}{2}-4\epsilon_{\mathsf{app}}(V^{(1)},\tau)\right)}.
\end{align*}
\end{lemma}
Technically the setting studied by \citet{chen2021much} differs from the setting that we study in our paper. They deal with neural networks with ReLU activations instead of Huberized ReLU activations that we consider here. However, it is easy to scan through the proof of their lemma to verify that it does not rely on any specific properties of ReLUs.

The next lemma bounds the approximation error of the neural tangent kernel in a neighbourhood around the initial weight matrix and provides a bound on the maximum norm of the gradient. The proof of this lemma below relies on several different lemmas that are collected and proved in Appendix~\ref{a:lemmas_ntk_approx_bound_aid}. 
\begin{restatable}{lem}{ntkapprox}
\label{l:ntk_app_error}
For any $\delta >0$, suppose that $\tau = \Omega\left(\frac{\log^2\left(\frac{nL}{\delta} \right)}{p^{\frac{3}{2}}L^3}\right)$ and, for
a sufficiently small positive constant $c$, we have
$\tau \leq \frac{c}{L^{12} \log^{\frac{3}{2}}(p)}$, $h\le \frac{\tau}{\sqrt{p}}$ and $p = \poly\left(L,\log\left(\frac{n}{\delta}\right)\right)$ for some sufficiently large polynomial. Then, with probability at least $1-\delta$ over the random initialization $V^{(1)}$, we have
\begin{enumerate}[(a)]
    \item $\epsilon_{\mathsf{app}}(V^{(1)},\tau) \le  O(\sqrt{p \log(p)}L^{5}\tau^{4/3})$, and
    \item $\Gamma(V^{(1)},\tau) \le O(\sqrt{p}L^2)$.
\end{enumerate}
\end{restatable}
Having provided a bound on the approximation error, let us continue and show that gradient descent reaches a weight matrix whose error is comparable to $\epsilon_{\mathsf{NT}}$.
\begin{lemma}
\label{l:general_NTK_descent_lemma}For any $L \in \N$, $\delta>0$, $$\tau = \Omega\left(\frac{\log^2\left(\frac{nL}{\delta} \right)}{p^{\frac{3}{2}}L^3}\right) \quad \text{and} \quad \tau \le \frac{c}{(p\log(p))^{\frac{3}{8}}L^{\frac{15}{4}}},$$  where $c$ is a small enough positive constant,
$\rho= \frac{\tau}{3L}$, $h \le \frac{\tau}{\sqrt{p}}$,
 and $p \ge \poly\left(L,\log\left(\frac{n}{\delta}\right)\right)$ for a large enough polynomial, if we run gradient descent with a constant step-size $\alpha_t =\alpha = \Theta\left(\frac{1}{pL^5}\right)$, 
for 
$T= \ceil{\frac{(L+1)\rho^2}{4\alpha \cdot \epsilon_{\mathsf{NT}}(V^{(1)},\rho)}}$ iterations,
 with probability $1-\delta$ over the random initialization
\begin{align*}
    \min_{t \in [T]} J(V^{(t)})\le 6\epsilon_{\mathsf{NT}}(V^{(1)},\rho).
\end{align*}
\end{lemma}
Our proof closely follows the proof of \citep[][Theorem~3.3]{chen2021much}.
\begin{proof} Recall the definition of
\begin{align*}
     V^{\star} \in \argmin_{V \in  \cB(V^{(1)},\rho)} \frac{1}{n}\sum_{s=1}^n \log(1+\exp(-y_s F_{V^{(1)},V}(x_s))).
\end{align*}
We would like to apply Lemma~\ref{l:average_loss_decrease} to show that the average loss of the iterates of gradient descent decreases. To do so we must first ensure that all iterates $V^{(t)}$ and $V^{\star}$ remain in a ball of radius $\tau$ around initialization. 

We have assumed that $\tau \le \frac{c}{(p\log(p))^{\frac{3}{8}}L^{\frac{15}{4}}}$ and that $p \ge \poly\left(L,\log\left(\frac{n}{\delta}\right)\right)$ for a large enough polynomial. Therefore if this polynomial is large enough we have that $\tau \le \frac{c_1}{L^{12}\log^{\frac{3}{2}}(p)}$,
for an arbitrarily small positive constant
$c_1$.  This means we can invoke Lemma~\ref{l:ntk_app_error} which guarantees that with probability at least $1-\delta$, the approximation error $\epsilon_{\mathsf{app}}(V^{(1)},\tau)\le O(\sqrt{p \log(p)}L^5 \tau^{4/3})$ and the maximum norm of the gradient $\Gamma(V^{(1)},\tau) \le O(\sqrt{p}L^2)$. Again recall that $\tau \le \frac{c}{(p\log(p))^{\frac{3}{8}}L^{\frac{15}{4}}}$, where $c$ is a small enough positive constant. Thus for a small enough value of $c$ the approximation error $\epsilon_{\mathsf{app}}(V^{(1)},\tau) \le \frac{1}{8}$. Let us assume that this is the case going forward.

Since $\rho = \frac{\tau}{3L} \le \tau$, $V^{\star}$ is clearly in  $\cB(V^{(1)},\tau)$. We will now show that the iterates $\{V^{(t)}\}_{t \in [T]}$ also lie in this ball by induction. The base case when $t = 1$ is trivially true. So now assume that $V^{(1)},\ldots,V^{(t-1)}$ lie in this ball and we will proceed to show that $V^{(t)}$ also lies in this ball. Since $\epsilon_{\mathsf{app}}(V^{(1)},\tau)\le 1/8$, by Lemma~\ref{l:average_loss_decrease} we infer that
\begin{align*}
 \frac{1}{t-1} \sum_{t'=1}^{t-1} J(V^{(t')}) \le \frac{\lv V^{(1)}-V^{\star} \rv^2-\lv V^{(t)}-V^{\star} \rv^2+2(t-1)\alpha \epsilon_{\mathsf{NT}}(V^{(1)},\rho)}{(t-1)\alpha },
\end{align*}
which in turn implies that
\begin{align*}
     \sum_{\ell \in [L+1]}\lv V^{(t)}_{\ell}-V^{\star}_{\ell}\rv^2  = \lv V^{(t)}-V^{\star}\rv^2 &\le \lv V^{(1)}-V^{\star}\rv^2+2\alpha (t-1) \epsilon_{\mathsf{NT}}(V^{(1)},\rho)-\alpha \sum_{t'=1}^{t-1} J(V^{(t')}) \\
     &\le \lv V^{(1)}-V^{\star}\rv^2+2\alpha (t-1) \epsilon_{\mathsf{NT}}(V^{(1)},\rho)\\
     &\overset{(i)}\le (L+1)\rho^2 + \frac{(L+1)\rho^2}{2} \le \frac{3(L+1)\rho^2}{2}\le 3L\rho^2
\end{align*}
where $(i)$ follows since 
$V^* \in \cB(V^{(1)}, \rho)$ and
$t\le T = \ceil{\frac{(L+1)\rho^2}{4\alpha \epsilon_{\mathsf{NT}}(V^{(1)},\rho)}}$. Taking square roots implies that for each $\ell\in [L+1]$, $\lv V_{\ell}^{(t)} - V_{\ell}^{\star}\rv \le \sqrt{3L}\rho$. By the triangle inequality for any $\ell \in [L+1]$
\begin{align*}
    \lv V^{(t)}_{\ell}-V^{(1)}_{\ell}\rv &\le \lv V^{(t)}_{\ell}-V^{\star}_{\ell}\rv+\lv V^{\star}_{\ell}-V^{(1)}_{\ell}\rv  \le \sqrt{3L}\rho+\rho < 3L\rho = \tau.
\end{align*}
This shows that $V_{\ell}^{(t)} \in \cB(V^{(1)},\tau)$ and completes the induction. 

Now that we have established that $V^{\star}$ and $V^{(t)}$ are all in a ball of radius $\tau$ around $V^{(1)}$ we can again invoke Lemma~\ref{l:average_loss_decrease} (recall from above that $\epsilon_{\mathsf{app}}(V^{(1)},\tau)\le \frac{1}{8}$ and $\Gamma(V^{(1)},\tau)\le O(\sqrt{p}L^2)$) to infer that
\begin{align*}
   \min_{t \in [T]} J(V^{(t)}) \le \frac{1}{T}\sum_{t=1}^T J(V^{(t)}) &\le \frac{\lv V^{(1)}-V^{\star}\rv^2-\lv V^{(T+1)}-V^{\star}\rv^2+2T\alpha \epsilon_{\mathsf{NT}}(V^{(1)},\rho)}{T\alpha}\\
    &\le \frac{\lv V^{(1)}-V^{\star}\rv^2+2T\alpha \epsilon_{\mathsf{NT}}(V^{(1)},\rho)}{T\alpha}\\
    &= 2\epsilon_{\mathsf{NT}}(V^{(1)},\rho)+ \frac{\lv V^{(1)}-V^{\star}\rv^2}{T\alpha}\le 6\epsilon_{\mathsf{NT}}(V^{(1)},\rho),
\end{align*}
where the last inequality follows since $V^{\star} \in \cB(V^{(1)},\rho)$, therefore $\lv V^{(1)}-V^{\star}\rv^2 \le (L+1)\rho^2$ and because $T= \ceil{\frac{(L+1)\rho^2}{4\alpha \epsilon_{\mathsf{NT}}(V^{(1)},\rho)}}$.
This completes our proof.
\end{proof}

Finally we shall show that under Assumption~\ref{as:frac_margin} the error $\epsilon_{\mathsf{NT}}(V^{(1)},\rho)$ is bounded with high probability. Recall the assumption on the data.
\fracmargin*
\begin{lemma}
\label{l:error_NTK_eror_frac_margin}
Under the Assumption~\ref{as:frac_margin}, for any $\epsilon,\delta>0$, if the radius $$\rho \ge \frac{c\left[ \sqrt{\log(n/\delta)}+\log\left(\frac{1}{\exp(\epsilon)-1}\right)\right]}{\sqrt{p}\gamma}$$ for some large enough positive absolute constant $c$ then, with probability $1-2\delta$ over the randomness in the initialization 
\begin{align*}
    \epsilon_{\mathsf{NT}}(V^{(1)},\rho)= \min_{V \in  \cB(V^{(1)},\rho)} \frac{1}{n}\sum_{s=1}^n \log(1+\exp(-y_s F_{V^{(1)},V}(x_s))) \le \epsilon.
\end{align*}
\end{lemma}
\begin{proof} 
Recall that, by definition,
\[
F_{V^{(1)},V}(x)
 = f_{V^{(1)}}(x) + \left( \nabla f_{V^{(1)}}(x)\right)\cdot (V-V^{(1)}).
\]
By Assumption~\ref{as:frac_margin} we know that,
with probability $1 - \delta$, 
there exists 
$W^{\star}$ with $\lv W^{\star}\rv=1$, such that for all $s \in [n]$
\begin{align} \label{e:good_event_assumption_ntk_1}
    y_i \left(\nabla f_{V^{(1)}}(x_s)\cdot W^{\star} \right)\ge \sqrt{p}\gamma.
\end{align}
By Lemma~\ref{l:bound_on_magnitude_of_function} proved below with know that
\begin{align} \label{e:good_event_assumption_ntk_2}
        \Pr\left[ |f_{V^{(1)}}(x_s)| \le c_1 \sqrt{\log(n/\delta)}  \right] \ge 1-\delta.
\end{align}
For the remainder of the proof let's assume that both events in \eqref{e:good_event_assumption_ntk_1} and \eqref{e:good_event_assumption_ntk_2} occur. This happens with probability at least $1-2\delta$.
Thus, for any positive $\lambda$
\begin{align*}
    y_i \left[f_{V^{(1)}}(x_s)+\lambda \nabla_V f_{V^{(1)}}(x_s)\cdot W^{\star}\right] \ge \lambda \sqrt{p}\gamma - c_1 \sqrt{\log(n/\delta)}.
\end{align*}
Setting $\lambda = \frac{c_1 \sqrt{\log(n/\delta)}+\log\left(\frac{1}{\exp(\epsilon)-1}\right)}{\sqrt{p}\gamma}$ we infer that
\begin{align}
    y_i \left[f_{V^{(1)}}(x_s)+\lambda \left(\nabla_V f_{V^{(1)}}(x_s)\cdot W^{\star}\right)\right] \ge \lambda \sqrt{p}\gamma - c_1 \sqrt{\log(n/\delta)}= \log\left(\frac{1}{\exp(\epsilon)-1}\right). \label{e:ntk_margin_lower_bound}
\end{align}
Set $V = V^{(1)}+\lambda W^{\star}$. The neural tangent kernel function at this weight vector is
\begin{align*}
    F_{V^{(1)},V}(x) = f_V^{(1)}(x) + \nabla_V f_{V^{(1)}}(x)\cdot (V-V^{(1)}) =  f_V^{(1)}(x)+\lambda \nabla_V f_{V^{(1)}}(x)\cdot W^{\star}.
\end{align*}
Thus by using \eqref{e:ntk_margin_lower_bound}
\begin{align*}
    \frac{1}{n}\sum_{s=1}^n \log\left(1+\exp\left(-y_i F_{V^{(1)},V}(x_s)\right)\right) & \le  \frac{1}{n}\sum_{s=1}^n \log\left(1+\exp\left(-\log\left(\frac{1}{\exp(\epsilon)-1}\right)\right)\right) \\
    &\le \epsilon.
\end{align*}
We can conclude that if we choose the radius $\rho \ge \lambda \lv W^{\star}\rv = \lambda = \frac{c_1 \sqrt{\log(n/\delta)}+\log\left(\frac{1}{\exp(\epsilon)-1}\right)}{\sqrt{p}\gamma}$ (since $\lv W^{\star}\rv =1$ by assumption) then there exists a function in the NT function class with training error at most $\epsilon$. This completes our proof.
\end{proof}

\subsection{Proof of Theorem~\ref{t:theorem_frac_margin}}
\ntktheorem*
\begin{proof}\textit{Proof of Part~1:} Define two events 
\begin{align*}
    \cE_{a} := \left\{  \epsilon_{\mathsf{NT}}(V^{(1)},\rho) \le \frac{1}{6n^{2+24L}}\right\} \quad \text{and} \quad
    \cE_{b}:= \left\{\min_{t \in [T]} J(V^{(t)}) \le 6\epsilon_{\mathsf{NT}}(V^{(1},\rho)\right\}.
\end{align*}
We will show that the $\cE_1 := \cE_a \cap \cE_b$ occurs with probability at least $1-3\delta$. That is,
\begin{align}
    \Pr\left[\cE_1 = \left\{\min_{t \in [T]} J(V^{(t)}) \le \frac{1}{n^{2+24L}} \right\}\right] \ge 1-3\delta. \label{e:definition_event_1_ntk_theorem}
\end{align}
The value of $\rho$ is set to be (this was done in equation~\eqref{e:definition_rho_radius})
\begin{align*}
\rho &= \frac{c_1}{\sqrt{p}\gamma} \mleft[\sqrt{\log\left(\frac{n}{\delta}\right)}+\log\mleft(6n^{(2+24L)}\mright)\mright]\\
& > \frac{c_1}{\sqrt{p}\gamma} \mleft[\sqrt{\log\left(\frac{n}{\delta}\right)}+\log\mleft(\frac{1}{\exp\left(\frac{1}{6n^{(2+24L)}}\right)-1}\mright)\mright] \qquad \mbox{(since $e^z \le 1+2z$ when $z \in [0,1]$).}
\end{align*}
With this choice of $\rho$, since $c_1$ is a large enough absolute constant, Lemma~\ref{l:error_NTK_eror_frac_margin} guarantees that 
\begin{align}
\Pr\left[\cE_{a}\right] \ge 1-2\delta ,   \label{e:ntk_error_theorem_proof_ref}
\end{align}
where the probability is over the randomness in the initialization. Continue by setting
\begin{align*}
    \tau = 3L\rho =  \frac{3c_1L}{\sqrt{p}\gamma} \mleft[\sqrt{\log\left(\frac{n}{\delta}\right)}+\log\mleft(6n^{(2+24L)}\mright)\mright].
\end{align*}
Since $p \ge \poly\left(L,\log\left(\frac{n}{\delta}\right)\right)/\gamma^2$ for a large enough polynomial it is guaranteed that 
\begin{align*}
    \tau = \Omega\left(\frac{\log^2\left(\frac{nL}{\delta} \right)}{p^{\frac{3}{2}}L^3}\right) \quad \text{and} \quad \tau \le \frac{c_3}{(p\log(p))^{\frac{3}{8}}L^{\frac{15}{4}}}
\end{align*}
where $c_3$ is the positive absolute constant from the statement of Lemma~\ref{l:general_NTK_descent_lemma}.
Also recall the value of $h=h_{\mathsf{NT}}$ from equation~\eqref{e:huberized_relu_h_def}
\begin{align*}
    h = h_{\mathsf{NT}}=\frac{(1+24L)\log(n)}{6(6p)^{\frac{L+1}{2}}L^{3}} \overset{(i)}{\le} \frac{3c_1L \left[\sqrt{\log(n/\delta)}+\log\left(6n^{(2+24L)}\right)\right]}{p\gamma} = \frac{\tau}{\sqrt{p}},
\end{align*}
where $(i)$ follows since $\gamma \in (0,1]$ by assumption and because $p$ is large enough. Under these choices of $\tau$ and $h$ along with the choice of the step-size $\alpha_t = \Theta\left(\frac{1}{pL^5}\right)$, and number of steps $T$, Lemma~\ref{l:general_NTK_descent_lemma} guarantees that
\begin{align} \label{e:ntk_error_small_optimization_ref}
   \Pr\left[ \cE_{b} \right]\ge 1-\delta.
\end{align}
A union bound over the events \eqref{e:ntk_error_theorem_proof_ref} and \eqref{e:ntk_error_small_optimization_ref} proves the Claim~\eqref{e:definition_event_1_ntk_theorem}, which completes the proof of this first part.

\textit{Proof of Part~2:} To prove this part of the lemma, we will invoke Theorem~\ref{t:main_theorem} to guarantee that the loss decreases in the steps $t \in \{T+1,\ldots \}$. We defined $V^{(T+1)} = V^{(s)}$, where $s \in \argmin_{t \in [T]}J(V^{(t)})$, thus we are guaranteed to have $J(V^{(T+1)})\le \frac{1}{n^{2+24L}}< \frac{1}{n^{1+24L}}$, if event $\cE_1$ defined above occurs. Define another event
\begin{align*}
    \cE_2 :=\left\{\lv V^{(1)}\rv \le \sqrt{5pL}  \right\}.
\end{align*}
Lemma~\ref{l:bound_on_norm_of_matrix} guarantees that $\Pr\left[\cE_2\right]\ge 1-\delta$. Define the ``good event'' $\cE:=\cE_1 \cap \cE_2$. A simple union bound shows that
\begin{align*}
    \Pr[\cE] \ge 1-4\delta.
\end{align*}
Assume that this event $\cE$ occurs for the remainder of this proof. This also establishes that the success probability of gradient descent is at least $1-4\delta$ as mentioned in the theorem statement.

To invoke Theorem~\ref{t:main_theorem} we need to ensure that $h<h_{\max}$. 
Recall that, 
in equation~\eqref{e:h_max_def}, we defined
\begin{align*}
     h_{\max} &:= \min\left\{\frac{L^{\frac{L}{2}-3}\log(1/J_{T+1})}{24\sqrt{p}\lv V^{(T+1) }\rv^{L}},1\right\}.
\end{align*}
 For all $\ell \in [L+1]$, $\lv V^{(T+1)}_{\ell} - V^{(1)}\rv \le \tau$ (this fact is implicit in the proof of Lemma~\ref{l:general_NTK_descent_lemma}). By the triangle inequality
\begin{align*}
    \lv V^{(T+1)}\rv \le \lv V^{(1)}\rv+\lv V^{(T+1)}-V^{(1)}\rv \le \sqrt{5pL}+\sqrt{L+1}\tau \le \sqrt{6pL}
\end{align*}
by the choice of $\tau$ above and since $p \ge \frac{\poly\left(L,\log\left(\frac{n}{\delta}\right)\right)}{\gamma^2}$ for a large enough polynomial. This means that
\begin{align*}
   h =  h_{\mathsf{NT}} &=\frac{(1+24L) \log(n)}{6^{\frac{L+3}{2}}p^{\frac{L+1}{2}}L^3}\le \frac{L^{\frac{L}{2}-3}(1+24L)\log(n)}{24\sqrt{p}(\sqrt{6pL})^{L}}\le h_{\max} .
\end{align*}
Thus, our choice of $h$ is valid. In this second stage the step-size is chosen to be
\begin{align*}
    \alpha_{\max}(h)  &= \min\left\{\frac{h}{1024\left( L+1\right)^2pJ_{T+1}\lv V^{(T+1)}\rv^{3L+5}}, \frac{(L+\frac{1}{2})\lv V^{(T+1)} \rv^2}{2L(L+\frac{3}{4})^2J_{T+1} \log^{2/L}(1/J_{T+1})}\right\} \\
    &=\frac{h}{1024\left( L+1\right)^2pJ_{T+1}\lv V^{(T+1)}\rv^{3L+5}},
\end{align*}
where the first term of the minima wins out above by our choice of $h$ and because $\lv V^{(T+1)}\rv \le \sqrt{6pL}$.
Thus Theorem~\ref{t:main_theorem} guarantees that
\begin{align*}
    J(V^{(t)}) \le \frac{J(V^{(T+1)})}{\widetilde{Q}(\alpha_{\max}(h))\cdot(t-T-1)+1},
\end{align*}
where $ \widetilde{Q}(\cdot)$ was defined in equation~\eqref{eq:Q_def}. Thus,
\begin{align*}
     \widetilde{Q}(\alpha_{\max}(h)) &= \frac{L(L+\frac{3}{4})^2\alpha_{\max}(h) J_{T+1}\log^{2/L}(1/J_{T+1})}{(L+\frac{1}{2})\lv V^{(T+1)}\rv^2} \\
    & = \frac{L(L+\frac{3}{4})^2 J_{T+1}\log^{2/L}(1/J_{T+1})}{(L+\frac{1}{2})\lv V^{(T+1)}\rv^2} \times  \frac{h}{1024\left( L+1\right)^2pJ_{T+1}\lv V^{(T+1)}\rv^{3L+5}}\\
    &= \frac{h L(L+\frac{3}{4})^2}{1024 (L+1)^2(L+\frac{1}{2})p \lv V^{(T+1)}\rv^{3L+7}}\\
    &\ge \frac{(L+\frac{3}{4})^2\log(n) }{50 L(L+1)^2(L+\frac{1}{2}) (6p)^{\frac{L+3}{2}}(6pL)^{\frac{3L+7}{2}}}\\
    &\ge \frac{\log(n)}{50 L^{\frac{3L+7}{2}}(L+1)^2 (6p)^{2L+5}} \\
    &\ge\frac{1}{50 L^{\frac{3L+7}{2}}(L+1)^2 (6p)^{2L+5}} .
\end{align*}
Thus, for all $t \ge T+1$
\begin{align*}
     J(V^{(t)})& \le \frac{J(V^{(T+1)})}{Q\cdot(t-T-1)+1}\\
      & \le \frac{1}{n^{1+24L}}\frac{1}{Q\cdot(t-T-1)+1} \\
      & \le \frac{1}{n^{1+24L}}\frac{50 L^{\frac{3L+7}{2}}(L+1)^2 (6p)^{2L+5}}{(t-T-1)+50 L^{\frac{3L+7}{2}}(L+1)^2 (6p)^{2L+5}}\\
      &< \frac{50 L^{\frac{3L+7}{2}}(L+1)^2 (6p)^{2L+5}}{n^{1+24L}(t-T-1)} \\
      & = O\left(\frac{L^{\frac{3L+11}{2}}(6p)^{2L+5}}{n^{1+24L}\cdot(t-T-1)}\right),
\end{align*}
this completes the proof.
\end{proof}

\section{Proof of Lemma~\ref{l:ntk_app_error}} \label{a:lemmas_ntk_approx_bound_aid}
In this section we prove Lemma~\ref{l:ntk_app_error} that controls the approximation error $\epsilon_{\mathsf{app}}(V^{(1)},\tau)$ and establishes a bound on the maximum norm of the gradient $\Gamma(V^{(1)},\tau)$ . The proof of this lemma requires analogs of several lemmas from \citep{allen2019convergence,zou2020gradient} adapted to our setting.  In Appendix~\ref{sss:lemmas_ntk_at_init} we prove that several useful properties hold at initialization with high probability. In Appendix~\ref{ss:ntk_lemmas_around_init} we show that some of these properties extend to weight matrices close to initialization and in Appendix~\ref{ss:proof_of_lemma_ntk_approx} we prove Lemma~\ref{l:ntk_app_error}. 

Throughout this section we analyze the initialization scheme described in Section~\ref{ss:ntk_init}. This scheme is as follows: for all $\ell \in [L]$ the entries of $V_{\ell}^{(1)}$ are drawn independently from $\cN\left(0,2/p\right)$ and the entries of $V_{L+1}^{(1)}$ are drawn independently from $\cN\left(0,1\right)$. Again, the results of this appendix apply only to the Huberized ReLU (see definition in \eqref{e:helu}). 
\subsection{Properties at Initialization} \label{sss:lemmas_ntk_at_init}
In the next lemma we show that several useful properties hold with high probability at initialization. 
\begin{lemma}\label{l:conc_main_lemma}
For any $\delta>0$, suppose that $h< \frac{1}{50\sqrt{p}L}$, $p \ge \poly\left(L,\log\left(\frac{n}{\delta}\right)\right)$ for a large enough polynomial and
$\tau = \Omega\left(\frac{\log^{2}(\frac{nL}{\delta})}{p^{\frac{3}{2}}L^{3}}\right)$. Then with probability at least $1-\delta$ over the randomness in $V^{(1)}$ we have the following:
\begin{enumerate}[(a)]
    \item For all $s \in [n]$ and all $\ell \in [L]$: \begin{align*}
    \lv x_{\ell,s}^{V^{(1)}} \rv \in \left[\frac{9}{10},\frac{11}{10}\right].
\end{align*}
\item For all all $\ell \in  [L]$,
         $\lv V^{(1)}_\ell \rv_{op} \le O(1)$,
     and $\lv V^{(1)}_{L+1}\rv \le O(\sqrt{p})$.
\item For all $s \in [n]$ and all $1\le \ell_1 \le \ell_2 \le L$,
     \begin{align*}
         \left\lv V^{(1)}_{\ell_2} \Sigma_{\ell_2-1,s}^{V^{(1)}} \cdots \Sigma_{\ell_1,s}^{V^{(1)}} V_{\ell_1}^{(1)}\right\rv_{op} \le O(L),
     \end{align*}
     and for any $1\le \ell_1 \le L$
     \begin{align*}
         \left\lv V^{(1)}_{L+1} \Sigma_{L,s}^{V^{(1)}} \cdots \Sigma_{\ell_1,s}^{V^{(1)}} V_{\ell_1}^{(1)}\right\rv_{op} \le O(\sqrt{p}L).
     \end{align*}
     \item For all $s \in [n]$ and all $1\le \ell_1 \le \ell_2 \le L$,
     \begin{align*}
         \left\lv V^{(1)}_{\ell_2} \Sigma_{\ell_2-1,s}^{V^{(1)}} \cdots \Sigma_{\ell_1,s}^{V^{(1)}} V_{\ell_1}^{(1)}a\right\rv \le 3 \lv a \rv 
     \end{align*}
     for all vectors $a$ with $\lv a \rv_{0}\le k = \frac{cp}{\log(p)L^2}$, where $c$ is a small enough positive absolute constant.
     \item For all $s \in [n]$ and all $1\le \ell_1 \le \ell_2 \le L$,
     \begin{align*}
         \left\lv a^{\top} V^{(1)}_{\ell_2} \Sigma_{\ell_2-1,s}^{V^{(1)}} \cdots \Sigma_{\ell_1,s}^{V^{(1)}} V_{\ell_1}^{(1)}\right\rv \le O(\lv a \rv)
     \end{align*}
     for all vectors $a$ with $\lv a \rv_{0}\le k = \frac{cp}{\log(p)L^2}$, where $c$ is a small enough positive absolute constant.
   \item For all $s \in [n]$ and all $1\le \ell_1 \le \ell_2 \le L$,
\begin{align*}
        \lvert a^{\top} V_{\ell_2}^{(1)}\Sigma_{\ell_2-1,s}^{V^{(1)}}\cdots \Sigma_{\ell_1,s}^{V^{(1)}}V_{\ell_1}^{(1)}b\rvert \le  O\left(\lv a \rv \lv b \rv\frac{\sqrt{k\log(p)}}{\sqrt{p}}\right)
     \end{align*}
     for all vectors $a,b$ with $\lv a \rv_{0},\lv b \rv_{0}\le k = \frac{cp}{\log(p)L^2}$, where $c$ is a small enough positive absolute constant.
     \item For all $s \in [n]$ and all $1\le \ell \le L$,
\begin{align*}
        \lvert V_{L+1}^{(1)}\Sigma_{L,s}^{V^{(1)}}\cdots \Sigma_{\ell_1,s}^{V^{(1)}}V_{\ell}^{(1)}a\rvert \le  O\left(\lv a \rv \sqrt{k\log(p)}\right)
     \end{align*}
     for all vectors $a$ with $\lv a \rv_{0}\le k = \frac{cp}{\log(p)L^2}$, where $c$ is a small enough positive absolute constant.
     \item 
     For $\beta=O\left(\frac{L^2 \tau^{2/3}}{\sqrt{p}}\right)$ and 
\begin{align*}
    \cS_{\ell,s}(\beta):= \left\{j\in [p]: |V^{(1)}_{\ell,j} x_{\ell,s}^{V^{(1)}}| \le \beta \right\},
\end{align*}
where $V^{(1)}_{\ell,j}$ refers to the $j$th row of $V^{(1)}_{\ell}$, for all $\ell \in [L]$ and all $s \in [n]$: $$|\cS_{\ell,s}(\beta)|\le O(p^{3/2}\beta)=O(p L^2 \tau^{2/3}).$$
\end{enumerate}

\end{lemma}
 We will prove this lemma part by part and show that each of the eight properties holds with probability at least $1-\delta/8$ and take a union bound at the end. We show that each of the parts hold with this probability in the eight lemmas (Lemmas~\ref{l:conc_event_part_a}-\ref{l:conc_event_part_h}) that follow.

\subsubsection{Proof of Part~(a)}
\begin{lemma}\label{l:conc_event_part_a}
For any $\delta>0$, suppose that $h< \frac{1}{50\sqrt{p}L}$ and $p \ge \poly\left(L,\log\left(\frac{n}{\delta}\right)\right)$ for a large enough polynomial, then with probability at least $1-\delta/8$ over the randomness in $V^{(1)}$ we have that for all $s \in [n]$ and all $\ell \in [L]$: \begin{align*}
    \lv x_{\ell,s}^{V^{(1)}} \rv \in \left[\frac{9}{10},\frac{11}{10}\right].
\end{align*}
\end{lemma}
\begin{proof}Fix any layer $\ell \in [L]$ and any sample $s \in [n]$. We will prove the result for this layer and sample, and apply a union bound at the end. To ease notation we drop $V^{(1)}$ from the superscript of $x_{\ell,s}^{V^{(1)}}$ and refer to $V_{\ell}^{(1)}$ as simply $V_{\ell}$.

By definition 
\begin{align*}
    x_{\ell,s} = \phi\left(V_{\ell}x_{\ell-1,s} \right).
\end{align*}
Conditioned on $x_{\ell-1,s}$, each coordinate
of 
$V_{\ell} x_{\ell-1,s}$ is distributed as $\cN\left(0,\frac{2\lv x_{\ell-1,s}\rv^2}{p}\right)$, since each entry of $V_{\ell}$ is drawn independently from $ \cN(0,\frac{2}{p})$. 
Let $\bar{\phi}(z) = \max\{0,z\}$ denote the ReLU activation function. Then we know that $\bar{\phi}(z)-\frac{h}{2} \le \phi(z) \le \bar{\phi}(z)$ for any $z \in \R$. Let  $(x_{\ell,s})_i$ denote the $i$th coordinate of $x_{\ell}$ and let $V_{\ell,i}$ denote the $i$th row of $V_{\ell}$. Therefore, conditioned on $x_{\ell-1,s}$, 
\begin{align*}
    \mathbb{E}\left[(x_{\ell,s})_i^2 \big| x_{\ell-1,s}\right] &=  \mathbb{E}\left[\phi^2\left(V_{\ell,i} x_{\ell-1,s} \right)\big| x_{\ell-1,s}\right] \\ 
    & \ge \mathbb{E}\left[\bar{\phi}^2\left(V_{\ell,i} x_{\ell-1,s} \right)\big| x_{\ell-1,s}\right]-h\mathbb{E}\left[\bar{\phi}\left(V_{\ell,i} x_{\ell-1,s} \right)\big| x_{\ell-1,s}\right] +\frac{h^2}{4}\\
    &\overset{(i)}{=} \frac{1}{2}\mathbb{E}\left[\left(V_{\ell,i} x_{\ell-1,s}\right)^2 \big| x_{\ell-1,s}\right]-\frac{h\mathbb{E}\left[\left|V_{\ell,i} x_{\ell-1,s} \right|\big| x_{\ell-1,s}\right]}{2} +\frac{h^2}{4} \\
    &= \frac{\lv x_{\ell-1,s}\rv^2}{p} -\frac{h\lv x_{\ell-1,s}\rv}{\sqrt{2p \pi}} +\frac{h^2}{4},
\end{align*}
where $(i)$ follows since $\bar{\phi}(z) = 0$ if $z<0$
and the 
distribution of $V_{\ell,i} x_{\ell-1,s}$ is symmetric about
the origin. Therefore summing up over all $i \in [p]$ we find
 \begin{align*} \mathbb{E}\left[\lv x_{\ell,s}\rv^2\;| x_{\ell-1,s}\right] = \sum_{i \in [p]} \mathbb{E}\left[(x_{\ell,s})_i^2 \;| x_{\ell-1,s}\right] 
 &\ge \lv x_{\ell-1,s}\rv^2-\frac{\sqrt{p}h\lv x_{\ell-1,s}\rv}{\sqrt{2 \pi}}+\frac{h^2p}{4}\\ 
 &\ge \left(\lv x_{\ell-1,s}\rv-\frac{h\sqrt{p}}{2} \right)^2. \numberthis \label{e:expected_value_lower_bound}\end{align*} Similarly we can also demonstrate an upper bound of 
$\mathbb{E}\left[ \lv x_{\ell,s} \rv^2\;|\; x_{\ell-1,s}\right] \le \lv x_{\ell-1,s}\rv^2$ 
since $\phi(z)\le \bar{\phi}(z)$ for any $z$ as stated previously.

Let $\lv \cdot \rv_{\psi_2}$ denote the sub-Gaussian norm of a random variable (see~Definition~\ref{def:subgaussian}) and let $\lv \cdot \rv_{\psi_1}$ denote the sub-exponential norm (see~Definition~\ref{def:subexp}). Since the function $\phi$ is $1$-Lipschitz, 
conditioned on $x_{\ell-1,s}$, 
\begin{align*}
    \lv (x_{\ell,s})_{i} \rv_{\psi_2}&=\lv \phi(V_{\ell,i}x_{\ell-1,s}) \rv_{\psi_2}\\ &\le \lv \phi(V_{\ell,i}x_{\ell-1,s})-\E\left[\phi(V_{\ell,i}x_{\ell-1,s})| x_{\ell-1,s}\right] \rv_{\psi_2} +\lv \E\left[\phi(V_{\ell,i}x_{\ell-1,s})| x_{\ell-1,s}\right] \rv_{\psi_2} \\
    & \overset{(i)}{\le} c\frac{\lv x_{\ell-1,s}\rv}{\sqrt{p}} + \lv \E\left[\phi(V_{\ell,i}x_{\ell-1,s})| x_{\ell-1,s}\right] \rv_{\psi_2} \overset{(ii)}\le c_1\frac{\lv x_{\ell-1,s}\rv}{\sqrt{p}} \label{e:sub-Gaussian_norm} \numberthis
\end{align*}
where $(i)$ follows by 
invoking Lemma~\ref{l:sub_gaussian_lipschitz}, and $(ii)$ follows since we showed above that $$\lv \E\left[\phi(V_{\ell,i}x_{\ell-1,s})| x_{\ell-1,s}\right] \rv_{\psi_2}=|\mathbb{E}\left[\phi(V_{\ell,i}x_{\ell-1,s})| x_{\ell-1,s}\right]| \le \sqrt{\mathbb{E}\left[\phi^2(V_{\ell,i}x_{\ell-1,s}) | x_{\ell-1,s} \right]}\le  \frac{\lv x_{\ell-1,s}\rv}{\sqrt{p}}. $$ Therefore $\lv (x_{\ell,s})_i^2 \rv_{\psi_1}\le \lv (x_{\ell,s})_i \rv_{\psi_2}^2 \le \frac{c_2 \lv x_{\ell-1,s}\rv^2}{p}$ by Lemma~\ref{l:sub_gaussian_squared}.
Since the random variables $(x_{\ell,s})_1^2,\ldots,(x_{\ell,s})_p^2$ are conditionally independent given
$x_{\ell-1,s}$,
applying Bernstein's inequality (see Theorem~\ref{thm:bernstein}) we get that
for any
$\eta \in (0,1]$
\begin{align*}
& \Pr\left(\Big| \lv x_{\ell,s}\rv^2 - \E\left[\lv x_{\ell,s}\rv^2\;| x_{\ell-1,s}\right]\;\Big|\;\le \eta \lv x_{\ell-1,s}\rv^2 \;\Big| x_{\ell-1,s}\right) \\
& \qquad \ge 1-2\exp\left(-c \min\left\{
         \frac{\eta^2 \lv x_{\ell-1,s}\rv^4 }{p \times \left(c_2^2 \lv  x_{\ell-1,s}\rv^4/p^2\right)},
         \frac{\eta \lv x_{\ell-1,s}\rv^2 }{c_2 \lv x_{\ell-1,s}\rv^2/p}
                            \right\}
      \right) \\
      &\qquad \ge 1-2\exp\left(-c_3 \min\left\{
        \eta^2 p,
         \eta p
                            \right\}
      \right) \\
& \qquad\ge 1-2\exp\left(-c_3 p \eta^2\right).
\end{align*}
We established above that the expected value
satisfies the following bounds:
$$\left(\lv x_{\ell-1,s}\rv -\frac{h\sqrt{p}}{2} \right)^2\le \mathbb{E}\left[\lv x_{\ell,s}\rv^2\;| x_{\ell-1,s}\right]\le \lv x_{\ell-1,s}\rv^2.$$
Thus
\begin{align*}
    &\Pr\left(\lv x_{\ell,s}\rv^2 \in \left[\left(\lv x_{\ell-1,s}\rv -\frac{h\sqrt{p}}{2} \right)^2-\eta \lv x_{\ell-1,s}\rv^2,\lv x_{\ell-1,s}\rv^2(1+\eta)\right]\;\Bigg|\; x_{\ell-1,s}\right)\\& \qquad \ge 1-2\exp\left(-c_3 p \eta^2 \right).
\end{align*}
Taking a union bound over all samples and all hidden layers we find that
\begin{align*}
&\Pr\left(\forall s\in [n],\ell \in [L],\; \lv x_{\ell,s}\rv^2 \in \left[\left(\lv x_{\ell-1,s}\rv -\frac{h\sqrt{p}}{2} \right)^2-\eta \lv x_{\ell-1,s}\rv^2,\lv x_{\ell-1,s}\rv^2(1+\eta)\right]\right)\\& \qquad \ge 1-2nL\exp\left(-c_3 p\eta^2 \right).
\end{align*}
This implies that
\begin{align*}
   & \Pr\left(\forall s \in [n], \ell \in [L],\; \Big\lvert \lv x_{\ell,s}\rv^2 - \lv x_{\ell-1,s}\rv^2\Big\rvert \le \eta \lv x_{\ell-1,s}\rv^2 + h\sqrt{p}\lv x_{\ell-1,s}\rv+\frac{h^2p}{4} \right) \\ &\ge 1- 2nL\exp\left(-c_3 p \eta^2 \right).
\end{align*}
Setting $\eta = \frac{1}{50L}$ and because by assumption $h\sqrt{p}<\frac{1}{50L}=\eta$ we get that 
\begin{align} 
\nonumber
  &  \Pr\left(\forall s \in [n], \ell \in [L],\; \Big\lvert \lv x_{\ell,s}\rv^2 - \lv x_{\ell-1,s}\rv^2\Big\rvert \le \eta \lv x_{\ell-1,s}\rv^2 + \eta \lv x_{\ell-1,s}\rv+\frac{\eta^2}{4} \right) \\ &\ge 1- 2nL\exp\left(-\frac{c_4 p}{L^2} \right). \label{e:event_in_hidden_layer_convergence}
\end{align}
Let us assume that the event of
\eqref{e:event_in_hidden_layer_convergence}
holds for the rest of this proof.
Starting with $\ell=1$ we know that $\lv x_{0,s} \rv = \lv x_s \rv = 1 $, thus if the event in the previous display holds then by the choice of $\eta = 1/(50L)$ we have that
\begin{align*}
    \lv x_{1,s} \rv^2 \in  [1-3\eta,1+3\eta].
\end{align*}
For any $z\in [0,1]$ we have that $(1+z)^{1/2}\le 1+z$ and $(1-z)^{1/2}\ge 1-z$. Thus, by taking square roots
\begin{align*}
     \lv x_{1,s} \rv  \in  \left[1-3\eta ,1+3\eta\right].
\end{align*}
We will now prove that $\lv x_{\ell,s} \rv  \in [1-3\ell \eta,1+3\ell \eta]$ using an inductive argument over $\ell=1,\ldots,L$. The base case when $\ell=1$ of course holds by the display above. Now let us prove it for a layer $\ell > 1$ assuming it holds at layer $\ell-1$. 

Let us first prove the upper bound on $\lv x_{\ell,s}\rv$, the lower bound will follow by the same logic. If the event in \eqref{e:event_in_hidden_layer_convergence} holds then we know that
\begin{align*}
    \lv x_{\ell,s} \rv^2 - \lv x_{\ell-1,s} \rv^2  \le \eta \lv x_{\ell-1,s}\rv^2 + \eta\lv x_{\ell-1,s}\rv+\frac{\eta^2}{4} 
   \end{align*}
   which implies that
   \begin{align*}
   \lv x_{\ell,s} \rv^2 &\le \lv x_{\ell-1,s} \rv^2\left(1+\eta\right)+\eta\lv x_{\ell-1,s}\rv+\frac{\eta^2}{4}\\
   &= \lv x_{\ell-1,s} \rv^2\left(1+\eta +\frac{\eta}{ \lv x_{\ell-1,s} \rv}+\frac{\eta^2}{4 \lv x_{\ell-1,s} \rv^2}\right)\\
   &\overset{(i)}{\le} \lv x_{\ell-1,s} \rv^2\left(1+\eta +\frac{10\eta}{9}+\frac{25\eta^2 }{ 81}\right)\\&=\lv x_{\ell-1,s} \rv^2\left(1 +\frac{19\eta}{9}+\frac{25\eta^2 }{ 81}\right)
   \overset{(ii)}{\le} \lv x_{\ell-1,s} \rv^2\left(1+\frac{20\eta}{9}\right)
\end{align*}
where $(i)$ follows since by the inductive hypothesis $\lv x_{\ell-1,s}\rv \ge 1-3(\ell-1)\eta$ and because $\eta=\frac{1}{50L}$, therefore $\lv x_{\ell-1,s}\rv \ge 1-\frac{3(\ell-1)}{100L} \ge \frac{97}{100} > \frac{9}{10}$, and $(ii)$ again follows because  $\eta=\frac{1}{50L}$ and $L\ge 1$. 
Taking square roots we find that
\begin{align*}
     \lv x_{\ell,s} \rv &\le \lv x_{\ell-1,s} \rv \sqrt{1+\frac{20\eta}{9}}\\
     &\le \left(1+3(\ell-1)\eta\right)\sqrt{1+\frac{20\eta}{9}} \qquad \mbox{(by the IH)}\\
     &\overset{(i)}{\le} \left(1+3(\ell-1)\eta\right)\left(1+\frac{20\eta}{9}\right)\\
     &= 1+3\left(\ell-1\right)\eta+\frac{20\eta}{9}+\frac{60(\ell-1)\eta^2}{9} \overset{(ii)}{\le} 1+3\ell\eta,
\end{align*}
where $(i)$ follows since $\sqrt{1+z}\le 1+z$ and $(ii)$ follows since $\eta =\frac{1}{50L}$ and $L\ge 1$. This establishes the desired upper bound on $\lv x_{\ell,s}\rv$. As mentioned above, the lower bound $(1-3\ell\eta)\le \lv x_{\ell,s}\rv$ follows by mirroring the logic. This completes our induction and proves that for all $s$ and all $\ell$ with probability at least  $1- \delta/8$
\begin{align*}
    \lv x_{\ell,s}\rv \in \left[1-3\ell \eta,1+3\ell\eta\right].
\end{align*}
Our choice of $\eta = \frac{1}{50L}$ establishes that 
\begin{align}
    \lv x_{\ell,s}\rv \in \left[\frac{9}{10},\frac{11}{10}\right] \label{e:bound_on_post_activation_feature_high_probability_statement}
\end{align}
for all $s\in [n]$ and $\ell \in [L]$ with probability at least $1- 2nL\exp\left(-\frac{c_4 p}{L^2} \right) \ge 1-\delta/8$, which
follows since $p \ge \poly\left(L,\log\left(\frac{n}{\delta}\right)\right)$ for a large enough polynomial. This wraps up our proof.
\end{proof}
\subsubsection{Proof of Part~(b)}
\begin{lemma}\label{l:conc_event_part_b}
For any $\delta>0$ suppose that $p \ge \poly\left(L,\log\left(\frac{n}{\delta}\right)\right)$ for a large enough polynomial, then with probability at least
$1-\delta/8$ 
over the randomness in $V^{(1)}$:
     \begin{align*}
         \text{for all } \ell\in [L], \quad \lv V^{(1)}_{\ell}\rv_{op} \le O(1), \quad \text{ and }\quad \lv V^{(1)}_{L+1}\rv \le O(\sqrt{p}).
     \end{align*}
\end{lemma}
\begin{proof}
For any fixed $\ell \in [L]$ recall that each entry of $V_{\ell}^{(1)}$ is drawn independently from $\cN\left(0,\frac{2}{p}\right)$. Thus, by invoking \citep[Theorem~4.4.5]{vershynin2018high} we know that
\begin{align*}
    \lv V^{(1)}_{\ell}\rv_{op} \le O(1)
\end{align*}
with probability at least $1-\exp(-\Omega(p))$. The entries of $V_{L+1}^{(1)}$ are drawn from $\cN(0,1)$, therefore by Theorem~\ref{thm:gaussian_concentration} we find that
\begin{align*}
    \lv V_{L+1}^{(1)}\rv^2 \le 2p
\end{align*}
with probability $1-\exp(-\Omega(p))$. By a union bound over the $L+1$ layers and noting that $p \ge \poly\left(L,\log\left(\frac{n}{\delta}\right)\right)$ yields that 
\begin{align*}
   \text{for all } \ell\in [L], \quad \lv V^{(1)}_{\ell}\rv_{op} \le O(1) \quad \text{ and }\quad \lv V^{(1)}_{L+1}\rv \le O(\sqrt{p})
\end{align*}
with probability at least $1-\delta/8$ as claimed.
\end{proof}

\subsubsection{Proof of Part~(c)}
\begin{lemma}\label{l:conc_event_part_c}
For any $\delta>0$ suppose that $p \ge \poly\left(L,\log\left(\frac{n}{\delta}\right)\right)$ for a large enough polynomial, then with probability at least $1-\delta/8$ over the randomness in $V^{(1)}$ we have that for all $s \in [n]$ and all $1\le \ell_1 \le \ell_2 \le L$,
     \begin{align*}
         \left\lv V^{(1)}_{\ell_2} \Sigma_{\ell_2-1,s}^{V^{(1)}} \cdots \Sigma_{\ell_1,s}^{V^{(1)}} V_{\ell_1}^{(1)}\right\rv_{op} \le O(L),
     \end{align*}
          and all $1\le \ell_1  \le L$
     \begin{align*}
         \left\lv V^{(1)}_{L+1} \Sigma_{L,s}^{V^{(1)}} \cdots \Sigma_{\ell_1,s}^{V^{(1)}} V_{\ell_1}^{(1)}\right\rv_{op} \le O(\sqrt{p}L).
     \end{align*}
\end{lemma}
\begin{proof} We begin by analyzing the case where $\ell_2 < L+1$. A similar analysis works to prove the claim when $\ell_2 = L+1$. This is because the variance of each entry of $V_{L+1}^{(1)}$ is $1$, whereas when $\ell_2 < L+1$ the variance of each entry of $V_{\ell}^{(1)}$ is $2/p$. Therefore the bound is simply multiplied by a factor of $\sqrt{2p}$ in the case when $\ell_2 = L+1$.

Fix the layers $1\le \ell_1\le \ell_2 \le L$ and fix the sample index $s$. At the end of the proof we shall take a union bound over all pairs of layers and all samples. Now to ease notation let us denote $V_{\ell}^{(1)}$ by simply $V_{\ell}$ and let $\Sigma_{\ell,s}^{V^{(1)}}$ be denoted by $\Sigma_{\ell,s}$. 

To bound the operator norm 
\begin{align}
\left\lv V^{(1)}_{\ell_2} \Sigma_{\ell_2-1,s}^{V^{(1)}} \cdots \Sigma_{\ell_1,s}^{V^{(1)}} V_{\ell_1}^{(1)}\right\rv_{op} = \sup_{a: \lv a \rv = 1}    \left\lv V^{(1)}_{\ell_2} \Sigma_{\ell_2-1,s}^{V^{(1)}} \cdots \Sigma_{\ell_1,s}^{V^{(1)}} V_{\ell_1}^{(1)}a\right\rv \label{e:part_c_operator_norm}
\end{align}
we will first consider a supremum over
vectors that are non-zero only on an arbitrary fixed subset $S \subseteq [p]$ with cardinality $|S|\le \floor{\frac{c_1p}{L^2}}$, where $c_1$ is small enough absolute constant. That is, we shall bound
\begin{align*}
\Xi:= \sup_{a: \lv a \rv = 1, \supp(a) \subseteq S}    \left\lv V^{(1)}_{\ell_2} \Sigma_{\ell_2-1,s}^{V^{(1)}} \cdots \Sigma_{\ell_1,s}^{V^{(1)}} V_{\ell_1}^{(1)}a\right\rv.
\end{align*}
Using this we will then bound the operator norm in \eqref{e:part_c_operator_norm} by decomposing any unit vector $a$ into $\frac{p}{\floor{\frac{c_1p}{L^2}}}$ vectors that are non-zero only on subsets of size at most $\floor{\frac{c_1p}{L^2}}$.

Let us begin by first bounding $\Xi$. 
 Part~(b) of Lemma~\ref{l:analog_of_conc_part_a_with_sigmas}, which is proved below, establishes that,
for any fixed unit vector $z \in \S^{p-1}$
\begin{align*}
    \lv V_{\ell_2} \Sigma_{\ell_2-1,s} \cdots \Sigma_{\ell_1,s}V_{\ell_1}z \rv \le 2
\end{align*}
with probability at least $1-O(nL^3)e^{-\Omega\left(\frac{p}{L^2}\right)}$.

We take a $1/4$-net (see the definition of an $\epsilon$-net in Definition~\ref{def:epsilon_net}) of unit vectors $\{a_i\}_{i=1}^m$
whose coordinates are non-zero only on this particular subset $S$, with respect to the Euclidean norm.
There exists such a $1/4$-net of size  $m=9^{c_1p/L^2}$ (see Lemma~\ref{l:covering_numbers_unit_vectors}). By a union bound, 
\begin{align} \label{e:cover_good_event}
  \forall i \in [m],\;
    \lv V_{\ell_2} \Sigma_{\ell_2-1,s} \cdots \Sigma_{\ell_1,s}V_{\ell_1}a_i \rv \le 2 
\end{align}
with probability at least $1-O(nL^3 \cdot 9^{c_1p/L^2})e^{-\Omega\left(\frac{p}{L^2}\right)} = 1-e^{-\Omega\left(\frac{p}{L^2}\right)}$ since $p \ge \poly\left(L,\log\left(\frac{n}{\delta}\right)\right)$ for a large enough polynomial, and because $c_1$ is a small enough constant. We will now proceed to show that if the ``good event'' 
\eqref{e:cover_good_event} 
regarding the $1/4$-net holds
then we can use it to establish guarantees for all unit vectors $a$ that are only non-zero on this subset $S$. 
To see this, 
if $\zeta(a)$ maps
each unit vector $a$ with support
contained in $S$ to its nearest neighbor in $\{ a_1,\ldots,a_m\}$, then
if the event in \eqref{e:cover_good_event} holds then
\begin{align*}
    &\Xi = \sup_{a:\lv a\rv=1, \supp(a)\subseteq S} \lv V_{\ell_2} \Sigma_{\ell_2-1,s} \cdots \Sigma_{\ell_1,s}V_{\ell_1}a \rv \\ &    
    =
    \sup_{a:\lv a\rv=1, \supp(a)\subseteq S} \lv V_{\ell_2} \Sigma_{\ell_2-1,s} \cdots \Sigma_{\ell_1,s}V_{\ell_1}(a-\zeta(a)+\zeta(a)) \rv\\
    &    \le \sup_{j \in [m]}\lv V_{\ell_2} \Sigma_{\ell_2-1,s} \cdots \Sigma_{\ell_1,s}V_{\ell_1}a_j \rv + \sup_{a:\lv a\rv=1, \supp(a)\subseteq S} \left\lv V_{\ell_2} \Sigma_{\ell_2-1,s} \cdots \Sigma_{\ell_1,s}V_{\ell_1}(a-\zeta(a)) \right\rv\\
    &  \overset{(i)}{\le} \sup_{j \in [m]}\lv V_{\ell_2} \Sigma_{\ell_2-1,s} \cdots \Sigma_{\ell_1,s}V_{\ell_1}a_j \rv +\frac{1}{4} \sup_{a:\lv a\rv=1, \supp(a)\subseteq S} \left\lv V_{\ell_2} \Sigma_{\ell_2-1,s} \cdots \Sigma_{\ell_1,s}V_{\ell_1}\frac{(a-\zeta(a))}{\lv a-\zeta(a)\rv} \right\rv\\
    & \overset{(ii)}{\le} 2+\frac{\Xi}{4}
\end{align*}
where 
and $(i)$ follows since $\lv a- \zeta(a)\rv \le 1/4$, inequality~$(ii)$ follows since we assumed the event \eqref{e:cover_good_event} to hold and by the definition of $\Xi$. By rearranging terms we find that,
with the same probability that is at least $1-e^{-\Omega\left(\frac{p}{L^2}\right)}$,
for any unit vector $a$ that is only non-zero on subset $S$, we have that
\begin{align}
\left\lv V_{\ell_2} \Sigma_{\ell_2-1,s} \cdots \Sigma_{\ell_1,s}V_{\ell_1}a \right\rv \le \frac{1}{1-\frac{1}{4}} \times 2 < 3. \label{e:allen-zhu_lemma_7.3_part_a_midway}
\end{align}
%

As mentioned above we will now consider a partition of $[p] = S_1 \cup \ldots \cup S_q$, such that for all $i \in [q]$, $|S_i|\le \floor{\frac{c_1p}{L^2}}$, and the number of sets ($q$) in the partition satisfies $q \le \frac{p}{\floor{\frac{c_1p}{L^2}}}= \ceil{\frac{L^2}{c_1}}$. Given an arbitrary unit vector $b\in \S^{p-1}$, we can decompose it as $b = u_1 + \ldots + u_q$, where each $u_i$ is non-zero only on the set $S_i$. Invoking the triangle inequality
\begin{align*}
    \lv V_{\ell_2} \Sigma_{\ell_2-1,s} \cdots \Sigma_{\ell}V_{\ell_1}b \rv &\le \sum_{i=1}^q \lv V_{\ell_2} \Sigma_{\ell_2-1,s} \cdots \Sigma_{\ell_1,s}V_{\ell_1}u_i \rv\\
    &=\sum_{i=1}^q \left\lv V_{\ell_2} \Sigma_{\ell_2-1,s} \cdots \Sigma_{\ell_1,s}V_{\ell_1}\frac{u_i}{\lv u_i\rv} \right\rv \lv u_i\rv.
\end{align*}
By applying the result of \eqref{e:allen-zhu_lemma_7.3_part_a_midway} to each term in the sum above along with a union bound over the $q$ sets $S_1,\ldots,S_q$ we find the following: 
with probability at least $1-qe^{-\Omega\left(\frac{p}{L^2}\right)}= 1-O(L^2)e^{-\Omega\left(\frac{p}{L^2}\right)}$, for all unit vectors $b \in \S^{p-1}$
\begin{align*}
    \lv V_{\ell_2} \Sigma_{\ell_2-1,s} \cdots \Sigma_{\ell_1,s}V_{\ell_1}b \rv &\le 3\sum_{i=1}^q \lv u_i\rv \le 3\sqrt{q}\left(\sum_{i=1}^q \lv u_i\rv^2  \right)^{1/2}= 3\sqrt{q}  =O(L).
\end{align*}
%
The definition of the operator norm of a matrix $\lv A\rv_{op} = \sup_{v: \lv v \rv =1} \lv Av\rv $ along with the previous display establishes the claim for this particular pair of layers $\ell_1$ and $\ell_2$ and sample $s$. A union bound over pairs of layers and all samples to establish that,
with probability at least $1-O(nL^4)e^{-\Omega\left(\frac{p}{L^2}\right)}$,
for all pairs $1\le \ell_1\le \ell_2\le L$ and all $s\in [n]$
\begin{align}
    \lv V_{\ell_2} \Sigma_{\ell_2-1} \cdots \Sigma_{\ell}V_{\ell_1} \rv_{op} \le O(L). \label{e:part_c_bound_when_ell_2_less_than_L+1}
\end{align}
As claimed above, a similar analysis shows that,
with probability at least $1-O(nL^4)e^{-\Omega\left(\frac{p}{L^2}\right)}$,
for all $s \in [n]$ and all $\ell_1 \in [L]$, we have
\begin{align}
    \lv V_{L+1} \Sigma_{L} \cdots \Sigma_{\ell}V_{\ell_1} \rv_{op} \le O(\sqrt{p}L). \label{e:part_c_bound_when_ell_2_equal_to_L+1}
\end{align}
Since $p \ge \poly\left(L,\log\left(\frac{n}{\delta}\right)\right)$ we can ensure that both events in \eqref{e:part_c_bound_when_ell_2_less_than_L+1} and \eqref{e:part_c_bound_when_ell_2_equal_to_L+1} occur simultaneously with probability at least $1-\delta/8$.
\end{proof}
\subsubsection{Proof of Part~(d)}
\begin{lemma}\label{l:conc_event_part_d}
For any $\delta>0$, suppose that $p \ge \poly\left(L,\log\left(\frac{n}{\delta}\right)\right)$ for a large enough polynomial, then with probability at least $1-\delta/8$ over the randomness in $V^{(1)}$ we have that for all $s \in [n]$ and all $1\le \ell_1 \le \ell_2 \le L$,
     \begin{align*}
         \left\lv V^{(1)}_{\ell_2} \Sigma_{\ell_2-1,s}^{V^{(1)}} \cdots \Sigma_{\ell_1,s}^{V^{(1)}} V_{\ell_1}^{(1)}a\right\rv \le 3 \lv a \rv 
     \end{align*}
     for all vectors $a$ with $\lv a \rv_{0}\le k= \frac{cp}{\log(p)L^2}$, where $c$ is a small enough positive absolute constant.
\end{lemma}
\begin{proof}We fix the layers $1\le \ell_1\le \ell_2 \le L$ and fix the sample index $s$. At the end of the proof we shall take a union bound over all pairs of layers and all samples. Again, to ease notation, let us denote $V_{\ell}^{(1)}$ by simply $V_{\ell}$ and let $\Sigma_{\ell,s}^{V^{(1)}}$ be denoted by $\Sigma_{\ell,s}$.

For a fixed unit vector $z \in \S^{p-1}$ by Part~(b) of Lemma~\ref{l:analog_of_conc_part_a_with_sigmas} that is proved below we have
\begin{align} \label{e:bound_on_norm_sparse_single_vector}
    \lv V_{\ell_2} \Sigma_{\ell_2-1,s} \cdots \Sigma_{\ell_1,s}V_{\ell_1}z \rv \le 2
\end{align}
with probability at least $1-O(nL^3)e^{-\Omega\left(\frac{p}{L^2}\right)}$. Consider a $1/4$-net of $k$-sparse unit vectors $\{a_i\}_{i=1}^m$, where $m=\binom{p}{k}9^{k}$ (such a net exists, see Lemma~\ref{l:covering_numbers_sparse_vectors}).

Using \eqref{e:bound_on_norm_sparse_single_vector} and taking a union bound, 
%
with probability at least $1-O\left(\binom{p}{k}9^k \cdot nL^3\right)e^{-\Omega\left(\frac{p}{L^2}\right)}$,
for all vectors $\{a_i\}_{i=1}^m$ 
\begin{align*}
    \lv V_{\ell_2} \Sigma_{\ell_2-1,s} \cdots \Sigma_{\ell_1,s}V_{\ell_1}a_i \rv \le 2.
\end{align*}
 
Now by mirroring the logic that lead from inequality~\eqref{e:cover_good_event} to inequality~\eqref{e:allen-zhu_lemma_7.3_part_a_midway} in the proof of the previous lemma, we can establish that,
again with probability that is at least $1-\binom{p}{k}9^k \cdot O(nL^3)e^{-\Omega\left(\frac{p}{L^2}\right)}$,
for any vector $a$ that is $k$-sparse 
\begin{align*}
    \left\lv V_{\ell_2} \Sigma_{\ell_2-1,s} \cdots \Sigma_{\ell_1,s}V_{\ell_1}\frac{a}{\lv a\rv} \right\rv \le 3.
\end{align*}
A union bound over all pairs of layers and all samples we find that,
with probability at least 
$1- O\left(\binom{p}{k}9^k \cdot n^2L^5\right)e^{-\Omega\left(\frac{p}{L^2}\right)}$,
for all $1\le \ell_1\le \ell_2\le L$, for all $s\in [n]$ and for all vectors $a$ that are $k$-sparse
\begin{align*}
    \left\lv V_{\ell_2} \Sigma_{\ell_2-1,s} \cdots \Sigma_{\ell_1,s}V_{\ell_1}\frac{a}{\lv a\rv} \right\rv \le 3.
\end{align*}
Moreover,
\begin{align*}
 1- O\left(\binom{p}{k}9^k \cdot n^2L^5\right)e^{-\Omega\left(\frac{p}{L^2}\right)}   & \ge 1- O\left(\left(\frac{ep}{k}\right)^k9^k \cdot n^2L^5\right)e^{-\Omega\left(\frac{p}{L^2}\right)}   \qquad \mbox{(since $\binom{p}{k}\le \left(\frac{ep}{k} \right)^k$)}\\
 &= 1- O\left(\left(\frac{9ep}{k}\right)^k \cdot n^2L^5\right)e^{-\Omega\left(\frac{p}{L^2}\right)}\\
 &= 1- O\left(n^2L^5\right)e^{-\Omega\left(\frac{p}{L^2}-k\log(9ep)\right)}\\
 &\ge 1-\delta/8
\end{align*}
where the last inequality follows since $k \le \frac{cp}{\log(p)L^2}$ where $c$ is a small enough absolute constant and $p \ge \poly\left(L,\log\left(\frac{n}{\delta}\right)\right)$ for a large enough polynomial. This completes the proof.
\end{proof}

\subsubsection{Proof of Part~(e)}
\begin{lemma}\label{l:conc_event_part_e}
For any $\delta>0$, suppose that $p \ge \poly\left(L,\log\left(\frac{n}{\delta}\right)\right)$ for a large enough polynomial, then with probability at least $1-\delta/8$ over the randomness in $V^{(1)}$ we have that for all $s \in [n]$ and all $1\le \ell_1 \le \ell_2 \le L$,
     \begin{align*}
         \left\lv a^{\top} V^{(1)}_{\ell_2} \Sigma_{\ell_2-1,s}^{V^{(1)}} \cdots \Sigma_{\ell_1,s}^{V^{(1)}} V_{\ell_1}^{(1)}\right\rv \le O(\lv a \rv)
     \end{align*}
     for all vectors $a$ with $\lv a \rv_{0} \le k  = \frac{cp}{\log(p)L^2}$, where $c$ is a small enough positive absolute constant.
\end{lemma}
\begin{proof}We fix the layers $1\le \ell_1\le \ell_2 \le L$ and fix the sample index $s$. At the end of the proof we shall take a union bound over all pairs of layers and all samples. 
In the proof let us denote $V_{\ell}^{(1)}$ by simply $V_{\ell}$ and let $\Sigma_{\ell,s}^{V^{(1)}}$ be denoted by $\Sigma_{\ell,s}$.

For any fixed vector $z$ we know from Part~(a) of Lemma~\ref{l:analog_of_conc_part_a_with_sigmas} that with probability at least $1-O(nL^3)e^{-\Omega\left(\frac{p}{L^2}\right)}$ over the randomness in $(V_{\ell_2-1},\ldots,V_{1})$
\begin{align} \label{e:part_d_good_event_1}
    \lv \Sigma_{\ell_2-1,s}V_{\ell_2-1}\ldots \Sigma_{\ell_1,s}V_{\ell_1} z\rv\le 2\lv z \rv.
\end{align}
Recall that the entries of $V_{\ell_2}$ are drawn independently from $\cN(0,\frac{2}{p})$. Thus, conditioned on this event above, for any fixed vector $w$ the random variable $w^{\top} V_{\ell_2} \left(\Sigma_{\ell_2-1,s}\cdots \Sigma_{\ell_1}V_{\ell_1} z\right)$ is a mean-zero Gaussian with variance at most $\frac{8\lv w \rv^2 \lv z \rv^2}{p}$. Thus over the randomness in $V_{\ell_2}$
\begin{align}\label{e:part_d_good_event_2}
    \Pr\left(\left\lvert w^{\top} V_{\ell_2} \Sigma_{\ell_2-1,s}\cdots \Sigma_{\ell_1,s}V_{\ell_1} z \right\rvert \le \frac{4}{L}\lv w\rv\lv z \rv\Big| V_{\ell_2-1},\ldots,V_{1}\right)\ge 1- e^{-\Omega\left(\frac{p}{L^2}\right)}. 
\end{align}
By union bound over the events in \eqref{e:part_d_good_event_1} and \eqref{e:part_d_good_event_2} we have
\begin{align}\label{e:part_d_good_event_3} 
    \Pr\left(\left\lvert w^{\top} V_{\ell_2} \Sigma_{\ell_2-1,s}\cdots \Sigma_{\ell_1,s}V_{\ell_1} z \right\rvert \le \frac{4}{L}\lv w\rv\lv z \rv\right)\ge 1-O(nL^3)e^{-\Omega\left(\frac{p}{L^2}\right)}. 
\end{align}
Similar to the proof of Lemma~\ref{l:conc_event_part_c} our strategy will be to first bound  
\begin{align*}
     \sup_{a: \lv a\rv =1,\lv a \rv_0 \le k}\sup_{b: \lv b\rv =1,\supp(b)\in \subseteq S}\left\lvert a^{\top} V_{\ell_2} \Sigma_{\ell_2-1,s}\cdots \Sigma_{\ell_1,s}V_{\ell_1} b \right\rvert
\end{align*}
where $S$ is a fixed subset of $[p]$ with $|S|\le \frac{c_1p}{L^2}$, where $c_1$ is a small enough absolute constant. Let $\{z_i\}_{i=1}^r$ be a $1/4$-net of unit vectors with respect to the Euclidean norm whose coordinates are non-zero only on this subset $S$. There exists such a $1/4$-net of size  $r=9^{c_1p/L^2}$ (see Lemma~\ref{l:covering_numbers_unit_vectors}). Let $\{w_i\}_{i=1}^m$ be a $1/4$-net of $k$-sparse unit vectors in Euclidean norm of size $m=\binom{p}{k}9^{k}$ (Lemma~\ref{l:covering_numbers_sparse_vectors} guarantees the existence of such a net). Therefore by using \eqref{e:part_d_good_event_3} and taking a union bound we get that 
\begin{align}\label{e:part_d_good_event_cover}
  \forall \; i \in [r],j \in [m], \;\;
    \left\lvert w_j^{\top} V_{\ell_2} \Sigma_{\ell_2-1,s}\cdots \Sigma_{\ell_1,s}V_{\ell_1} z_i \right\rvert \le \frac{4}{L}
\end{align}
with probability at least $1-mrO(nL^3)e^{-\Omega\left(\frac{p}{L^2}\right)}=1-O(9^{c_1p/L^2}\binom{p}{k}9^{k}nL^3)e^{-\Omega\left(\frac{p}{L^2}\right)}=1-e^{-\Omega\left(\frac{p}{L^2}\right)}$, since $k=\frac{cp}{\log(p)L^2}$ where both $c$ and $c_1$ are small enough absolute constants and because $p \ge \poly\left(L,\log\left(\frac{n}{\delta}\right)\right)$ for a large enough polynomial.

We will now demonstrate that if the ``good event'' in \eqref{e:part_d_good_event_cover} holds then we can use this to establish a similar guarantee for all $k$-sparse unit vectors $a$ and all unit vectors $b$ that are only non-zero on the subset $S$. To see this, 
as before, suppose $\zeta$ maps any unit-length
vector with support in $S$
to its nearest neighbor in $\{ z_1,\ldots,z_r\}$
and
$\lambda$ maps any $k$-sparse unit vector
to 
its nearest neighbor in $\{ w_1,\ldots,w_m\}$.
Then
if the event in \eqref{e:part_d_good_event_cover} holds, we have
\begin{align*}
    &\Xi:= \sup_{a:\lv a \rv=1,\lv a\rv_0 \le k} \sup_{b:\lv b \rv=1,\supp(b) \subseteq S} \left\lvert a^{\top} V_{\ell_2} \Sigma_{\ell_2-1,s}\cdots \Sigma_{\ell_1,s}V_{\ell_1} b \right\rvert \\
    & =
     \sup_{a:\lv a \rv=1,\lv a\rv_0 \le k} \sup_{b:\lv b \rv=1,\supp(b) \subseteq S} \left\lvert (a-\lambda(a)+\lambda(a) )^{\top}V_{\ell_2} \Sigma_{\ell_2-1,s}\cdots \Sigma_{\ell_1,s}V_{\ell_1} (b-\zeta(b)+\zeta(b)) \right\rvert \\
    &\le  \sup_{i \in [m],j\in [r]}\left\lvert w_i^{\top} V_{\ell_2} \Sigma_{\ell_2-1,s}\cdots \Sigma_{\ell_1,s}V_{\ell_1} z_j \right\rvert \\
    & \qquad + \sup_{a:\lv a \rv=1,\lv a\rv_0 \le k, j \in [r]}  \left\lvert (a-\lambda(a))^{\top} V_{\ell_2} \Sigma_{\ell_2-1,s}\cdots \Sigma_{\ell_1,s}V_{\ell_1} z_j \right\rvert\\
    & \qquad +\sup_{i \in [m], b:\lv b \rv=1,\supp(b) \subseteq S} \left\lvert w_i^{\top} V_{\ell_2} \Sigma_{\ell_2-1,s}\cdots \Sigma_{\ell_1,s}V_{\ell_1} (b-\zeta(b)) \right\rvert\\
  &\qquad +\sup_{a:\lv a \rv=1,\lv a\rv_0 \le k} \sup_{i \in [m], b:\lv b \rv=1,\supp(b) \subseteq S} \left\lvert (a-\lambda(a))^{\top} V_{\ell_2} \Sigma_{\ell_2-1,s}\cdots \Sigma_{\ell_1,s}V_{\ell_1} (b-\zeta(b)) \right\rvert  \\
  &\overset{(i)}{\le} \frac{4}{L}+
 \frac{\Xi}{4} + \frac{\Xi}{4} +\frac{\Xi}{16} \le \frac{4}{L}+\frac{9}{16}\Xi \numberthis \label{e:part_d_good_event_general_point}
\end{align*}
where 
$(i)$ follows  
by the definition of $\Xi$ along with Lemma~\ref{l:covering_numbers_sparse_vectors}, because we assume that the event in \eqref{e:part_d_good_event_cover} holds, and also because $\lv a-\lambda(a)\rv\le 1/4$ and $\lv b-\zeta(b)\rv\le 1/4$. 

By rearranging terms in the previous display we can infer that 
\begin{align}\label{e:part_d_pre_final_good_event}
   \Xi := \sup_{a:\lv a \rv=1,\lv a\rv_0 \le k} \sup_{b:\lv b \rv=1,\supp(b) \subseteq S}\left\lvert a^{\top} V_{\ell_2} \Sigma_{\ell_2-1,s}\cdots \Sigma_{\ell_1,s}V_{\ell_1} b \right\rvert \le \frac{1}{\left(1-\frac{9}{16} \right)}\frac{4}{L} < \frac{10}{L}
\end{align}
with probability at least $1-e^{-\Omega\left(\frac{p}{L^2}\right)}$.

Finally, when $b$ is an arbitrary unit vector we can partition $[p]=S_1 \cup \ldots \cup S_m$, such that for all $i \in [m]$, $|S_i|\le \floor{\frac{c_1p}{L^2}}$ and the number of the sets in the partition $q \le \frac{p}{\floor{\frac{c_1p}{L^2}}}= \ceil{\frac{L^2}{c_1}}$. Thus, given an arbitrary unit vector $b\in \S^{p-1}$, we can decompose it as $b = u_1 + \ldots + u_q$, where each $u_i$ is non-zero only on the set $S_i$. By invoking the triangle inequality
\begin{align*}
    \lvert a^{\top} V_{\ell_2} \Sigma_{\ell_2-1,s} \cdots \Sigma_{\ell}V_{\ell_1}b \rvert &\le \sum_{i=1}^q \lvert a^{\top}V_{\ell_2} \Sigma_{\ell_2-1} \cdots \Sigma_{\ell_1,s}V_{\ell_1}u_i \rvert.
\end{align*}
By applying the result of \eqref{e:part_d_pre_final_good_event} to each term in the sum above we find that: for all $k$-sparse unit vectors $a$ and all unit vectors $b \in \S^{p-1}$
\begin{align*}
    \lvert a^{\top}V_{\ell_2} \Sigma_{\ell_2-1,s} \cdots \Sigma_{\ell_1,s}V_{\ell_1}b \rvert &\le \frac{10}{L}\sum_{i=1}^q \lv u_i\rv \le \frac{10}{L}\sqrt{q}\left(\sum_{i=1}^q \lv u_i\rv^2  \right)^{1/2}= \frac{10\sqrt{q}}{L} =O(1)
\end{align*}
with probability at least $1-qe^{-\Omega\left(\frac{p}{L^2}\right)}= 1-O(L^2)e^{-\Omega\left(\frac{p}{L^2}\right)}$. In other words for all $k$-sparse unit vectors $a$
\begin{align*}
    \lv a^{\top}V_{\ell_2} \Sigma_{\ell_2-1,s} \cdots \Sigma_{\ell_1,s}V_{\ell_1} \rv &= \sup_{b:\lv b\rv=1} \lvert a^{\top}V_{\ell_2} \Sigma_{\ell_2-1,s} \cdots \Sigma_{\ell_1,s}V_{\ell_1}b \rvert\le O(1)
\end{align*}
with the same probability that is at least $1-O(L^2)e^{-\Omega\left(\frac{p}{L^2}\right)}$. By a union bound over the pairs of layers $\ell_1$ and $\ell_2$ and all samples $s\in [n]$ we establish that,
with probability at least $1-O(nL^4)e^{-\Omega\left(\frac{p}{L^2}\right)}$,
for all pairs $1\le \ell_1\le \ell_2\le L$, all $s\in [n]$ and all $k$-sparse vectors $a$
\begin{align*}
    \left\lv \frac{a^{\top}}{\lv a \rv}V_{\ell_2} \Sigma_{\ell_2-1,s} \cdots \Sigma_{\ell_1,s}V_{\ell_1} \right\rv \le O(1).
\end{align*}
Since, $p \ge \poly\left(L,\log\left(\frac{n}{\delta}\right)\right)$ we can ensure that this happens with probability at least $1-\delta/8$ which completes the proof.
\end{proof}
\subsubsection{Proof of Part~(f)}
\begin{lemma}\label{l:conc_event_part_f}
For any $\delta>0$, if
$p \ge \poly\left(L,\log\left(\frac{n}{\delta}\right)\right)$ for a large enough polynomial, then with probability at least $1-\delta/8$ over the randomness in $V^{(1)}$ we have that for all $s \in [n]$ and all $1\le \ell_1 \le \ell_2 \le L$,
\begin{align*}
        \lvert a^{\top} V_{\ell_2}^{(1)}\Sigma_{\ell_2-1,s}^{V^{(1)}}\cdots \Sigma_{\ell_1,s}^{V^{(1)}}V_{\ell_1}^{(1)}b\rvert \le O\left( \lv a \rv \lv b \rv\sqrt{\frac{k\log(p)}{p}}\right)
     \end{align*}
     for all vectors $a,b$ with $\lv a \rv_{0},\lv b \rv_{0}\le k= \frac{cp}{\log(p) L^2 }$, where $c$ is a small enough positive absolute constant.
\end{lemma}
\begin{proof}Fix the layers $1\le \ell_1\le \ell_2 \le L$ and the sample index $s$. At the end of the proof we shall take a union bound over all pairs of layers and all samples. 
In the proof, let us denote $V_{\ell}^{(1)}$ by 
$V_{\ell}$ and 
$\Sigma_{\ell,s}^{V^{(1)}}$ 
by $\Sigma_{\ell,s}$.

For any fixed vector $z$ we know from Part~(a) of Lemma~\ref{l:analog_of_conc_part_a_with_sigmas} that with probability at least $1-O(nL^3)e^{-\Omega\left(\frac{p}{L^2}\right)}$ over the randomness in $(V_{\ell_2-1},\ldots,V_{1})$
\begin{align} \label{e:part_e_good_event_1}
    \lv \Sigma_{\ell_2-1,s}V_{\ell_2-1}\ldots \Sigma_{\ell_1,s}V_{\ell_1} z\rv\le 2\lv z \rv.
\end{align}
Recall that the entries of $V_{\ell_2}$ are drawn independently from $\cN(0,\frac{2}{p})$. Thus, conditioned on this event above, for any fixed vector $w$ the random variable $w^{\top} V_{\ell_2} \left(\Sigma_{\ell_2-1,s}\cdots \Sigma_{\ell_1}V_{\ell_1} z\right)$ is a mean-zero Gaussian with variance at most $\frac{8\lv w \rv^2 \lv z \rv^2}{p}$. Therefore over the randomness in $V_{\ell_2}$
\begin{align}\label{e:part_e_good_event_2}
    \Pr\left(\left\lvert w^{\top} V_{\ell_2} \Sigma_{\ell_2-1,s}\cdots \Sigma_{\ell_1,s}V_{\ell_1} z \right\rvert \ge \frac{1}{c_2}\sqrt{\frac{k\log(p)}{p}}\lv w\rv\lv z \rv\Big| V_{\ell_1-1},\ldots,V_{1}\right)\le e^{-\frac{k\log(p)}{128 c_2^2}},
\end{align}
where $c_2$ is a small enough positive absolute constant that will be chosen only as a function of the constant $c$. A union bound over the events in \eqref{e:part_e_good_event_1} and \eqref{e:part_e_good_event_2} yields
\begin{align}\nonumber
    \Pr\left(\left\lvert w^{\top} V_{\ell_2} \Sigma_{\ell_2-1,s}\cdots \Sigma_{\ell_1,s}V_{\ell_1} z \right\rvert \le \frac{1}{c_2}\sqrt{\frac{k\log(p)}{p}}\lv w\rv\lv z \rv\right)&\ge 1-O(nL^3)e^{-\Omega\left(\frac{p}{L^2}\right)}-e^{-\frac{k\log(p)}{128 c_2^2}}\\
    & = 1-e^{-\Omega\left(\frac{p}{L^2}\right)}-e^{-\frac{k\log(p)}{128 c_2^2}}\label{e:part_e_good_event_3} 
\end{align}
where the last equality holds since $p\ge \poly\left(L,\log\left(\frac{n}{\delta}\right)\right)$ for a large enough polynomial.

Let $\{w_i\}_{i=1}^m$ be a $1/4$-net of $k$-sparse unit vectors in Euclidean norm of size $m=\binom{p}{k}9^{k}$ (such a net exists, see Lemma~\ref{l:covering_numbers_sparse_vectors}). Therefore by using \eqref{e:part_e_good_event_3} and taking a union bound we find that
\begin{align}\label{e:part_e_good_event_cover}
 \forall \; i,j \in [m], \; \;   \left\lvert w_i^{\top} V_{\ell_2} \Sigma_{\ell_2-1,s}\cdots \Sigma_{\ell_1,s}V_{\ell_1} w_j \right\rvert \le \frac{1}{c_2}\sqrt{\frac{k\log(p)}{p}}
\end{align}
with probability at least
\begin{align*}
    1-m^2\left(e^{-\Omega(p/L^2)}+e^{-\frac{k\log(p)}{128 c_2^2}}\right)&=1-O\left(\left(\binom{p}{k}9^{k}\right)^2\right)(e^{-\Omega\left(\frac{p}{L^2}\right)}+e^{-\frac{k\log(p)}{128 c_2^2}})\\
    &\overset{(i)}{\ge} 1-O\left(\left(\frac{9ep}{k}\right)^{2k}\right)(e^{-\Omega(p/L^2)}+e^{-\frac{k\log(p)}{128 c_2^2}})\\
    &= 1-\left(e^{-\Omega\left(\frac{p}{L^2}\right)+2k\log(9ep)}+e^{-\frac{k\log(p)}{128 c_2^2}+2k\log(9ep)}\right) \\
    &\overset{(ii)}{=} 1-\left(e^{-\Omega\left(\frac{p}{L^2}\right)}+e^{-\Omega(k\log(p))}\right) \overset{(iii)}{=} 1-e^{-\Omega\left(\frac{p}{L^2}\right)}
\end{align*}
where $(i)$ follows since $\binom{p}{k}\le \left(\frac{ep}{k}\right)^k$, $(ii)$ follows since $p\ge \poly\left(L,\log\left(\frac{n}{\delta}\right)\right)$, $k = \frac{cp}{\log(p)L^2}$ and because $c_2$ is a small enough absolute constant (which can be chosen given the constant $c$), and $(iii)$ again follows since $k = \frac{cp}{\log(p)L^2}$.
 
We will now demonstrate that if the ``good event'' in \eqref{e:part_e_good_event_cover} holds then we can use this to establish a similar guarantee for all $k$-sparse unit vectors $a$ and $b$. 
Suppose, for each $k$-sparse unit vector
$w$, that $\zeta(w)$ 
is 
its nearest neighbor in $\{ w_1,\ldots,w_m\}$.
Then,
if
the event in \eqref{e:part_e_good_event_cover} holds,
\begin{align*}
    &\Xi:= \sup_{a:\lv a \rv=1,\lv a\rv_0 \le k} \sup_{b:\lv b \rv=1,\lv b\rv_0 \le k} \left\lvert a^{\top} V_{\ell_2} \Sigma_{\ell_2-1,s}\cdots \Sigma_{\ell_1,s}V_{\ell_1} b \right\rvert \\
    &
     =
     \sup_{a:\lv a \rv=1,\lv a\rv_0 \le k} \sup_{b:\lv b \rv=1,\lv b\rv_0 \le k} \left\lvert (a-\zeta(a)+\zeta(a))^{\top} V_{\ell_2} \Sigma_{\ell_2-1,s}\cdots \Sigma_{\ell_1,s}V_{\ell_1} (b-\zeta(b)+\zeta(b)) \right\rvert \\
    &\le  \sup_{i,j \in [m]}\left\lvert w_i^{\top} V_{\ell_2} \Sigma_{\ell_2-1,s}\cdots \Sigma_{\ell_1,s}V_{\ell_1} w_j \right\rvert \\
    & \qquad 
    + \sup_{a:\lv a \rv=1,\lv a\rv_0 \le k, j \in [m]}  \left\lvert (a-\zeta(a))^{\top} V_{\ell_2} \Sigma_{\ell_2-1,s}\cdots \Sigma_{\ell_1,s}V_{\ell_1} w_j \right\rvert\\
    & \qquad  +\sup_{i \in [m], b:\lv b \rv=1,\lv b\rv_0 \le k} \left\lvert w_i^{\top} V_{\ell_2} \Sigma_{\ell_2-1,s}\cdots \Sigma_{\ell_1,s}V_{\ell_1} (b-\zeta(b)) \right\rvert\\
  & \qquad +\sup_{a:\lv a \rv=1,\lv a\rv_0 \le k} \sup_{b:\lv b \rv=1,\lv b\rv_0 \le k} \left\lvert (a-\zeta(a))^{\top} V_{\ell_2} \Sigma_{\ell_2-1,s}\cdots \Sigma_{\ell_1,s}V_{\ell_1} (b-\zeta(b)) \right\rvert  \\
  &\overset{(i)}{\le} \frac{1}{c_2}\sqrt{\frac{k\log(p)}{p}}+ \frac{\Xi}{4}+\frac{\Xi}{4}+\frac{\Xi}{16}  = \frac{1}{c_2}\sqrt{\frac{k\log(p)}{p}}+\frac{9}{16}\Xi 
\end{align*}
where $(i)$ follows by the definition of $\Xi$ along with Lemma~\ref{l:covering_numbers_sparse_vectors}, because we assume that the event in \eqref{e:part_e_good_event_cover} holds, and also since $\lv a-\zeta(a)\rv\le 1/4$ and $\lv b-\zeta(b)\rv\le 1/4$. 

By rearranging terms in the previous display we can infer that
\begin{align*}
   \Xi = \sup_{a:\lv a \rv=1,\lv a\rv_0 \le k} \sup_{b:\lv b \rv=1,\lv b\rv_0 \le k} \left\lvert a^{\top} V_{\ell_2} \Sigma_{\ell_2-1,s}\cdots \Sigma_{\ell_1,s}V_{\ell_1} b\right\rvert & \le \frac{1}{\left(1-\frac{9}{16} \right)}\frac{1}{c_2}\sqrt{\frac{k\log(p)}{p}}\\ &=  O\left(\sqrt{\frac{k\log(p)}{p}}\right)
\end{align*}
with probability at least $1-e^{-\Omega\left(\frac{p}{L^2}\right)}$. Taking a union bound over all pairs of layers and all sample we find that,
with probability at least $1-O(nL^2)e^{-\Omega\left(\frac{p}{L^2}\right)}$,
for all $1\le \ell_1 \le \ell_2 \le L$, for all $s \in [n]$ and all $k$-sparse vectors $a$ and $b$
\begin{align}\label{e:part_e_final_good_event}
    \left\lvert \frac{a^{\top}}{\lv a \rv} V_{\ell_2} \Sigma_{\ell_2-1,s}\cdots \Sigma_{\ell_1,s}V_{\ell_1} \frac{b}{\lv b\rv} \right\rvert = O\left(\sqrt{\frac{k\log(p)}{p}}\right).
\end{align}
Since $p\ge \poly\left(L,\log\left(\frac{n}{\delta}\right)\right)$ for a large enough polynomial we can ensure that this probability is at least $1-\delta/8$ which completes our proof.
\end{proof}
\subsubsection{Proof of Part~(g)}
\begin{lemma}\label{l:conc_event_part_g}
For any $\delta>0$, if $p \ge \poly\left(L,\log\left(\frac{n}{\delta}\right)\right)$ for a large enough polynomial, then, with probability at least $1-\delta/8$ over the randomness in $V^{(1)}$, for all $s \in [n]$ and all $1\le \ell \le L$,
\begin{align*}
        \lvert  V_{L+1}^{(1)}\Sigma_{L,s}^{V^{(1)}}\cdots \Sigma_{\ell,s}^{V^{(1)}}V_{\ell}^{(1)}a\rvert \le O\left( \lv a \rv \sqrt{k\log(p)}\right)
     \end{align*}
     for all vectors $a$ with $\lv a \rv_{0}\le k= \frac{cp}{\log(p) L^2 }$, where $c$ is a small enough positive absolute constant.
\end{lemma}
\begin{proof}Fix the layer $1\le \ell \le L$ and the sample index $s$. At the end of the proof we shall take a union bound over all  layers and all samples. 
Let us denote $V_{\ell}^{(1)}$ by 
$V_{\ell}$ and 
$\Sigma_{\ell,s}^{V^{(1)}}$ 
by $\Sigma_{\ell,s}$.

For any fixed vector $z$ we know from Part~(a) of Lemma~\ref{l:analog_of_conc_part_a_with_sigmas} that with probability at least $1-O(nL^3)e^{-\Omega\left(\frac{p}{L^2}\right)}$ over the randomness in $(V_{L},\ldots,V_{1})$
\begin{align} \label{e:part_g_good_event_1}
    \lv \Sigma_{L,s}V_{L}\ldots \Sigma_{\ell,s}V_{\ell} z\rv\le 2\lv z \rv.
\end{align}
Recall that the entries of $V_{L+1}$ are drawn independently from $\cN(0,1)$. Thus, conditioned on this event above, for any fixed vector $w$ the random variable $w^{\top} V_{L+1} \left(\Sigma_{L,s}\cdots \Sigma_{\ell}V_{\ell} z\right)$ is a mean-zero Gaussian with variance at most $4\lv z \rv^2$. Therefore over the randomness in $V_{L+1}$
\begin{align}\label{e:part_g_good_event_2}
    \Pr\left(\left\lvert V_{L+1} \Sigma_{L,s}\cdots \Sigma_{\ell,s}V_{\ell} z \right\rvert \ge \frac{\sqrt{k\log(p)}}{c_2} \lv z \rv\Big| V_{L},\ldots,V_{1}\right)\le e^{-\frac{k\log(p)}{32 c_2^2}},
\end{align}
where $c_2$ is a small enough positive absolute constant that will be chosen only as a function of the constant $c$. A union bound over the events in \eqref{e:part_g_good_event_1} and \eqref{e:part_g_good_event_2} yields
\begin{align}\nonumber
    \Pr\left(\left\lvert  V_{L+1} \Sigma_{L,s}\cdots \Sigma_{\ell,s}V_{\ell} z \right\rvert \le \frac{\sqrt{k\log(p)}}{c_2}\lv z \rv\right)&\ge 1-O(nL^3)e^{-\Omega\left(\frac{p}{L^2}\right)}-e^{-\frac{k\log(p)}{32 c_2^2}}\\
    & = 1-e^{-\Omega\left(\frac{p}{L^2}\right)}-e^{-\frac{k\log(p)}{32 c_2^2}}\label{e:part_g_good_event_3} 
\end{align}
where the last equality holds since $p\ge \poly\left(L,\log\left(\frac{n}{\delta}\right)\right)$ for a large enough polynomial.

Let $\{z_i\}_{i=1}^m$ be a $1/4$-net of $k$-sparse unit vectors in Euclidean norm of size $m=\binom{p}{k}9^{k}$ (such a net exists, see Lemma~\ref{l:covering_numbers_sparse_vectors}). Therefore by using \eqref{e:part_g_good_event_3} and taking a union bound we find that
\begin{align}\label{e:part_g_good_event_cover}
    \forall \; i\in [m], \; \;\left\lvert  V_{L+1} \Sigma_{L,s}\cdots \Sigma_{\ell,s}V_{\ell} z_i \right\rvert \le \frac{\sqrt{k\log(p)}}{c_2}
\end{align}
with probability at least
\begin{align*}
    1-m\left(e^{-\Omega(p/L^2)}+e^{-\frac{k\log(p)}{128 c_2^2}}\right)&=1-O\left(\binom{p}{k}9^{k}\right)(e^{-\Omega\left(\frac{p}{L^2}\right)}+e^{-\frac{k\log(p)}{32 c_2^2}})\\
    &\overset{(i)}{\ge} 1-O\left(\left(\frac{9ep}{k}\right)^{k}\right)(e^{-\Omega(p/L^2)}+e^{-\frac{k\log(p)}{32 c_2^2}})\\
    &= 1-\left(e^{-\Omega\left(\frac{p}{L^2}\right)+k\log(9ep)}+e^{-\frac{k\log(p)}{32 c_2^2}+k\log(9ep)}\right) \\
    &\overset{(ii)}{=} 1-\left(e^{-\Omega\left(\frac{p}{L^2}\right)}+e^{-\Omega(k\log(p))}\right) \overset{(iii)}{=} 1-e^{-\Omega\left(\frac{p}{L^2}\right)}
\end{align*}
where $(i)$ follows since $\binom{p}{k}\le \left(\frac{ep}{k}\right)^k$, $(ii)$ follows since $p\ge \poly\left(L,\log\left(\frac{n}{\delta}\right)\right)$, $k = \frac{cp}{\log(p)L^2}$ and because $c_2$ is a small enough absolute constant (which can be chosen given the constant $c$), and $(iii)$ follows again since $k = \frac{cp}{\log(p)L^2}$.
 
We will now demonstrate that if the ``good event'' \eqref{e:part_g_good_event_cover} holds then we can use this to establish a similar guarantee for all $k$-sparse unit vectors $a$. To see this, 
as before, suppose $\zeta$ maps any unit-length
$k$-sparse vector to 
its nearest neighbor in $\{ z_1,\ldots,z_r\}$.
Suppose that the event in \eqref{e:part_g_good_event_cover} holds then
\begin{align*}
    &\Xi:= \sup_{a:\lv a \rv=1,\lv a\rv_0 \le k} \left\lvert  V_{L+1} \Sigma_{L,s}\cdots \Sigma_{\ell,s}V_{\ell} a \right\rvert \\
    &= \sup_{a:\lv a \rv=1,\lv a\rv_0 \le k}  \left\lvert  V_{L+1} \Sigma_{L,s}\cdots \Sigma_{\ell,s}V_{\ell} (a-\zeta(a)+\zeta(a)) \right\rvert \\
    &\le  \sup_{i \in [m]}\left\lvert  V_{L+1} \Sigma_{L,s}\cdots \Sigma_{\ell,s}V_{\ell} z_j \right\rvert 
 +\sup_{a:\lv a \rv=1,\lv a\rv_0 \le k} \left\lvert V_{L+1} \Sigma_{L,s}\cdots \Sigma_{\ell,s}V_{\ell} (a-\zeta(a)) \right\rvert\\
  &\overset{(i)}{\le} \frac{\sqrt{k\log(p)}}{c_2}+\frac{1}{4} \sup_{a:\lv a \rv=1,\lv a\rv_0 \le k} \left\lvert V_{L+1} \Sigma_{L,s}\cdots \Sigma_{\ell,s}V_{\ell} \frac{a-\zeta(a)}{\lv a-\zeta(a)\rv} \right\rvert\overset{(ii)}{\le} \frac{\sqrt{k\log(p)}}{c_2} +\frac{\Xi}{4}
\end{align*}
where $(i)$ holds because we assume that the event in \eqref{e:part_g_good_event_cover} holds and since $\lv a-\zeta(a)\rv\le 1/4$, and $(ii)$ follows by the definition of $\Xi$ along with Lemma~\ref{l:covering_numbers_sparse_vectors}. 

By rearranging terms in the previous display we infer that
\begin{align*}
   \Xi = \sup_{a:\lv a \rv=1,\lv a\rv_0 \le k}  \left\lvert  V_{L+1} \Sigma_{L,s}\cdots \Sigma_{\ell,s}V_{\ell} a\right\rvert & \le \frac{1}{\left(1-\frac{1}{4} \right)}\frac{\sqrt{k\log(p)}}{c_2}=  O\left(\sqrt{k\log(p)}\right)
\end{align*}
with probability at least $1-e^{-\Omega\left(\frac{p}{L^2}\right)}$. Taking a union bound over all layers and all sample we find that, for all $1\le \ell \le L$, for all $s \in [n]$ and all $k$-sparse vectors $a$ 
\begin{align}\label{e:part_g_final_good_event}
    \left\lvert V_{L+1} \Sigma_{L,s}\cdots \Sigma_{\ell,s}V_{\ell} \frac{a}{\lv a\rv} \right\rvert = O\left(\sqrt{k\log(p)}\right)
\end{align}
with probability at least $1-O(nL)e^{-\Omega\left(\frac{p}{L^2}\right)}$. Since $p\ge \poly\left(L,\log\left(\frac{n}{\delta}\right)\right)$ for a large enough polynomial we can ensure that this probability is at least $1-\delta/8$ which completes our proof.
\end{proof}
\subsubsection{Proof of Part~(h)} 
\begin{lemma}\label{l:conc_event_part_h}
For any $\delta>0$, suppose that $h< \frac{1}{50\sqrt{p}L}$, $p \ge \poly\left(L,\log\left(\frac{n}{\delta}\right)\right)$ for a large enough polynomial and $\tau = \Omega\left(\frac{\log^{2}(\frac{nL}{\delta})}{p^{\frac{3}{2}}L^{3}}\right)$. 
%
For $\beta =O\left(\frac{L^2 \tau^{2/3}}{\sqrt{p}}\right)$, if
\begin{align*}
    \cS_{\ell,s}(\beta):= \left\{j\in [p]: |V^{(1)}_{\ell,j} x_{\ell,s}^{V^{(1)}}| \le \beta \right\}
\end{align*}
where $V^{(1)}_{\ell,j}$ refers to the $j$th row of $V^{(1)}_{\ell}$, then with probability at least $1-\delta/8$ over the randomness in $V^{(1)}$ we have that for all $\ell \in [L]$ and all $s \in [n]$: $$|\cS_{\ell,s}(\beta)|\le O(p^{3/2}\beta)=O(p L^2 \tau^{2/3}).$$
\end{lemma}
\begin{proof} To ease notation let us refer to $V_{\ell}^{(1)}$ as $V_{\ell}$ and $x_{\ell,s}^{V^{(1)}}$ as $x_{\ell,s}$. For a fixed $\ell \in [L]$ and sample $s \in [n]$ define $$Z(\ell,j,s):= \mathbb{I}\left[|V_{\ell,j} x_{\ell,s}|\le \beta\right]$$
so
that $|\cS_{\ell,s}(\beta)| = \sum_{j=1}^p Z(j,\ell,s)$. Define 
$\cE$ to be the event that $\lv x_{\ell-1,s}\rv\ge \frac{1}{2}$.
By inequality~\eqref{e:bound_on_post_activation_feature_high_probability_statement} in the proof of Lemma~\ref{l:conc_event_part_a} above 
\begin{align}\label{e:part_f_good_event_norm_of_x}
    \Pr\left[\cE \right] \ge 1-O(nL)\exp\left(-\Omega\left(\frac{p}{L^2}\right)\right).
\end{align}
 Conditioned on $x_{\ell-1,s}$ since each entry of $V_{\ell,j}$ is drawn independently from $\cN(0,\frac{2}{p})$ we know that the distribution of $V_{\ell,j}x_{\ell-1,s} \sim \cN\left(0,\frac{2\lv x_{\ell-1,s}\rv^2}{p}\right)$. Thus, conditioned on the event $\cE$,
 which is determined by the random weights {\em before} layer $\ell$,
 we have that
\begin{align*}
    \mathbb{E}\left[ Z(j,\ell,s) \; \big| \; \cE \right]  
    = \Pr\left[j \in \cS_{\ell,s}(\beta)  \; \big| \; \cE \right]&= \sqrt{\frac{p}{4\pi\lv x_{\ell-1,s}\rv^2}}\int_{-\beta}^{\beta} \exp\left(-\frac{x^2 p}{4 \lv x_{\ell-1,s}\rv^2}\right) \mathrm{d} x \\ &\le \sqrt{\frac{p}{\pi}}\int_{-\beta}^{\beta} \exp\left(-\frac{x^2 p}{4 \lv x_{\ell-1,s}\rv^2}\right) \mathrm{d} x \le 2\beta \sqrt{\frac{p}{\pi}}.
\end{align*}
On applying Hoeffding's inequality (see Theorem~\ref{thm:hoeffding}) we find that
\begin{align*}
  &\Pr\left[ \left|\cS_{\ell,s}(\beta)\right| \le \E\mleft[\sum_{j=1}^{p}Z(j,\ell,s)\;  \Big| \; \cE\mright]+p^{3/2}\beta  \le p\left(2\beta \sqrt{\frac{p}{\pi}}\right) +p^{3/2}\beta \le 3p^{3/2}\beta \;  \bigg| \; \cE \right] \\
  & \qquad \qquad \ge 1-\exp(-\Omega(p^{3/2}\beta)). \numberthis \label{e:part_f_one_sample_one_layer_conditional_bound}
\end{align*}
Taking a union bound over the events in \eqref{e:part_f_good_event_norm_of_x} and \eqref{e:part_f_one_sample_one_layer_conditional_bound} we find that 
\begin{align*}
   \Pr\left[ |\cS_{\ell,s}(\beta)| \le 3p^{3/2}\beta\right] \ge 1-\exp(-\Omega(p^{3/2}\beta))-O(nL)\exp\left(-\Omega\left(\frac{p}{L^2}\right)\right).
\end{align*}
Applying a union bound over all samples and all layers we find that,
with probability at least $1-O(nL)\exp(-\Omega(p^{3/2}\beta))-O(n^2L^2)\exp\left(-\Omega\left(\frac{p}{L^2}\right)\right)$,
for all $\ell \in [L]$ and all $s \in [n]$,
\begin{align*}
    |\cS_{\ell,s}(\beta)| \le 3p^{3/2}\beta = O(p L^2 \tau^{2/3}).
\end{align*}
We shall now demonstrate that this probability of success is at least $1-\delta/8$. On substituting the value of $\beta = O(\frac{L^{8/3} \tau^{2/3}}{\sqrt{p}})$ we find that this probability is at least
\begin{align*}
    &1-O(nL)\exp(-\Omega(p^{3/2}\beta))-O(n^2L^2)\exp\left(-\Omega\left(\frac{p}{L^2}\right)\right) \\& \qquad \quad= 1-O(nL)\exp(-\Omega(p L^2 \tau^{2/3}))-O(n^2L^2)\exp\left(-\Omega\left(\frac{p}{L^2}\right)\right)\\
    & \qquad \quad \overset{(i)}{=} 1-O(nL)\exp(-\Omega(p L^2 \tau^{2/3}))-\exp\left(-\Omega\left(\frac{p}{L^2}\right)\right)\\
    & \qquad \quad \overset{(ii)}{=} 1-O(nL)\exp\left(-\Omega\left(\log^{\frac{4}{3}}\left( \frac{nL}{\delta}\right)\right)\right)-\exp\left(-\Omega\left(\frac{p}{L^2}\right)\right)\\& \qquad \quad\ge 1-\delta/8, 
\end{align*}
where $(i)$ follows since $p \ge \poly\left(L,\log\left(\frac{n}{\delta}\right)\right)$ for a large enough polynomial and $(ii)$ follows by assumption that $\tau = \Omega\left(\frac{\log^{2}(\frac{nL}{\delta})}{p^{\frac{3}{2}}L^3}\right)$. This completes our proof.
\end{proof}

\subsubsection{Other Useful Concentration Lemmas} \label{sss:other_concentration_lemmas}
The following lemma is useful in the proofs of Lemmas~\ref{l:conc_event_part_c}-\ref{l:conc_event_part_g}. It bounds the norm of an arbitrary unit vector $z$ that is multiplied by alternating weight matrices $V_{\ell}^{(1)}$ and corresponding 
$\Sigma^{V^{(1)}}_{\ell,s}$.
\begin{lemma} \label{l:analog_of_conc_part_a_with_sigmas}
If
$p \ge \poly(L,\log(n))$ for a large enough polynomial, then given an arbitrary unit vector $z\in \S^{p-1}$, with probability at least $1-O(nL^3)\exp\left(-\Omega\left( \frac{p}{L^2}\right)\right)$ over the randomness in $V^{(1)}$, for all $1\le \ell_1 \le \ell_2 \le L$ and for all $s\in [n]$,
\begin{enumerate}[(a)]
\item $\left\lv \Sigma_{\ell_2-1,s}^{V^{(1)}} \cdots \Sigma_{\ell_1,s}^{V^{(1)}} V_{\ell_1}^{(1)}z\right\rv \le 2,$ and
    \item $
         \left\lv V^{(1)}_{\ell_2} \Sigma_{\ell_2-1,s}^{V^{(1)}} \cdots \Sigma_{\ell_1,s}^{V^{(1)}} V_{\ell_1}^{(1)}z\right\rv \le 2.$
\end{enumerate}
     
\end{lemma}
\begin{proof} 
We denote $V_{\ell}^{(1)}$ by 
$V_{\ell}$, 
$\Sigma_{\ell,s}^{V^{(1)}}$ 
by $\Sigma_{\ell,s}$, and 
$x_{\ell,s}^{V^{(1)}}$ 
by
$x_{\ell,s}$. 

\textit{Proof of Part~(a):} For any layer $\ell \in \{\ell_1,\ldots,\ell_2-1\}$ define $$z_{\ell,s}:= \Sigma_{\ell,s}V_{\ell}\Sigma_{\ell-1,s}\cdots\Sigma_{\ell_1,s}V_{\ell_1}z$$
with the convention that $z_{\ell_1-1,s}:=z$.

Conditioned on $z_{\ell-1,s}$ the distribution of $V_{\ell}z_{\ell-1,s} \sim \cN\left(0,\frac{2\lv z_{\ell-1,s}\rv^2 I}{p}\right)$, since each entry of $V_{\ell}$ is drawn independently from $\cN(0,\frac{2}{p})$. We begin by evaluating the expected value of its squared norm conditioned on the randomness in $V_{\ell-1},\ldots,V_{1}$. Let $V_{\ell,j}$ denote the $j$th row of $V_{\ell}$ and let $(\Sigma_{\ell,s})_{jj}$ denote the $j$th element on the diagonal of $\Sigma_{\ell,s}$, then
 \begin{align*}
     \E\left[\lv z_{\ell,s} \rv^2 |V_{\ell-1},\ldots,V_{1}\right] & = \E\left[\lv \Sigma_{\ell,s}V_{\ell} z_{\ell-1,s} \rv^2 |V_{\ell-1},\ldots,V_{1}\right]\\   &= \E\left[\sum_{j=1}^p \left((\Sigma_{\ell,s})_{jj}V_{\ell,j} z_{\ell-1,s} \right)^2 \big|V_{\ell-1},\ldots,V_{1}\right].
 \end{align*}
 By the definition of the Huberized ReLU observe that each entry $$(\Sigma_{\ell,s})_{jj}= \phi'(V_{\ell,j}x_{\ell,s}) \le \mathbb{I}\left[V_{\ell,j}x_{\ell-1,s} \ge 0\right]$$ and therefore
 \begin{align} \label{e:aux_lemma_upper_bound_in_terms_of_indicator}
     \E\left[\lv z_{\ell,s} \rv^2 |V_{\ell-1},\ldots,V_{1}\right]
     & \le \E\left[\sum_{j=1}^p \mathbb{I}\left[V_{\ell,j}x_{\ell-1,s} \ge 0\right] \left(V_{\ell,j} z_{\ell-1,s} \right)^2 \Big|V_{\ell-1},\ldots,V_{1}\right].
 \end{align}
Let us decompose 
$V_{\ell,j}$ into its component in
the $x_{\ell-1,s}$ direction, and its
a component that is perpendicular to
$x_{\ell-1,s}$.  That is, define
\begin{align*}
V_{\ell,j}^{\parallel}
 := \left( V_{\ell,j} \cdot
 \frac{x_{\ell-1,s}}{\lv x_{\ell-1,s}\rv}\right)
 \frac{x_{\ell-1,s}}{\lv x_{\ell-1,s}\rv} \quad\text{and}\quad
V_{\ell,j}^{\perp} := 
V_{\ell,j} - V_{\ell,j}^{\parallel}.
\end{align*}
Since $V_{\ell,j}$ is Gaussian,
$V_{\ell,j}^{\parallel}$ and 
$V_{\ell,j}^{\perp}$ are conditionally independent
given the previous layers. 
 
Thus,
 continuing from inequality~\eqref{e:aux_lemma_upper_bound_in_terms_of_indicator}, we have
 \begin{align*}
      & \E\left[\lv z_{\ell,s} \rv^2 |V_{\ell-1},\ldots,V_{1}\right] \\
      & \le \E\left[\sum_{j=1}^p \mathbb{I}\left[V_{\ell,j}x_{\ell-1,s}^{V^{(1)}} \ge 0\right] \left((V_{\ell,j}^{\parallel} + V_{\ell,j}^{\perp} ) z_{\ell-1,s}\right)^2 \big|V_{\ell-1},\ldots,V_{1}\right]\\
      & = \E\left[\sum_{j=1}^p \mathbb{I}\left[V_{\ell,j}^{\parallel} x_{\ell-1,s}^{V^{(1)}} \ge 0\right] \left((V_{\ell,j}^{\parallel} + V_{\ell,j}^{\perp}) z_{\ell-1,s} \right)^2 \big|V_{\ell-1},\ldots,V_{1}\right]\\
      & = \E\left[\sum_{j=1}^p \mathbb{I}\left[V_{\ell,j}^{\parallel} x_{\ell-1,s}^{V^{(1)}} \ge 0\right] \left(V_{\ell,j}^{\parallel}  z_{\ell-1,s}\right)^2 \big|V_{\ell-1},\ldots,V_{1}\right]\\
      & \qquad +\E\left[\sum_{j=1}^p \mathbb{I}\left[V_{\ell,j}^{\parallel} x_{\ell-1,s}^{V^{(1)}} \ge 0\right]  \left( V_{\ell,j}^{\perp}  z_{\ell-1,s}\right)^2 \big|V_{\ell-1},\ldots,V_{1}\right] \\
      & \qquad +2 \E\left[\sum_{j=1}^p \mathbb{I}\left[V_{\ell,j}^{\parallel} x_{\ell-1,s}^{V^{(1)}} \ge 0\right]  
       \left( V_{\ell,j}^{\parallel}  z_{\ell-1,s}\right)
       \left( V_{\ell,j}^{\perp}  z_{\ell-1,s}\right)
       \big|V_{\ell-1},\ldots,V_{1}\right] \\
      & = \E\left[\sum_{j=1}^p \mathbb{I}\left[V_{\ell,j} x_{\ell-1,s}^{V^{(1)}} \ge 0\right] \left(V_{\ell,j}^{\parallel}  z_{\ell-1,s}\right)^2 \big|V_{\ell-1},\ldots,V_{1}\right]\\
      & \qquad +\E\left[\sum_{j=1}^p \mathbb{I}\left[V_{\ell,j} x_{\ell-1,s}^{V^{(1)}} \ge 0\right]  \left( V_{\ell,j}^{\perp}  z_{\ell-1,s}\right)^2 \big|V_{\ell-1},\ldots,V_{1}\right],
       \numberthis \label{e:expected_value_decomposition.V}
 \end{align*}
 since, after conditioning on
 $V_{\ell-1},\ldots,V_{1}$, we have that
 $V_{\ell,j}^{\parallel}$ and
 $V_{\ell,j}^{\perp}$ 
 independent 
 and $V_{\ell,j}^{\perp}  z_{\ell-1,s}$
 is zero mean.
 
Now, decompose
the vector $z_{\ell-1,s}$ into {\em its} 
component
in the $x_{\ell-1,s}$ direction, which we refer to
 as $z_{\ell-1,s}^{\parallel}$, and a component that is perpendicular to $x_{\ell-1,s}$, which we refer to as $z_{\ell-1,s}^{\perp}$. That is, define $$z_{\ell-1,s}^{\parallel} := 
 \left( z_{\ell-1,s}\cdot
 \frac{x_{\ell-1,s}}{\lv x_{\ell-1,s}\rv}\right)
 \frac{x_{\ell-1,s}}{\lv x_{\ell-1,s}\rv}  \qquad \text{ and} \qquad z_{\ell-1,s}^{\perp} := z_{\ell-1,s}-z_{\ell-1,s}^{\parallel}.$$
Since $V_{\ell,j}^{\parallel} z_{\ell-1,s}
 = V_{\ell,j} z_{\ell-1,s}^{\parallel}$
and $
V_{\ell,j}^{\perp} z_{\ell-1,s}
 = V_{\ell,j} z_{\ell-1,s}^{\perp}$,
inequality~\eqref{e:expected_value_decomposition.V}
 implies
 \begin{align*}
      & \E\left[\lv z_{\ell,s} \rv^2 |V_{\ell-1},\ldots,V_{1}\right] \\
      & \quad\le \E\left[\sum_{j=1}^p \mathbb{I}\left[V_{\ell,j} x_{\ell-1,s}^{V^{(1)}} \ge 0\right] \left(V_{\ell,j}  z_{\ell-1,s}^{\parallel} \right)^2 \big|V_{\ell-1},\ldots,V_{1}\right]\\
      & \quad \qquad +\E\left[\sum_{j=1}^p \mathbb{I}\left[V_{\ell,j} x_{\ell-1,s}^{V^{(1)}} \ge 0\right]  \left( V_{\ell,j}  z_{\ell-1,s}^{\perp}\right)^2 \big|V_{\ell-1},\ldots,V_{1}\right].
       \numberthis \label{e:expected_value_decomposition}
\end{align*}

We begin by evaluating the term involving the parallel components. 
 For any $j$, conditioned on $V_{\ell-1},\ldots,V_{1}$, recalling
 that $V_{\ell,j}$ is the $j$th row
 of $V_{\ell}$, the random variable $V_{\ell,j}x_{\ell-1,s} \sim \cN\left(0,\frac{2\lv x_{\ell-1,s}\rv^2}{p}\right)$, and therefore
 \begin{align*}
     &\E\left[\mathbb{I}\left[V_{\ell,j}x_{\ell-1,s} \ge 0\right] \left(V_{\ell,j} z_{\ell-1,s}^{\parallel} \right)^2 \big|V_{\ell-1},\ldots,V_{1}\right] \\& = \left(z_{\ell-1,s}\cdot \left(\frac{x_{\ell-1,s}}{\lv x_{\ell-1,s}\rv^2}\right)\right)^2 \E\left[\mathbb{I}\left[V_{\ell,j}x_{\ell-1,s} \ge 0\right] \left(V_{\ell,j} x_{\ell-1,s} \right)^2 \big|V_{\ell-1},\ldots,V_{1}\right]\\
     &= \left(z_{\ell-1,s}\cdot \left(\frac{x_{\ell-1,s}}{\lv x_{\ell-1,s}\rv^2}\right)\right)^2 \frac{1}{2} \times \E\left[\left(V_{\ell,j} x_{\ell-1,s} \right)^2 \big|V_{\ell-1},\ldots,V_{1}\right] \\
     & = \left(z_{\ell-1,s}\cdot \left(\frac{x_{\ell-1,s}}{\lv x_{\ell-1,s}\rv^2}\right)\right)^2 \frac{1}{2} \times \frac{2\lv x_{\ell-1,s}\rv^2}{p} = \frac{\lv z_{\ell-1,s}^{\parallel}\rv^2}{p}. \label{e:parallel_component_bound} \numberthis 
 \end{align*}
 For the perpendicular component, notice that, conditioned on $(V_{\ell-1},\ldots,V_{1})$, 
 we have that
 $V_{\ell,j} z_{\ell-1,s}^{\perp} = V_{\ell,j}^{\perp} z_{\ell-1,s}$ 
and $\mathbb{I} \left[V_{\ell,j}x_{\ell-1,s} \ge 0 \right] = \mathbb{I} \left[V_{\ell,j}^{\parallel} x_{\ell-1,s} \ge 0 \right]$ are independent,
 and hence 
 \begin{align*}
        &\E\left[\mathbb{I}\left[V_{\ell,j}x_{\ell-1,s} \ge 0\right] \left(V_{\ell,j} z_{\ell-1,s}^{\perp} \right)^2 \big|V_{\ell-1},\ldots,V_{1}\right] \\& = \frac{1}{2} \E\left[\left(V_{\ell,j} z_{\ell-1,s}^{\perp} \right)^2 \big|V_{\ell-1},\ldots,V_{1}\right] = \frac{1}{2}\times \frac{2\lv z_{\ell-1,s}^{\perp}\rv^2}{p} = \frac{\lv z_{\ell-1,s}^{\perp}\rv^2}{p}. \label{e:perp_component_bound} \numberthis 
 \end{align*}
 By combining the results of \eqref{e:expected_value_decomposition}-\eqref{e:perp_component_bound} we find that 
 \begin{align*}
      \E\left[\lv z_{\ell,s} \rv^2 |V_{\ell-1},\ldots,V_{1}\right] &\le p\left(\frac{\lv z_{\ell-1,s}^{\perp}\rv^2+\lv z_{\ell-1,s}^{\parallel}\rv^2}{p}\right) = \lv z_{\ell-1,s}\rv^2. \label{e:expected_value_norm_squared_bound} \numberthis
 \end{align*}
By symmetry among the $p$ coordinates we can also infer that $\E\left[ (z_{\ell,s})_{i}^2 |V_{\ell-1},\ldots,V_{1}\right]\le \lv z_{\ell-1,s}\rv^2/p$ for each $i \in [p]$. Thus, by the same argument as we used in Lemma~\ref{l:conc_event_part_a} to arrive at \eqref{e:sub-Gaussian_norm} we can show that conditioned on $V_{\ell-1},\ldots,V_{1}$ the sub-Gaussian norm $\lv (z_{\ell,s})_i\rv_{\psi_2}$ is at most $c_1\lv z_{\ell-1,s}\rv/\sqrt{p}$ and hence the sub-exponential norm $\lv (z_{\ell,s})_i^2\rv_{\psi_1} \le \lv (z_{\ell,s})_i\rv_{\psi_2}^2 \le c_2 \lv z_{\ell-1,s}\rv^2/p$ (by Lemma~\ref{l:sub_gaussian_squared}). Therefore by Bernstein's inequality (see Theorem~\ref{thm:bernstein}) for any $\eta \in (0,1]$
\begin{align*}
    \Pr\left[\lv z_{\ell,s}\rv^2 \le  \lv z_{\ell-1,s}\rv^2(1+ \eta) \big| V_{\ell-1},\ldots,V_{1} \right] \ge 1-\exp\left( -c_3 p\eta^2\right).
\end{align*}
Setting $\eta=\frac{1}{50L}$ and taking a union bound we infer that
\begin{align}
    \nonumber &\Pr\left[
    \forall s \in [n], \ell \in \{\ell_2-1,\ldots,\ell_1\}, \;
    \lv z_{\ell,s}\rv \le  \lv z_{\ell-1,s}\rv
    \sqrt{1+ \eta} \right] \\ &\qquad \qquad \ge 1-O(nL)\exp\left( -\Omega\left(\frac{p}{L^2}\right)\right). \label{good_event_sigma_feature_analog_proof_1}
\end{align}
 We will now show by an inductive argument for the layers that if the ``good event'' in \eqref{good_event_sigma_feature_analog_proof_1} holds then $\lv z_{\ell,s}\rv \le 1+3(\ell-\ell_1+1) \eta$, for all $\ell \in \{\ell_1-1,\ldots,\ell_2-1\}$ and all $s \in [n]$. The base case holds at $\ell_1-1$ since by definition $\lv z_{\ell_1-1,s}\rv = \lv z\rv =1$. Now assume that the inductive argument holds at any layers $\ell_1,\ldots,\ell-1$. Then if the event in \eqref{good_event_sigma_feature_analog_proof_1} holds we have
 \begin{align*}
   \lv z_{\ell,s}\rv &\le   \lv z_{\ell-1,s}\rv \sqrt{1+\eta} \\ &\le (1+3(\ell-\ell_1) \eta)\left(1+\eta\right) \qquad \mbox{(by the IH and because $\sqrt{1+\eta}\le 1+\eta$)}\\
   &= 1+3\left(\ell-\ell_1+\frac{1}{3}\right)\eta + 3(\ell-\ell_1)\eta^2 \\
  &\le  1+3(\ell-\ell_1+1)\eta \qquad  \mbox{(since $\eta = \frac{1}{50 L}$ and $L\ge 1$).}
 \end{align*}
  This completes the induction. Hence we have shown that for all 
 \begin{align}
        \Pr\left[
        \forall s \in [n],\;
        \lv z_{\ell_2-1,s}\rv \le 1+\frac{3(\ell_2-\ell_1+1)}{50L} \right] \ge 1-O(nL)\exp\left( -\Omega\left(\frac{p}{L^2}\right)\right). \label{good_event_sigma_feature_analog_proof_2}
 \end{align}
 Recall that $z_{\ell_2-1,s}:= \Sigma_{\ell_2-1}\cdots\Sigma_{\ell_1,s}V_{\ell_1}z$, therefore taking union bound over all pairs of layers we get that,
 with probability at least 
 $1-O(nL^3)\exp\left( -\Omega\left(\frac{p}{L^2}\right)\right)$,
 for all $1\le \ell_1\le \ell_2\le L+1$ and all $s\in [n]$,
\begin{align*}
\lv z_{\ell_2-1,s}\rv \le 1+\frac{3(\ell_2-\ell_1+1)}{50L} .
 \end{align*}
 This completes the proof of the first part of the lemma.
 
 \textit{Proof of Part~(b):} For a fixed $s \in [n]$ we condition on $z_{\ell_2-1,s}$ and consider the random variable $a_s = V_{\ell_2} z_{\ell_2-1,s}$. Since $\ell_2 \in [L]$ each entry of $V_{\ell}$ is drawn independently from $\cN(0,\frac{2}{p})$. 
 
 The distribution of each entry of $a_s$ conditioned on $z_{\ell_2-1,s}$ is $\cN\left(0,\frac{2\lv z_{\ell_2-1,s}\rv^2}{p}\right)$. Therefore by the Gaussian-Lipschitz concentration inequality (see Theorem~\ref{thm:gaussian_concentration}) for any $\eta' >0$
 \begin{align*}
     \Pr\left[\lv a_s\rv \le \sqrt{2}\lv z_{\ell_2-1,s}\rv  (1+\eta') \; \big| \; z_{\ell_2-1,s}\right] \ge 1-\exp\left(-c_4 p\eta'^2\right).
 \end{align*}
 Setting $\eta'= \frac{1}{50L}$ and taking a union bound over all samples we get that
\begin{align*}
     \Pr\left[
     \forall s \in [n],\;
     \lv a_s\rv \le \sqrt{2}\lv z_{\ell_2-1,s}\rv  \left(1+\frac{1}{50L}\right) \big| z_{\ell_2-1,s}\right] \ge 1-n\exp\left(- \frac{c_4p}{L^2}\right). \label{good_event_sigma_feature_analog_proof_3} \numberthis
 \end{align*}
By a union bound over the events in \eqref{good_event_sigma_feature_analog_proof_2} and \eqref{good_event_sigma_feature_analog_proof_3} we find that 
 \begin{align*}
        \Pr\left[\forall s \in [n],\;
        \lv a_s\rv \le \sqrt{2}\left( 1+\frac{3(\ell_2-\ell_1+1)}{50L}\right) \left(1+\frac{1}{50L}\right) \right] \ge 1-O(nL)\exp\left(-\Omega\left( \frac{p}{L^2}\right)\right).
 \end{align*}
The definition of $a_s = V_{\ell_2}z_{\ell_2-1} =V_{\ell_2}\Sigma_{\ell_2-1,s}\cdots\Sigma_{\ell_1,s}V_{\ell_1,s}$ and the previous display above yields that 
 \begin{align*}
        &\Pr\left[
        \forall s \in [n],\;
        \lv V_{\ell_2}z_{\ell_2-1} =V_{\ell_2}\Sigma_{\ell_2-1,s}\cdots\Sigma_{\ell_1,s}V_{\ell_1,s} \rv \le \sqrt{2}\left( 1+\frac{3L}{50L}\right)\left( 1+\frac{1}{50L}\right)\le 2 \right]\\ & \qquad \qquad \ge 1-O(nL)\exp\left(-\Omega\left( \frac{p}{L^2}\right)\right).
 \end{align*}
 Finally a union bound over all pairs of  $1\le \ell_1\le \ell_2\le L$  completes the proof of the second part.
\end{proof}

The next lemma bounds the magnitude of the initial function values with high probability.
\begin{lemma}
\label{l:bound_on_magnitude_of_function}For any $\delta >0$, suppose that $p \ge \poly\left(L,\log\left(\frac{n}{\delta}\right)\right)$ for a large enough polynomial, then with probability at least $1-\delta$ over the randomness in $V^{(1)}$ for all $s \in [n]$, $$\lvert f_{V^{(1)}}(x_s)\rvert \le c \sqrt{\log(2n/\delta)}.$$
\end{lemma}
\begin{proof}By Lemma~\ref{l:conc_event_part_a}, with probability at least $1-\delta/8$
\begin{align} \label{e:x_L_norm_large}
    \lv x_{L,s}^{V^{(1)}}\rv \le 2
\end{align}
for all $s \in [n]$. Fix a sample with index $s\in [n]$. Conditioned on $x_{L,s}^{V^{(1)}}$, the random variable $V_{L+1}^{(1)}x_{L,s}^{V^{(1)}} \sim \cN(0,\lv x_{L,s}^{V^{(1)}}\rv^2)$ since each entry of $V_{L+1}^{(1)}$ is drawn independently from $\cN(0,1)$. Therefore for any $\eta >0$
\begin{align*}
    \Pr\left[ |f_{V^{(1)}}(x_s)| \le \eta \lv x_{L,s}^{V^{(1)}}\rv^2 \; \big| \; x_{L,s}^{V^{(1)}} \right] \ge 1- 2\exp\left(-c_1 \eta^2 \right).
\end{align*}
A union bound over all samples implies 
\begin{align*}
    \Pr\left[\forall \; s \in [n], \;|f_{V^{(1)}}(x_s)| \le \eta \lv x_{L,s}^{V^{(1)}}\rv^2 \; \big| \; x_{L,s}^{V^{(1)}} \right] \ge 1- 2n\exp\left(-c_1 \eta^2 \right).
\end{align*}
Setting $\eta = c_2\sqrt{\log(n/\delta)}$ where $c_2$ is a large enough absolute constant we get that
\begin{align}  \label{e:conditional_good_event}
    \Pr\left[\forall \; s \in [n], \;|f_{V^{(1)}}(x_s)| \le c_2\sqrt{\log(n/\delta)} \lv x_{L,s}^{V^{(1)}}\rv^2 \; \big| \;x_{L,s}^{V^{(1)}} \right] \ge 1-\frac{7\delta}{8}.
\end{align}
Taking union bound over the events in \eqref{e:x_L_norm_large} and \eqref{e:conditional_good_event} we find that
\begin{align*} 
        \Pr\left[\forall \; s \in [n],\; |f_{V^{(1)}}(x_s)| \le c_3 \sqrt{\log(n/\delta)}  \right] \ge 1-\delta
\end{align*}
which completes the proof.
\end{proof}

Lastly we prove a lemma that bounds the norm of the initial weight matrix with high probability.

\begin{lemma}
\label{l:bound_on_norm_of_matrix}For any $\delta >0$, suppose that $p \ge \poly\left(L,\log\left(\frac{n}{\delta}\right)\right)$ for a large enough polynomial, then with probability at least $1-\delta$ over the randomness in $V^{(1)}$ $$\lv V^{(1)}\rv \le \sqrt{5pL}.$$
\end{lemma}
\begin{proof}By definition
\begin{align*}
    \lv V^{(1)}\rv^2 = \sum_{\ell \in [L+1]} \lv V_{\ell}^{(1)} \rv^2.
\end{align*}
When $\ell \in [L]$, the matrix $V^{(1)}_{\ell}$ has its entries drawn from $\cN(0,\frac{2}{p})$. Therefore by applying Theorem~\ref{thm:gaussian_concentration} we find that for any fixed $\ell \in [L]$, 
\begin{align*}
    \Pr\left[\lv V_{\ell}^{(1)}\rv^2 \le p^2\times\frac{2}{p}\times\frac{5}{4} = \frac{5p}{2}  \right] \le \exp\left(-\Omega(p)\right).
\end{align*}
While when $\ell = L+1$, the $p$-dimensional vector $V_{L+1}^{(1)}$ has its entries drawn from $\cN(0,1)$. Hence, again applying Theorem~\ref{thm:gaussian_concentration} we get
\begin{align*}
     \Pr\left[\lv V_{L+1}^{(1)}\rv^2 \le p\times1\times\frac{5}{4} = \frac{5p}{4}  \right] \le \exp\left(-\Omega(p)\right).
\end{align*}
Taking a union bound over all $L+1$ layers we find that
\begin{align*}
     \Pr\left[\forall \; \ell \in [L+1]: \lv V_{\ell}^{(1)}\rv^2 \le \frac{5p}{2}  \right] \le (L+1)\exp\left(-\Omega(p)\right) \le 1-\delta
\end{align*}
where the last inequality follows since $p \ge \poly\left(L,\log\left(\frac{n}{\delta}\right)\right)$. Therefore,
\begin{align*}
     \lv V^{(1)}\rv^2 \le (L+1)\times \frac{5p}{2} \le 5pL
\end{align*}
with probability at least $1-\delta$. Taking square roots establishes the claim.
\end{proof}

\subsection{Useful Properties in a Neighborhood Around the Initialization} \label{ss:ntk_lemmas_around_init}
In the next two lemmas we shall assume that the ``good event'' described in Lemma~\ref{l:conc_main_lemma} holds. We shall show that when the initial weight matrices satisfy those properties, we can also extend some of these properties to matrices in a neighborhood around the initial parameters.

\begin{lemma}
\label{l:ntk_product_matrices_in_a_ball} Let the event in Lemma~\ref{l:conc_main_lemma} hold and suppose that the conditions on $h$, $p$ and $\tau$ described in that lemma hold. Let $\tV$ be 
weights
such that $\lv \tV_{\ell} - 
V_{\ell}^{(1)}
\rv_{op}\le \tau$ for all $\ell \in [L]$. For all $\ell \in [L]$ and $s\in [n]$, let $\widetilde{\Sigma}_{\ell,s}$ be diagonal matrices such that $\lv \tSigma_{\ell,s} - 
\Sigma^{V^{(1)}}_{\ell,s} 
\rv_{0}\le k$, and $(\widetilde{\Sigma}_{\ell,s})_{jj} \in [-3,3]$ for all $j \in [p]$. 
There is an absolute constant
$C'$ such that, for all small
enough $c > 0$, if $\tau \le \sqrt{\frac{k\log(p)}{p}}\le \frac{c}{L^3}$ then,
for all $1 \leq \ell_1 \leq \ell_2 \leq L$,
\begin{align*}
    \left\lv \prod_{j=\ell_1}^{\ell_2} \tV_{j}^{\top}\tSigma_{j} \right\rv_{op} \le C' L^2.
\end{align*}
\end{lemma}
\begin{proof}Fix an arbitrary sample index $s$.
To ease notation let us refer to $V^{(1)}$ as $V$, $\Sigma_{\ell,s}^{V^{(1)}}$ as $\Sigma_{\ell}$, and $\tSigma_{\ell,s}$ as $\tSigma_{\ell}$. Note that for any $j \in [L]$
\begin{align}
    \tV_{j}^{\top}\tSigma_{j} = V_{j}^{\top}\Sigma_{j} + \underbrace{V_{j}^{\top}\left(\tSigma_{j} -\Sigma_{j}  \right)}_{=:\Gamma_{j}}+\underbrace{(\tV_{j}- V_j)^{\top} \tSigma_{j}}_{=:\Delta_{j}}. \label{e:def_flip_matrices}
\end{align}
Let us refer to $\Gamma_j$ and $\Delta_j$ as ``flip matrices''. Then, if we define the set $\cA_j = \{ V_{j}^{\top}\Sigma_{j}, \Gamma_{j}, \Delta_{j}$\}, expanding the product into a sum of terms yields
\begin{align}\label{e:decomposition_into_sum_of_products}
     \prod_{j=\ell_1}^{\ell_2}\tV_{j}^{\top}\tSigma_{j} =  \prod_{j=\ell_1}^{\ell_2}\left( V_{j}^{\top}\Sigma_{j,s} + \Gamma_{j,s}+\Delta_{j,s}\right) 
&= \sum_{A_{\ell_1} \in \cA_{\ell_1},\ldots,
        A_{\ell_2} \in \cA_{\ell_2}}
        \prod_{j = \ell_1}^ {\ell_2}A_{j}.
\end{align}

Each term in the sum on the RHS of \eqref{e:decomposition_into_sum_of_products}
is a product of $\ell_2 - \ell_1 + 1$ matrices ($A_j$),
some of which are flip matrices.
We will bound the operator norm of the sum
by bounding the operator norms of each of the
terms, and applying the triangle inequality.
To bound the operator norms of the terms, we
will decompose the terms into products of
subsequences of matrices, and bound the operator
norms of the subsequences.  
The subsequences will have at most two flip matrices, and will be determined by the positions of those flip matrices. 
One term in the sum has no flip matrices---it will have
a single subsequence that is the entire term.
Some terms have exactly one flip matrix. 
Those terms will be broken into two subsequences,
one that ends at the flip matrix, and the 
other consisting of the rest of the term.
The other terms in the sum have at least two flips.
Each such term can be broken down as follows:
\begin{itemize}
\item one or more subsequences with at exactly two
flip matrices ending in a flip matrix, 
\item possibly a subsequence with one
flip matrix, ending with the flip matrix, and
\item a (possibly empty) subsequence with no flip matrices.
\end{itemize}

In the calculations that follow the indices $q_1$, $q_2$ and $q_3$ satisfy: $1\le \ell_1\le q_1 \le q_2\le q_3 \le \ell_2 \le L$. Let $C>1$ be a large enough positive constant such that all the upper bounds in Lemma~\ref{l:conc_main_lemma} hold with this constant. 


\textit{Subsequences with no flip matrices:} First, 
subsequences for which
which $A_j = V_{j}^{\top}\Sigma_{j}$
for all $j$ can be bounded by Part~(c) of Lemma~\ref{l:conc_main_lemma}:
\begin{align*}\label{e:no_flips} \numberthis
\left\lv \prod_{j=q_1}^{q_2}V_{j}^{\top}\Sigma_{j}\right\rv_{op} = \left\lv \prod_{j=q_2}^{q_1}\Sigma_{j}V_{j}\right\rv_{op} 
   \leq CL.
\end{align*}

\textit{Subsequences with one flip matrix:} There will be two types of sub-sequences with just one flip matrix. First, let us consider the following type of subsequence:
\begin{align*}
    \left\lv \left( \prod_{j=q_1}^{q_2-1} V_{j}^{\top}\Sigma_{j} \right)\Delta_{q_2}\right\rv_{op} = \left\lv \Delta_{q_2}^{\top}\left( \prod_{j=q_2-1}^{q_1} \Sigma_{j}V_{j} \right)\right\rv_{op} &\le \left\lv \Delta_{q_2}^{\top}\right\rv_{op}\left\lv \prod_{j=q_2-1}^{q_1} \Sigma_{j}V_{j} \right\rv_{op}\\
    &\overset{(i)}{\le} CL \left\lv \Delta_{q_2}^{\top}\right\rv_{op} \\
    &= CL \left\lv \tSigma_{q_2}(\tV_{q_2} -V_{q_2} )\right\rv_{op}\\
    &\overset{(ii)}{\le} CL \left\lv \tV_{q_2} -V_{q_2} \right\rv_{op}\\ &\overset{(iii)}{\le} C\tau L\overset{(iv)}{\le} \frac{cC}{L^2}, \numberthis \label{e:delta_single_flip}
\end{align*}
where $(i)$ follows by again invoking Part~(c) of Lemma~\ref{l:conc_main_lemma}, $(ii)$ follows since by assumption the diagonal matrix $\tSigma_{q_2}$ has its entries bounded between $[-3,3]$, $(iii)$ follows since by assumption $\left\lv \tV_{q_2} -V_{q_2} \right\rv_{op} \le \tau$ and $(iv)$ follows since by assumption $\tau = c/L^3$.

Next, let us consider the second type of subsequence with just one flip matrix:
\begin{align*}
    \left\lv \left( \prod_{j=q_1}^{q_2-1} V_{j}^{\top}\Sigma_{j} \right)\Gamma_{q_2}\right\rv_{op}  &= \left\lv \left( \prod_{j=q_1}^{q_2-1} V_{j}^{\top}\Sigma_{j} \right)V_{q_2}^{\top} \left( \tSigma_{q_2}-\Sigma_{q_2}\right)\right\rv_{op}\\
    & = \sup_{a: \lv a \rv=1} \left\lv \left( \prod_{j=q_1}^{q_2-1} V_{j}^{\top}\Sigma_{j} \right)V_{q_2}^{\top} \left( \tSigma_{q_2}-\Sigma_{q_2}\right)a\right\rv\\
    & = \sup_{a: \lv a \rv=1} \left\lv a^{\top} \left( \tSigma_{q_2}-\Sigma_{q_2}\right)V_{q_2}  \prod_{j=q_2-1}^{q_1} \Sigma_{j}V_{j} \right\rv.
\end{align*}
For each $a$ let's define $b = \left( \tSigma_{q_2}-\Sigma_{q_2}\right) a$.
Since $\lv \tSigma_{q_2}-\Sigma_{q_2} \rv_0
\leq k$, therefore $b$ is $k$-sparse.
Also since the diagonal matrix $\tSigma_{q_2}$ has entries in $[-3,3]$ and $\Sigma_{q_2}$ has entries in $[0,1]$, therefore the entries of $\tSigma_{q_2}-\Sigma_{q_2}$ lie in $[-4,4]$. This implies that $\lv b \rv \leq 4\lv a \rv \leq 4$.
Applying Part~(e) of Lemma~\ref{l:conc_main_lemma}, we have
\begin{align*}
    \left\lv \left( \prod_{j=q_1}^{q_2-1} V_{j}^{\top}\Sigma_{j} \right)\Gamma_{q_2}\right\rv_{op} 
    & \leq C \lv b\rv = 4C. \numberthis \label{e:gamma_single_flip}
\end{align*}

\textit{Subsequences with two flip matrices:} Now we continue to subsequences with two flip matrices. There shall be four types of such subsequences. We begin by consider subsequences of the type
\begin{align*}
    &\left\lv \left( \prod_{j=q_1}^{q_2-1} V_{j}^{\top}\Sigma_{j} \right)  \Delta_{q_2}\left( \prod_{j=q_2+1}^{q_3-1} V_{j}^{\top}\Sigma_{j} \right)\Delta_{q_3}\right\rv_{op} \\& \qquad \le  \left\lv \left( \prod_{j=q_1+1}^{q_2-1} V_{j}^{\top}\Sigma_{j} \right)\Delta_{q_2} \right\rv_{op} \left\lv \left( \prod_{j=q_2+1}^{q_3-1} V_{j}^{\top}\Sigma_{j} \right)\Delta_{q_3}\right\rv_{op}
    \numberthis \label{e:delta_double_flip}   \overset{(i)}{\le} (C\tau L)^2 \overset{(ii)}{\le} \frac{cC}{L^2}\cdot C\tau L, 
\end{align*}
where $(i)$ follows by \eqref{e:delta_single_flip} and $(ii)$ follows since $\tau \leq c/L^3$. 

Next, we bound the operator norm of a subsequence of the type
\begin{align*} 
     & \left\lv  \left( \prod_{j=q_1}^{q_2-1} V_{j}^{\top}\Sigma_{j} \right) \Delta_{q_2}\left( \prod_{j=q_2+1}^{q_3-1} V_{j}^{\top}\Sigma_{j} \right)\Gamma_{q_3} \right\rv_{op} \\ &\qquad \le \left\lv \left( \prod_{j=q_1}^{q_2-1} V_{j}^{\top}\Sigma_{j} \right) \Delta_{q_2}\right\rv_{op} \left\lv \left( \prod_{j=q_2+1}^{q_3-1} V_{j}^{\top}\Sigma_{j} \right)\Gamma_{q_3}\right\rv_{op}
   \\& \qquad \overset{(i)}{\le} 4C\cdot C\tau L \\
   &\qquad = 4C^2\tau L, \label{e:delta_gamma_double_flip} \numberthis 
\end{align*}
where $(i)$ follows by invoking inequalities \eqref{e:delta_single_flip} and \eqref{e:gamma_single_flip}.

We continue to bound the operator norm of subsequences
\begin{align*}
    \left\lv \left( \prod_{j=q_1}^{q_2-1} V_{j}^{\top}\Sigma_{j} \right) \Gamma_{q_2}\left( \prod_{j=q_2+1}^{q_3-1} V_{j}^{\top}\Sigma_{j} \right)\Delta_{q_3} \right\rv_{op}
    &\le  \left\lv \left( \prod_{j=q_1}^{q_2-1} V_{j}^{\top}\Sigma_{j} \right)\Gamma_{q_2}\right\rv_{op}\left\lv \left( \prod_{j=q_2}^{q_3-1} V_{j}^{\top}\Sigma_{j} \right) \Delta_{q_3}\right\rv_{op}
   \\& \overset{(i)}{\le} 4C^2\tau L\label{e:gamma_delta_doubleflip}\numberthis 
\end{align*}
where $(i)$ follows again invoking inequalities \eqref{e:delta_single_flip} and \eqref{e:gamma_single_flip}.

Finally we bound the operator norm of subsequences of the type
\begin{align*}
     &\left\lv \left( \prod_{j=q_1}^{q_2-1} V_{j}^{\top}\Sigma_{j} \right)\Gamma_{q_2}\left( \prod_{j=q_2+1}^{q_3-1} V_{j}^{\top}\Sigma_{j} \right)\Gamma_{q_3} \right\rv_{op}\\
     &\qquad  \le \left\lv \prod_{j=q_1}^{q_2-1} V_{j}^{\top}\Sigma_{j} \right\rv_{op}\left\lv V_{q_2}^{\top} \left( \tSigma_{q_2}-\Sigma_{q_2}\right)\left( \prod_{j=q_2+1}^{q_3-1} V_{j}^{\top}\Sigma_{j} \right)V_{q_3}^{\top} \left( \tSigma_{q_3}-\Sigma_{q_3}\right)\right\rv_{op} \\
     &\qquad  \overset{(i)}{\le} CL \sup_{a:\lv a \rv =1}\left\lv V_{q_2}^{\top} \left( \tSigma_{q_2}-\Sigma_{q_2}\right)\left( \prod_{j=q_2+1}^{q_3-1} V_{j}^{\top}\Sigma_{j} \right)V_{q_3}^{\top} \left( \tSigma_{q_3}-\Sigma_{q_3}\right)a   \right\rv\\
     &\qquad  = CL\sup_{a:\lv a \rv =1}\sup_{b:\lv b \rv =1}\left| b^{\top} V_{q_2}^{\top} \left( \tSigma_{q_2}-\Sigma_{q_2}\right)\left( \prod_{j=q_2+1}^{q_3-1} V_{j}^{\top}\Sigma_{j} \right)V_{q_3}^{\top} \left( \tSigma_{q_3}-\Sigma_{q_3}\right)a   \right|\\
     & \qquad = CL\sup_{a:\lv a \rv =1}\sup_{b:\lv b \rv =1}\left| a^{\top}\left( \tSigma_{q_3}-\Sigma_{q_3}\right) V_{q_3} \left( \prod_{j=q_3-1}^{q_2+1} \Sigma_{j}V_{j} \right) \left( \tSigma_{q_2}-\Sigma_{q_2}\right)V_{q_2}b   \right| \\
     & \qquad = CL\sup_{a:\lv a \rv =1}\sup_{b:\lv b \rv =1}\left| a^{\top}\left( \tSigma_{q_3}-\Sigma_{q_3}\right) V_{q_3} \left( \prod_{j=q_3-1}^{q_2+1} \Sigma_{j}V_{j} \right) \left( \tSigma_{q_2}-\Sigma_{q_2}\right)\frac{V_{q_2}b}{\lv V_{q_2}b \rv}   \right| \lv V_{q_2}b\rv \\
     &\qquad  \overset{(ii)}{\le} 4C^2L\sqrt{\frac{k\log{p}}{p}} \sup_{b: \lv b \rv = 1 }\lv V_{q_2}b\rv =  4C^2L\sqrt{\frac{k\log{p}}{p}}  \lv V_{q_2}\rv_{op} \overset{(iii)}{\le} 4C^3L \sqrt{\frac{k\log{p}}{p}} \label{e:gamma_gamma_doubleflip} \numberthis
\end{align*}
where $(i)$ follows by invoking Part~(c) of Lemma~\ref{l:conc_main_lemma}. Inequality~$(ii)$ follows since the vectors $a^{\top}\left( \tSigma_{q_2}-\Sigma_{q_2}\right)$ and $ \left( \tSigma_{q_1}-\Sigma_{q_1}\right)\frac{V_{q_1}b}{\lv V_{q_1}b \rv}$ are $k$-sparse and both have norm less than or equal to 4, thus we can apply Part~(f) of Lemma~\ref{l:conc_main_lemma}, and $(iii)$ follows by applying Part~(b) of Lemma~\ref{l:conc_main_lemma}.

As stated above we can decompose each product in \eqref{e:decomposition_into_sum_of_products} that has
at least two flips into subsequences that
end in a flip and have exactly two flips,
and subsequences that have at most one
flip. The subsequences that have at most one
flip have operator norm at most $4CL$ (by inequalities~\eqref{e:no_flips}-\eqref{e:gamma_single_flip}).

The above logic \eqref{e:delta_double_flip}-\eqref{e:gamma_delta_doubleflip} implies that
subsequences with exactly two flips
that have at least one $\Delta$ flip
have operator norm at most
$4C^2\tau L \le 4C^3\sqrt{\frac{k\log{p}}{p}}L $ (since $\tau \le \sqrt{\frac{k\log{p}}{p}}$ by assumption and $C>1$). Subsequences with
two $\Gamma$ flips have operator norm
at most $4C^3\sqrt{\frac{k\log{p}}{p}}L $.  Define $\psi := 4C^3\sqrt{\frac{k\log{p}}{p}}L  $.
So, if a sequence has $r$ flip matrices then its operator norm is bounded by
\begin{align*}
 \psi^{\floor{r/2}} \times 4CL.
\end{align*}

So, putting it together, by recalling the decomposition in \eqref{e:decomposition_into_sum_of_products} we have
\begin{align*}
    \left\lv \prod_{j=\ell_1}^{\ell_2}\tV_{j}^{\top}\tSigma_{j} \right\rv_{op}
&= \left\lv \sum_{A_{\ell_1} \in \cA_{\ell_1},\ldots,
        A_{\ell_2} \in \cA_{\ell_2}}
        \prod_{j = \ell_1}^ {\ell_2}A_{j} \right\rv_{op}\\ & \overset{(i)}{\le} (1 + 2L) \times 4C L
 +  \sum_{A_{\ell_1} \in \cA_{\ell_1},\ldots,
        A_{\ell_2} \in \cA_{\ell_2}: \mbox{$\geq 2$ flips}}\left\lv
        \prod_{j = \ell_1}^ {\ell_2}A_{j} \right\rv_{op} \\
& \overset{(ii)}{\le} 12C L^2
 + \sum_{r=2}^L 
     \binom{L}{r} 2^r
     \cdot 4CL \psi^{\floor{r/2}}  \\
& \le 12C L^2
 + 8C L\sum_{r=2}^L 
     \binom{L}{r} 
     (4 \psi )^{\floor{r/2}} \\
& \leq
12C L^2
 + 8C L\left[\sum_{r \in \{2,\ldots,L \}, r \text{ even}}
     \left(\binom{L}{r} + \binom{L}{r+1}\right)
     (\sqrt{4 \psi} )^{r} \right]\\
     & = 12C L^2
 + 8C L\left[\sum_{r \in \{2,\ldots,L \}, r \text{ even}}
     \binom{L}{r}\left(1 + \frac{\binom{L}{r+1}}{\binom{L}{r}}\right)
     (\sqrt{4 \psi} )^{r} \right]\\
     & \le 12C L^2
 + 16C L^2\left[\sum_{r \in \{2,\ldots,L \}, r \text{ even}}
     \binom{L}{r}
     (\sqrt{4 \psi} )^{r} \right]\\
     & \le 12C L^2
 + 16C L^2\sum_{r =0}^L
     \binom{L}{r}
     (\sqrt{4 \psi} )^{r} \\
     & \overset{(iii)}{\le}  12CL^2
 + 16CL^2 \left(1+\sqrt{4\psi}\right)^L \\
 & \overset{(iv)}{=} 12C L^2
 + 16CL^2 \left(1+4C^{3/2}\sqrt{L}\left(\frac{k\log(p)}{p} \right)^{1/4}\right)^L \\
 & \overset{(v)}{\le}12C L^2
 + 16CL \left(1+\frac{4C^{3/2}\sqrt{c}}{L}\right)^L  \overset{(vi)}{\le}30CL^2
\end{align*}
where $(i)$ follows since the number of terms with at most one flip matrix is $(1+2L)$ and the operator norm of each such term is upper bounded by $4CL$ by inequalities~\eqref{e:no_flips}-\eqref{e:gamma_single_flip}. Inequality~$(ii)$ is because the number of terms with $r$ flip matrices is $\binom{L}{r}2^r$, $(iii)$ is by the Binomial theorem, $(iv)$ is by our definition of $\psi$. Inequality~$(v)$ follows since by assumption $\sqrt{\frac{k \log(p)}{p}}\le \frac{c}{L^{3}}$, and $(vi)$ follows since $c$ is small enough and because there exists positive constants $c_1$ and $c_2(c_1)$ such that, for any $0\le z < \frac{c_1}{L}$, $(1+z)^L \le 1+c_2Lz$. This completes our proof.
\end{proof}

The following lemma bounds the difference between post-activation features at the $\ell$th layer when the weight matrix is perturbed from its initial value. 

\begin{lemma} \label{l:ntk_x_perturbation_in_a_ball}
Let the event in Lemma~\ref{l:conc_main_lemma} hold and suppose that the conditions on $h$, $p$ and $\tau$ described in that lemma hold with the additional assumptions that $\tau \le\frac{c}{L^{12}\log^{\frac{3}{2}}(p)}$, where $c$ is a small enough constant, and $h \le \frac{\tau}{\sqrt{p}}$. Let $V^{(1)}$ be the initial 
weights 
and $\tV$, $\hV$ 
be
such that $\lv \tV_{\ell}-V^{(1)}_{\ell}\rv_{op},\lv \hV_{\ell}-V^{(1)}_{\ell}\rv_{op}\le \tau$ for all $\ell \in [L]$.  Then 
\begin{enumerate}
\item $\lv \Sigma_{\ell,s}^{\tV}-\Sigma_{\ell,s}^{\hV}\rv_0 \le O(pL^2 \tau^{2/3})$;
    \item $\lv x_{\ell,s}^{\tV}-x_{\ell,s}^{\hV}\rv \le O(L^3 \tau)$;
\end{enumerate}
for all $\ell \in [L]$ and all $s \in [n]$.
\end{lemma}
\begin{proof}
Fix the sample $s$
In this proof, we will
refer to $\Sigma_{\ell,s}^{V^{(1)}}, \Sigma_{\ell,s}^{\tV}, \Sigma_{\ell,s}^{\hV}, x_{\ell,s}^{V^{(1)}}, x_{\ell,s}^{\tV}$ and $x_{\ell,s}^{\hV}$ as $\Sigma_{\ell}$, $\tSigma_{\ell}$, $\hSigma_{\ell}$, $x_{\ell}$, $\tx_{\ell}$ and $\hx_{\ell}$ respectively.

Before the first layer (at layer $0$) define $\Sigma_{0}=\tSigma_{0}=\hSigma_{0} = I$ and recall that by definition for any sample $s \in [n]$, $x_{0,s}^{V^{(1)}} = x^{\tV}_{0,s} =x^{\hV}_{0,s} = x_s$.

For constants $c_1, c_2$ to be determined later, we will prove using induction that, for all $\ell \in [L]$,
\begin{enumerate}
\item $\lv \Sigma_{\ell}-\tSigma_{\ell}\rv_0$, $\lv \Sigma_{\ell}-\hSigma_{\ell}\rv_0 \le c_1  p L^{2} \tau^{2/3}$,
    \item $\lv \tSigma_{\ell}-\hSigma_{\ell}\rv_0\le 2c_1  p L^{2} \tau^{2/3}$, and
    \item $\lv x_{\ell}-\tx_{\ell}\rv,\lv x_{\ell}-\hx_{\ell}\rv,\lv \tx_{\ell}-\hx_{\ell}\rv
    \le c_2  L^3 \tau$.
\end{enumerate}

The base case, where $\ell=0$, is trivially true since $x_0 = \hx_0 = \tx_0$  and $\Sigma_0 = \tSigma_0 = \hSigma_0$.

Now let us assume that the inductive hypothesis holds for all layers $r = 1,\ldots,\ell-1$. We shall prove that the inductive hypothesis holds at layer $\ell$ in two steps. 

\paragraph{Step~1:}  By the triangle inequality
\begin{align}
    \lv \tSigma_{\ell} - \hSigma_{\ell} \rv_{0} \le \lv \Sigma_{\ell} - \tSigma_{\ell} \rv_{0}+\lv \Sigma_{\ell} - \hSigma_{\ell} \rv_{0}. \label{e:ntk_sigma_sparsity_decomposition}
\end{align}
Note that showing that $\lv \Sigma_{\ell} - \tSigma_{\ell} \rv_{0}$ and $\lv \Sigma_{\ell} - \hSigma_{\ell} \rv_{0}$ are at most $c_1pL^2 \tau^{2/3}$ also proves the claim that $ \lv \tSigma_{\ell} - \hSigma_{\ell} \rv_{0} \le 2c_1 pL^2 \tau^{2/3}$.

We begin by bounding $\lv \Sigma_{\ell} - \tSigma_{\ell} \rv_{0}$. Recall that by definition the diagonal matrix $(\Sigma_{\ell})_{jj} =(\phi'(V_{\ell}x_{\ell-1}))_j$. So to bound the difference between $\Sigma_{\ell} - \tSigma_{\ell}$ we characterize the difference between \begin{align*}
    V_{\ell}x_{\ell-1} - \tV_{\ell}\tx_{\ell-1} = (V_{\ell}-\tV_{\ell})x_{\ell-1} + \tV_{\ell}(x_{\ell-1}-\tx_{\ell-1}).
\end{align*}
We know that, by assumption, $\lv \tV_{\ell}- V_{\ell}\rv_{op}\le \tau$, and that $\lv x_{\ell-1}\rv\le 2$ by Part~(a) of Lemma~\ref{l:conc_main_lemma}, and $\lv \tx_{\ell-1}-x_{\ell-1}\rv\le c_2L^3\tau $ by the inductive hypothesis. Therefore,
\begin{align*}
    \lv V_{\ell}x_{\ell-1} - \tV_{\ell}\tx_{\ell-1} \rv&\le \lv(\tV_{\ell}-V_{\ell})x_{\ell-1}\rv + \lv\tV_{\ell}(\tx_{\ell-1}-x_{\ell-1}) \rv \\
    &\le \lv \tV_{\ell}-V_{\ell}\rv_{op} \lv x_{\ell-1}\rv + \lv\tV_{\ell}\rv_{op}\lv\tx_{\ell-1}-x_{\ell-1} \rv  \\
    & \le 2\tau + c_2L^3\tau \lv\tV_{\ell}\rv_{op} \\
    & \le 2\tau + c_2 L^3\tau \left(\lv\tV_{\ell}-V_{\ell}\rv_{op} + \lv V_{\ell}\rv_{op}\right)\\
    & \overset{(i)}{\le} 2\tau + c_2L^3\tau \left(\tau + c_3 \right) \overset{(ii)}{\leq} \tau \left(c_2c_4L^3+2\right),
\end{align*}
where $(i)$ follows since $\lv V_{\ell}\rv_{op}\le c_3$ by Part~(b) of Lemma~\ref{l:conc_main_lemma} and $(ii)$ follows since $\tau $ is smaller than a constant by assumption.

Let $\beta=\frac{c_5L^{2}\tau^{\frac{2}{3}}}{\sqrt{p}}>2h>0$. The reason for this particular choice of the value of $\beta$ shall become clear shortly, and $h \le \beta/2$ since $h\le \frac{\tau}{\sqrt{p}}$ by assumption. Define the set
\begin{align*}
    \cS_{\ell}(\beta) := \{j \in [p]: |V_{\ell,j}x_{\ell-1}|\le \beta \}
\end{align*}
where $V_{\ell,j}$ refers to the $j$th row of $V_{\ell}$. Also define 
\begin{align*}
    s_{\ell}^{(1)}(\beta) &:= \left|\{j\in  \cS_{\ell}(\beta): (\Sigma_{\ell})_{jj} \neq (\tSigma_{\ell})_{jj}  \}\right| \qquad \text{ and}\\
    s_{\ell}^{(2)}(\beta) &:= \left|\{j\in  \cS_{\ell}^{c}(\beta): (\Sigma_{\ell})_{jj} \neq (\tSigma_{\ell})_{jj} \}\right|.
\end{align*}
Clearly we must have that
\begin{align*}
    \lv \Sigma_{\ell} - \tSigma_{\ell}\rv_{0} = s_{\ell}^{(1)}(\beta)+s_{\ell}^{(2)}(\beta).
\end{align*}
To bound $s_{\ell}^{(1)}(\beta)$ we note that $s_{\ell}^{(1)}(\beta) \le |\cS_{\ell}(\beta)|\le c_6 p^{3/2}\beta$ by Part~(h) of Lemma~\ref{l:conc_main_lemma}. We focus on $s^{(2)}_{\ell}(\beta)$. For a $j \in \cS_{\ell}^{c}(\beta)$ by the definition of the Huberized ReLU if $(\Sigma_{\ell})_{jj} \neq (\tSigma_{\ell})_{jj}$ then we must have that
\begin{align*}
    \left| \tV_{\ell,j}\tx_{\ell-1} - V_{\ell,j}x_{\ell-1}\right| \ge \beta-h.
\end{align*}
This further implies that 
\begin{align*}
    (\beta-h)^2 s_{\ell}^{(2)}(\beta)  \le \sum_{j \in \cS_{\ell}^c(\beta):(\Sigma_{\ell})_{jj} \neq (\tSigma_{\ell})_{jj}}   \left| \tV_{\ell,j}\tx_{\ell-1} - V_{\ell,j}x_{\ell-1}\right|^2 &\le \lv V_{\ell}x_{\ell-1} - \tV_{\ell}\tx_{\ell-1} \rv^2 \\&\le  \tau^2 \left(c_2c_4L^3+2\right)^2.
\end{align*}
Therefore, we find that
\begin{align*}
     \lv \Sigma_{\ell} - \tSigma_{\ell}\rv_{0} & = s_{\ell}^{(1)}(\beta)+s_{\ell}^{(2)}(\beta) \\
     &\le \frac{ \tau^2 \left(c_2 c_4L^3+2\right)^2}{(\beta-h)^2} + c_6 p^{3/2}\beta \\
     &\le \frac{ \left(c_2 c_7 L^3 \tau \right)^2}{\beta^2} + c_6 p^{3/2}\beta,
\end{align*}
if $c_2 \geq 1/c_4$, since $h \leq \beta/2$.
Balancing both of these terms on the RHS leads to the choice $\beta = \frac{c_5L^2\tau^{2/3}}{\sqrt{p}}$. This choice of $\beta$ shows that
\begin{align*}
    \lv \Sigma_{\ell} - \tSigma_{\ell}\rv_{0} &\le 2c_6p^{3/2}\beta = 2c_6 c_5 p L^2 \tau^{2/3} = c_1p L^2 \tau^{2/3} .
\end{align*}
Similarly we can also show that $\lv \Sigma_{\ell} - \hSigma_{\ell}\rv_{0}\le c_1pL^2\tau^{2/3}$. These two bounds along with \eqref{e:ntk_sigma_sparsity_decomposition} proves the first part of the inductive hypothesis. This combined with inequality~\eqref{e:ntk_sigma_sparsity_decomposition} also proves the second part of the inductive hypothesis.

\paragraph{Step 2:} Now, for the third part we want to show that $\lv \tx_{\ell,s}-\hx_{\ell,s}\rv$ remains bounded. We can also show that $\lv x_{\ell,s}-\tx_{\ell,s}\rv$ and $\lv x_{\ell,s}-\tx_{\ell,s}\rv$ remain bounded by mirroring the logic that follows. Define a diagonal matrix $\cSigma_{\ell}$,
whose diagonal entries are
\begin{align*}
    (\cSigma_{\ell})_{jj}:= (\hSigma_{\ell} - \tSigma_{\ell})_{jj}\left[\frac{\tV_{\ell,j}\tx_{\ell-1}}{\hV_{\ell,j}\hx_{\ell-1}-\tV_{\ell,j}\tx_{\ell-1}}\right], \qquad \text{ for all, } j\in [p].
\end{align*}
In the definition above we use the convention that $0/0 = 0$. We will show that for any $j \in [p]$
\begin{align*}
    |(\cSigma_{\ell})_{jj}| = \left|(\hSigma_{\ell} - \tSigma_{\ell})_{jj}\left[\frac{\tV_{\ell,j}\tx_{\ell-1}}{\hV_{\ell,j}\hx_{\ell-1}-\tV_{\ell,j}\tx_{\ell-1}}\right] \right| &\le 1.
\end{align*}
 Firstly observe that the matrices $\tSigma_{\ell}$ and $\hSigma_{\ell}$ have entries between $[0,1]$, therefore $\tSigma_{\ell}-\Sigma_{\ell}$ has entries between $-1$ and $1$. Also recall that by the definition of the Huberized ReLU,
 \begin{align*}
     (\hSigma_{\ell})_{jj} & = \begin{cases}
      1 & \text{if } \hV_{\ell,j}\hx_{\ell-1}>h,\\
       \frac{\hV_{\ell,j}\hx_{\ell-1}}{h} & \text{if } \hV_{\ell,j}\hx_{\ell-1}\in [0,h],\\
       0 & \text{if } \hV_{\ell,j}\hx_{\ell-1}<0.
     \end{cases}
 \end{align*}
 
 Now we will analyze a few cases and show that the absolute values of the entries of $\cSigma_{\ell}$ are smaller than $1$ in each case. 
 
 If the signs of $\tV_{\ell,j}\tx_{\ell-1}$ and $\hV_{\ell,j}\hx_{\ell-1}$ are opposite then we must have that 
\begin{align*}
   \left|(\hSigma_{\ell} - \tSigma_{\ell})_{jj} \frac{\tV_{\ell,j}\tx_{\ell-1}}{\hV_{\ell,j}\hx_{\ell-1}-\tV_{\ell,j}\tx_{\ell-1}}\right| \overset{(i)}{\le}\left| \frac{\tV_{\ell,j}\tx_{\ell-1}}{\hV_{\ell,j}\hx_{\ell-1}-\tV_{\ell,j}\tx_{\ell-1}}\right|=  \frac{|\tV_{\ell,j}\tx_{\ell-1}|}{|\hV_{\ell,j}\hx_{\ell-1}|+|\tV_{\ell,j}\tx_{\ell-1}|} \le 1,
\end{align*}
where $(i)$ follows since $|(\hSigma_{\ell} - \tSigma_{\ell})_{jj}| \le 1$.
If they have the same sign and are both negative then $(\tSigma_{\ell}-\Sigma_{\ell})_{jj}=0$ in this case. The same is true when they are both positive and are bigger than $h$. Therefore, we are only left with the case when both are positive and one of them is smaller than $h$. If $\tV_{\ell,j}\tx_{\ell-1}>h$ and $\hV_{\ell,j}\hx_{\ell-1} \in [0,h]$ we have that
\begin{align*}
   \left|(\hSigma_{\ell} - \tSigma_{\ell})_{jj} \frac{\tV_{\ell,j}\tx_{\ell-1}}{\hV_{\ell,j}\hx_{\ell-1}-\tV_{\ell,j}\tx_{\ell-1}}\right| & =  \left| \frac{\left(\frac{\hV_{\ell,j}\hx_{\ell-1}}{h} - 1\right) \tV_{\ell,j}\tx_{\ell-1}}{\hV_{\ell,j}\hx_{\ell-1}-\tV_{\ell,j}\tx_{\ell-1}}\right|  =  \left| \frac{\frac{\hV_{\ell,j}\hx_{\ell-1}}{h} - 1}{\frac{\hV_{\ell,j}\hx_{\ell-1}}{\tV_{\ell,j}\tx_{\ell-1}}-1}\right| \le 1
\end{align*}
where the last inequality follows since $\tV_{\ell,j}\tx_{\ell-1}>h$.
And finally in the case where $\hV_{\ell,j}\hx_{\ell-1}>h$ and $\tV_{\ell,j}\tx_{\ell-1} \in [0,h]$ we have that
\begin{align*}
   \left|(\hSigma_{\ell} - \tSigma_{\ell})_{jj} \frac{\tV_{\ell,j}\tx_{\ell-1}}{\hV_{\ell,j}\hx_{\ell-1}-\tV_{\ell,j}\tx_{\ell-1}}\right| & =  \left|\left(1-\frac{\tV_{\ell,j}\tx_{\ell-1}}{h} \right) \frac{\tV_{\ell,j}\tx_{\ell-1}}{\hV_{\ell,j}\hx_{\ell-1}-\tV_{\ell,j}\tx_{\ell-1}}\right|   =   \frac{\left(h-\tV_{\ell,j}\tx_{\ell-1} \right)\tV_{\ell,j}\tx_{\ell-1}}{h\left(\hV_{\ell,j}\hx_{\ell-1}-\tV_{\ell,j}\tx_{\ell-1}\right)} .
\end{align*}
To show that this term of the RHS above is smaller than $1$ it is sufficient to show that
\begin{align*}
    (h-\tV_{\ell,j}\tx_{\ell-1})\tV_{\ell,j}\tx_{\ell-1} \le h(\hV_{\ell,j}\hx_{\ell-1}-\tV_{\ell,j}\tx_{\ell-1}),
\end{align*}
in our case where $0\le \tV_{\ell,j}\tx_{\ell-1}\le h < \hV_{\ell,j}\hx_{\ell-1}$. Consider the change of variables $a = \tV_{\ell,j}\tx_{\ell-1}$ and $b = \hV_{\ell,j}\hx_{\ell-1}$, then it suffices to show that
\begin{align*}
    (h-a)a \le h(b-a) \Leftarrow 0 \le a^2-2ah+hb.
\end{align*}
The derivative of $a^2-2ah+hb$ with respect to $a$ is $2(a-h)$, which is non-positive when $a \le h$. Therefore the minimum of the quadratic when $a \in [0,h]$ is at $a = h$ and the minimum value is $h^2-2h^2+hb = hb-h^2 = h(b-h)>0$. This proves that $|(\cSigma)_{jj}|\le 1$ in this final case as well.

With this established we note that
\begin{align*}
   e_{\ell} := \hx_{\ell} - \tx_{\ell} 
    &=  \phi(\hV_{\ell}\hx_{\ell-1})-\phi(\tV_{\ell}\tx_{\ell-1})  \\
    &= \hSigma_{\ell}\hV_{\ell}\hx_{\ell-1} -\tSigma_{\ell}\tV_{\ell}\tx_{\ell-1}  +\underbrace{\phi(\hV_{\ell}\hx_{\ell-1})-\hSigma_{\ell}\hV_{\ell}\hx_{\ell-1}- \phi(\tV_{\ell}\tx_{\ell-1}) +\tSigma_{\ell}\tV_{\ell}\tx_{\ell-1}}_{=:\chi_\ell}\\
    &\overset{(i)}{=}  \left( \hSigma_{\ell}+\cSigma_{\ell}\right)\left(\hV_{\ell}\hx_{\ell-1} - \tV_{\ell}\tx_{\ell-1}\right)+\chi_{\ell} \\
    &= \underbrace{\left( \hSigma_{\ell}+\cSigma_{\ell}\right)\hV_{\ell}}_{=: A_{\ell}}\underbrace{(\hx_{\ell-1}-\tx_{\ell-1})}_{=: e_{\ell-1}}+\underbrace{\left( \hSigma_{\ell}+\cSigma_{\ell}\right)\left(  \tV_{\ell}-V_{\ell}\right)\tx_{\ell-1}+\chi_{\ell}}_{=:b_{\ell}}\\
    &= A_{\ell}e_{\ell-1}+b_{\ell}
    \numberthis \label{e:unrolling_the_error_at_layer_ell_in_terms_of_ell-1}
\end{align*}
where $(i)$ follows because by the definition of the matrix $\cSigma_{\ell}$ for each coordinate $j$ we have
\begin{align*}
     \left( \hSigma_{\ell}+\cSigma_{\ell}\right)_{jj}\left(\hV_{\ell,j}\hx_{\ell-1} - \tV_{\ell,j}\tx_{\ell-1}\right)& = \left( \hSigma_{\ell}\right)_{jj}\left(\hV_{\ell,j}\hx_{\ell-1} -\tV_{\ell,j}\tx_{\ell-1}\right)+ \left( \cSigma_{\ell}\right)_{jj}\left(\hV_{\ell,j}\hx_{\ell-1} -\tV_{\ell,j}\tx_{\ell-1}\right) \\
     & = \left( \hSigma_{\ell}\right)_{jj}\left(\hV_{\ell,j}\hx_{\ell-1} -\tV_{\ell,j}\tx_{\ell-1}\right) + (\hSigma_{\ell}-\tSigma_{\ell})_{jj} \tV_{\ell,j}\tx_{\ell-1}\\
     &= (\hSigma_{\ell})_{jj}\hV_{\ell,j}\hx_{\ell-1} -(\tSigma_{\ell})_{jj}\tV_{\ell,j}\tx_{\ell-1}.
\end{align*}
In equation~\eqref{e:unrolling_the_error_at_layer_ell_in_terms_of_ell-1} above we have expressed the difference between the post-activation features at layer $\ell$ in terms of the difference at layer $\ell-1$ plus some error terms. Repeating this $\ell-1$ more times yields 
\begin{align*}
    e_{\ell} & = A_{\ell}e_{\ell-1}+b_{\ell} = A_{\ell}(A_{\ell-1}e_{\ell-2}+b_{\ell-1})+b_{\ell}  = \prod_{j=\ell}^{1}A_j e_0+ \left(\sum_{r=1}^{\ell-1} \left[\prod_{j=\ell}^{r+1}A_{j}\right] b_{r}\right) +b_{\ell}.
\end{align*}
Since $e_0=\lv \hx_{0}-\tx_{0}\rv = 0$, by re-substituting the values of $A_{\ell}$ and $b_{\ell}$ we find that
\begin{align} \nonumber
     \hx_{\ell} - \tx_{\ell}  & = \sum_{r=1}^{\ell-1} \left[\prod_{j=\ell}^{r+1} \left( \hSigma_{j}+\cSigma_j\right)\hV_j\right](\hSigma_r+\cSigma_{r})(\hV_r-\tV_r) \tx_{r-1}  +\left( \hSigma_{\ell}+\cSigma_{\ell}\right)\left(  \tV_{\ell}-\hV_{\ell}\right)\tx_{\ell-1}  \\ & \qquad+\sum_{r=1}^{\ell-1} \left[\prod_{j=\ell}^{r+1}(\hSigma_{j}+\cSigma_j)\hV_j\right]\chi_{r}+\chi_{\ell}, \label{e:ntk_x_difference_decomposition}
\end{align}
and therefore by the triangle inequality
\begin{align*}
     &\lv \hx_{\ell} - \tx_{\ell}\rv \\ & \le \left\lv\sum_{r=1}^{\ell-1} \left[\prod_{j=\ell}^{r+1} \left( \hSigma_{j}+\cSigma_j\right)\hV_j\right](\hSigma_r+\cSigma_{r})(\hV_r-\tV_r) \tx_{r-1}\right\rv + \left\lv \left( \hSigma_{\ell}+\cSigma_{\ell}\right)\left(  \tV_{\ell}-\hV_{\ell}\right)\tx_{\ell-1} \right\rv  \\ & \qquad \qquad +\left\lv \sum_{r=1}^{\ell-1} \left[\prod_{j=\ell}^{r+1}(\hSigma_{j}+\cSigma_j)V_j\right]\chi_{r}\right\rv+\left\rv\chi_{\ell}\right\rv \\
     &\le  \sum_{r=1}^{\ell-1} \left\lv\prod_{j=\ell}^{r+1} \left( \hSigma_{j}+\cSigma_j\right)\hV_j\right\rv_{op}\left\lv\hSigma_r +\cSigma_{r}\right\rv_{op}\left\lv \hV_r-\tV_r\right\rv_{op} \lv\tx_{r-1}\rv + \left\lv  \hSigma_{\ell}+\cSigma_{\ell}\right\rv_{op}\left\lv  \tV_{\ell}-\hV_{\ell}\right\rv_{op}\lv\tx_{\ell-1} \rv \\ & \qquad \qquad +  \sum_{r=1}^{\ell-1} \left\lv\prod_{j=\ell}^{r+1}(\hSigma_{j}+\cSigma_j)V_j\right\rv_{op}\lv\chi_{r}\rv +\lv\chi_{\ell}\rv. \\
     \end{align*}
     Recall that the diagonal matrices $(\Sigma_{\ell}-\hSigma_{\ell}-\cSigma_{\ell})$ are $3c_1 pL^2\tau^{2/3}$ sparse by the inductive hypothesis. Also the matrices $(\Sigma_{\ell}-\hSigma_{\ell}-\cSigma_{\ell})$ have entries in $[-3,3]$. Therefore by applying Lemma~\ref{l:ntk_product_matrices_in_a_ball} (note that since $\tau \le \frac{c}{L^{12}\log^{\frac{3}{2}}(p)}$ therefore Lemma~\ref{l:ntk_product_matrices_in_a_ball} applies at this level of sparsity) we find that
     for a constant $c_8$ (that does
     not depend on $c_1$), we have
     \begin{align*}
     \lv \hx_{\ell} - \tx_{\ell}\rv
     &\le c_8 L^2\left[\sum_{r=1}^{\ell} \lv \hV_{r}-\tV_{r}\rv_{op} \lv \tx_{r-1}\rv+\sum_{r=1}^{\ell}\lv \chi_{r}\rv\right] \\
     &\overset{(i)}{\le} c_8 L^2\left[\sum_{r=1}^{\ell} \lv \hV_{r}-\tV_{r}\rv_{op} \lv \tx_{r-1}\rv+\ell h\sqrt{p}\right] \\
     &\le c_8 L^2\left[\sum_{r=1}^{\ell} \lv \hV_{r}-\tV_{r}\rv_{op} \left(\lv \tx_{r-1}-x_{r-1}\rv+\lv x_{r-1}\rv\right)+\ell h\sqrt{p}\right] \\
     &\overset{(ii)}{\le} c_8L^2\left[\sum_{r=1}^{\ell} \lv \hV_{r}-\tV_{r}\rv_{op} \left(c_1L^3 \tau+2\right)+\ell h\sqrt{p}\right] \\
     &\overset{(iii)}{\le} c_9L^2\left[\sum_{r=1}^{\ell} \lv \hV_{r}-\tV_{r}\rv_{op} +L \tau\right]\le 2c_9 L^3 \tau = c_2 L^3\tau, 
\end{align*}
 where inequality~$(i)$ follows since by definition of the Huberized ReLU for any $z \in \mathbb{R}$ we have that $\phi(z)\le \phi'(z)z \le \phi(z)+h/2$, therefore 
\begin{align}
    \lv\chi_{r}\rv_{\infty} = \left\lv \phi(V_{\ell}x_{\ell-1})-\Sigma_{\ell}V_{\ell}x_{\ell-1}- \phi(\tV_{\ell}\tx_{\ell-1}) +\tSigma_{\ell}\tV_{\ell}\tx_{\ell-1} \right\rv_{\infty} \le 2\cdot\frac{h}{2}=h \label{e:ntk_chi_huberized_relu_infinity_norm_bound}
\end{align}
which implies that $\lv\chi_r \rv\le h\sqrt{p}$. Next $(ii)$ follows by bound on $\lv \tx_{r-1}-x_{r-1}\rv$ due to the inductive hypothesis and because $\lv x_{r-1}\rv_2 \le 2$ by Part~(a) of Lemma~\ref{l:conc_main_lemma}. Finally $(iii)$ follows by assumption $\tau \le O(1/L^3)$ and $h<\frac{\tau}{\sqrt{p}}$. This establishes a bound on $\lv \hx_{\ell} - \tx_{\ell}\rv$. We can also mirror the logic to bound $\lv x_{\ell} - \tx_{\ell}\rv$ and $\lv x_{\ell} - \tx_{\ell}\rv$. This completes the induction and the proof of the lemma.
\end{proof}

\begin{lemma}
\label{l:ntk_diff_between_prod_of_matrices}Let the event in Lemma~\ref{l:conc_main_lemma} hold and suppose that the conditions on $h$, $p$ and $\tau$ described in that lemma hold with the additional assumptions that, 
for
a sufficient small constant $c> 0$, $\tau \leq \frac{c}{L^{12}\log^{\frac{3}{2}}(p)}$, and $h \le \frac{\tau}{\sqrt{p}}$. Let $V^{(1)}$ be the initial weight matrix and $\tV$, $\hV$ be weight matrices such that $\lv \tV_{\ell}-V^{(1)}_{\ell}\rv_{op},\lv \hV_{\ell}-V^{(1)}_{\ell}\rv_{op}\le \tau$ for all $\ell \in [L]$. Also let $\bSigma_{\ell,s}$ be $O(pL^2\tau^{2/3})$-sparse diagonal matrices with entries in $[-1,1]$ for all $\ell \in [L]$ and $s\in [n]$. Then 
\begin{align*}
    \left\lv \tV_{L+1} \prod_{r = L}^{\ell} \left(\Sigma_{r,s}^{\tV}+\bSigma_{r,s}\right)\tV_r- \hV_{L+1} \prod_{r = L}^{\ell} \Sigma_{r,s}^{\hV}\hV_{r} \right\rv_{op} \le O\left(\sqrt{p\log(p)}L^4\tau^{1/3}\right).
\end{align*}
for all $\ell \in [L]$ and all $s \in [n]$.
\end{lemma}
\begin{proof}
We want to bound the operator norm of 
\begin{align*}
    &\tV_{L+1} \prod_{r = L}^{\ell} \left(\Sigma_{r,s}^{\tV}+\bSigma_{r,s}\right)\tV_r- \hV_{L+1} \prod_{r = L}^{\ell} \Sigma_{r,s}^{\hV}\hV_{r}\\ & = \underbrace{\tV_{L+1} \prod_{r = L}^{\ell} \left(\Sigma_{r,s}^{\tV}+\bSigma_{r,s}\right)\tV_r - V^{(1)}_{L+1} \prod_{r = L}^{\ell} \Sigma_{r,s}^{V^{(1)}}V^{(1)}_r}_{=:\chi_1} + \underbrace{ V^{(1)}_{L+1} \prod_{r = L}^{\ell} \Sigma_{r,s}^{V^{(1)}}V^{(1)}_r - \hV_{L+1} \prod_{r = L}^{\ell} \hSigma_{r,s}^{V}\hV_{r}}_{=:\chi_2}. \label{e:ntk_difference_of_products_decomposition} \numberthis
\end{align*}
We shall instead bound the operator norm of $\chi_1$ and $\chi_2$. Let us proceed to bound the operator norm of $\chi_1$ (the bound on $\chi_2$ will hold using exactly the same logic). Now to ease notation let us fix a sample index $s \in [n]$ and drop it from all subscripts. Also to simplify notation let us refer to $\Sigma_{r,s}^{\tV}$ as $\tSigma_{r}$, $\Sigma_{r,s}^{\hV}$ as $\hSigma_{r}$, $\bSigma_{r,s}$ as $\bSigma_{r}$, and $\Sigma_{r,s}^{V^{(1)}}$ as $\Sigma_{r}$. We shall also refer to $V^{(1)}$ as simply $V$.

By assumption the diagonal matrix $\bSigma_{r}$ is $O(pL^2 \tau^{2/3})$-sparse with entries in $[-1,1]$. Also the matrix $\Sigma_{r}-\tSigma_{r}$ is $O(pL^2 \tau^{2/3})$-sparse by Lemma~\ref{l:ntk_x_perturbation_in_a_ball}. Therefore the matrix $\cSigma_{r}:=\tSigma_{r}+\bSigma_{r}-\Sigma_{r}$ is also $O(pL^2 \tau^{2/3})$-sparse and has entries in $[-2,2]$. Thus,
\begin{align*}
    \chi_1 &= \tV_{L+1} \prod_{r = L}^{\ell} \left(\tSigma_{r}+\bSigma_{r}\right)\tV_r - V_{L+1} \prod_{r = L}^{\ell} \Sigma_{r}V_r \\
    & = \tV_{L+1} \prod_{r = L}^{\ell} \left(\Sigma_{r}+\cSigma_r\right)\tV_r - V_{L+1} \prod_{r = L}^{\ell} \Sigma_{r}V_r\\
    & = \underbrace{(\tV_{L+1}-V_{L+1}) \prod_{r = L}^{\ell} \left(\Sigma_{r}+\cSigma_{r}\right)\tV_r}_{=:\spadesuit}\\ &\qquad + \underbrace{V_{L+1} \left(\prod_{r = L}^{\ell} \left(\Sigma_{r}+\cSigma_{r}\right)\tV_r -\prod_{r = L}^{\ell} \Sigma_{r}V_r \right)}_{=:\clubsuit}. \label{e:ntk_decomposition_into_clubsuit_spadesuit} \numberthis
\end{align*}
The operator norm of $\spadesuit$ is easy to bound by invoking Lemma~\ref{l:ntk_product_matrices_in_a_ball}
\begin{align*}
    \lv \spadesuit\rv_{op} &\le \lv \tV_{L+1}-V_{L+1} \rv_{op} \left\lv  \prod_{r = L}^{\ell} \left(\Sigma_{r}+\cSigma_{r}\right)\tV_r\right\rv_{op} \le O(\tau L^2). \numberthis \label{e:chi_1_bound_spadesuit}
\end{align*}
To bound the operator norm of $\clubsuit$ we will decompose the difference of the products of matrices terms into a sum. Each term in this sum corresponds to either a flip from $V_r$ to $\tilde{V}_r$ or from $\Sigma_{r}$ to $\Sigma_{r}+\cSigma_r$. That is,
\begin{align*}
    \clubsuit& = - \left(V_{L+1}\prod_{r = L}^{\ell} \Sigma_{r}V_r-V_{L+1}\prod_{r = L}^{\ell} \left(\Sigma_{r}+\cSigma_{r}\right)\tV_r \right)\\
   & =  -\left[\sum_{q=L}^{\ell} \underbrace{V_{L+1}\overbrace{\left(\prod_{r=L}^{q+1} \left(\Sigma_{r}V_r \right)\right)}^{=:\omega_{1,q}}\left(\cSigma_{q}\right) \left(\tV_{q}\overbrace{\prod_{r=q-1}^{\ell} \left(\Sigma_{r}+\cSigma_{r}\right)\tV_r}^{=:\omega_{2,q}}\right) }_{=:\vardiamond_{q}} \right]\\
   & \qquad -  \left[\sum_{q=L}^{\ell} \underbrace{V_{L+1}\overbrace{\left(\prod_{r=L}^{q+1} \left(\Sigma_{r}V_r \right)\right)}^{=:\omega_{3,q}}\Sigma_{q}\left(V_{q}-\tV_{q}\right) \overbrace{\prod_{r=q-1}^{\ell} \left(\Sigma_{r}+\cSigma_{r}\right)\tV_r}^{=:\omega_{4,q}} }_{=:\varheart_q} \right] \numberthis \label{e:ntk_clubsuit_decomposition}
\end{align*}
 where in the previous equality above, 
 the indices in the products ``count down'',
 so that cases in which $q=L$ include
 ``empty products'', and we adopt the
 convention that, in such cases,
 \begin{align*}
     \omega_{1,q} = \omega_{3,q} = I, 
 \end{align*}
 and when $q = \ell$
 \begin{align*}
     \omega_{2,q} = \omega_{4,q} = I.
 \end{align*}
 
 We begin by bounding the operator norm of $\vardiamond_q$ (for a $q$ that is not $\ell$ or $L$, the exact same bound follows in these boundary cases):
\begin{align*}
    \lv \vardiamond_q \rv_{op}&= \left\lv V_{L+1}\prod_{r=L}^{q-1} \left(\Sigma_{r}V_r \right)\left(\cSigma_{q}\right) \left(\tV_{q}\prod_{r=q+1}^{\ell} \left(\Sigma_{r}+\cSigma_{r,s}\right)\tV_r\right) \right\rv_{op} \\
    &\overset{(i)}{=} \left\lv V_{L+1}\prod_{r=L}^{q-1} \left(\Sigma_{r}V_r \right)\Sigma^{\mathsf{0/1}}_{q}\cSigma_{q} \Sigma^{\mathsf{0/1}}_{q} \tV_{\ell}\prod_{r=q+1}^{\ell} \left(\Sigma_{r}+\cSigma_{r}\right)\tV_r \right\rv_{op} \\
    &\le \left\lv V_{L+1}\prod_{r=L}^{q-1} \left(\Sigma_{r}V_r \right)\Sigma^{\mathsf{0/1}}_{q}\right\rv_{op} \lv\cSigma_{q}\rv_{op} \left\lv\Sigma^{\mathsf{0/1}}_{q} \tV_{\ell}\prod_{r=q+1}^{\ell} \left(\Sigma_{r}+\cSigma_{r}\right)\tV_r \right\rv_{op} \\
    &\overset{(ii)}{\le} 2\left\lv V_{L+1}\prod_{r=L}^{q-1} \left(\Sigma_{r}V_r \right)\Sigma^{\mathsf{0/1}}_{q}\right\rv_{op}  \left\lv\Sigma^{\mathsf{0/1}}_{q} \tV_{\ell}\prod_{r=q+1}^{\ell} \left(\Sigma_{r}+\cSigma_{r}\right)\tV_r \right\rv_{op} \\
    &\overset{(iii)}{\le} O\left(\sqrt{pL^2 \tau^{2/3}\log(p)}\right) \left\lv\Sigma^{\mathsf{0/1}}_{q}\right\rv_{op} \left\lv\tV_{\ell}\prod_{r=q+1}^{\ell} \left(\Sigma_{r}+\cSigma_{r}\right)\tV_r \right\rv_{op} \\
    &\overset{(iv)}{\le} O\left(\sqrt{pL^2 \tau^{2/3}\log(p)}\right) \times O(L^2) = O\left(\sqrt{p\log(p)}L^3\tau^{1/3} \right) \label{e:ntk_bound_on_diamond} \numberthis
\end{align*}
where in $(i)$ we define $\Sigma^{\mathsf{0/1}}_{q}$ to be a diagonal matrix with $(\Sigma^{\mathsf{0/1}}_{q})_{jj}:= \mathbb{I}\left[(\tSigma_{q,s})_{jj}\neq 0\right]$. Note that since $\cSigma_{q}$ is $O(pL^2 \tau^{2/3})$ sparse, therefore $\Sigma^{\mathsf{0/1}}_{q}$ is also $O(pL^2 \tau^{2/3})$ sparse. Inequality~$(ii)$ follows since the entries of $\cSigma_{q,s}$ lie between $[-2,2]$, $(iii)$ follows by applying Part~(g) of Lemma~\ref{l:conc_main_lemma}. Finally, $(iv)$ by applying Lemma~\ref{l:ntk_product_matrices_in_a_ball} since the matrix $\cSigma_{r}$ is $O(pL^2\tau^{2/3})$-sparse and has entries in $[-2,2]$. 

To control the operator norm of $\varheart_q$ (again for a $q \neq \ell$ or $L$, the exact same bound follows in these boundary cases):
\begin{align*}
\lv \varheart_q \rv_{op}& = \left\lv V_{L+1}\prod_{r=L}^{q-1} \left(\Sigma_{r}V_r \right)\Sigma_{q}\left(\tV_{q}-V_{q}\right) \left(\prod_{r=q+1}^{\ell} \left(\Sigma_{r}+\cSigma_{r}\right)\tV_r\right) \right\rv_{op} \\
&\le \left\lv V_{L+1}\prod_{r=L}^{q-1} \left(\Sigma_{r}V_r \right)\right\rv_{op} \lv \Sigma_{q}\rv_{op}\left\lv \tV_{q}-V_{q}\right\rv_{op} \left\lv\prod_{r=q+1}^{\ell} \left(\Sigma_{r}+\cSigma_{r}\right)\tV_r \right\rv_{op} \\
&\le 2\tau \left\lv V_{L+1}\prod_{r=L}^{q-1} \left(\Sigma_{r}V_r \right)\right\rv_{op}  \left\lv \prod_{r=q+1}^{\ell} \left(\Sigma_{r}+\cSigma_{r}\right)\tV_r \right\rv_{op} \\
&\overset{(i)}{\le} O(\tau L^2) \left\lv V_{L+1}\prod_{r=L}^{q-1} \left(\Sigma_{r}V_r \right)\right\rv_{op}   \overset{(ii)}{\le} O(\sqrt{p}\tau L^3) 
\end{align*}
where $(i)$ follows by applying Lemma~\ref{l:ntk_product_matrices_in_a_ball} and $(ii)$ follows by Part~(c) of Lemma~\ref{l:conc_main_lemma}.

With these bounds on $\vardiamond_q$ and $\varheart_q$ along with the decomposition in \eqref{e:ntk_clubsuit_decomposition} we find that
\begin{align*}
    \lv \clubsuit \rv_{op} &\le L \times\left(O\left(\sqrt{p\log(p)}L^3\tau^{1/3}\right)+ O(\sqrt{p}\tau L^3) \right) \le O\left(\sqrt{p\log(p)}L^4\tau^{1/3}\right).
\end{align*}
Thus by using this bound on $\lv \clubsuit\rv_{op}$ along with \eqref{e:ntk_decomposition_into_clubsuit_spadesuit} and \eqref{e:chi_1_bound_spadesuit} we get that
\begin{align*}
    \lv \chi_1\rv_{op} \le O(\tau L^2)+O\left(\sqrt{p\log(p)}L^4\tau^{1/3}\right) = O\left(\sqrt{p\log(p)}L^4\tau^{1/3}\right).
\end{align*}
As mentioned above we can also bound $\lv \chi_2 \rv_{op}$ using the exact same logic to get that 
\begin{align*}
    \lv \chi_2 \rv_{op}\le O\left(\sqrt{p\log(p)}L^4\tau^{1/3}\right).
\end{align*}
Thus, the decomposition in \eqref{e:ntk_difference_of_products_decomposition} along with an application of the triangle inequality proves the claim of the lemma.
\end{proof}

\subsection{The Proof} \label{ss:proof_of_lemma_ntk_approx}
With these various lemmas in place we are now finally ready to prove Lemma~\ref{l:ntk_app_error}

\ntkapprox*

\begin{proof}Note that since $\tau=\Omega\left(\frac{\log^2\left(\frac{nL}{\delta} \right)}{p^{\frac{3}{2}}L^3}\right)$ and $\tau \le \frac{c}{L^{12} \log^{\frac{3}{2}}(p)}$, $h \le \frac{\tau}{\sqrt{p}}\le \frac{1}{50 \sqrt{p}L}$, and because $p \ge \poly\left(L,\log\left(\frac{n}{\delta}\right)\right)$ for a large enough polynomial all the conditions required to invoke Lemma~\ref{l:conc_main_lemma} are satisfied. Let us assume that the event in Lemma~\ref{l:conc_main_lemma} which occurs with probability at least $1-\delta$ holds in the rest of this proof.

\paragraph{Proof of Part~(a):} Recall the definition of the approximation error
\begin{align*}
       \epsilon_{\mathsf{app}}(V^{(1)},\tau) := \sup_{s\in [n]}\sup_{\hV,\tV \in \cB(V^{(1)},\tau)}\left\lvert f_{\hV}(x_s)-f_{\tV}(x_s)-\nabla f_{\tV}(x_s)\cdot\left(\hV-\tV \right)\right\rvert.
\end{align*} 
Fix a $\hV,\tV \in  \cB(V^{(1)},\tau)$ and a sample $s \in [n]$. To ease notation denote $\Sigma_{\ell,s}^{\hV}$ by $\hSigma_{\ell}$, $\Sigma_{\ell,s}^{\tV}$ by $\tSigma_{\ell}$, $x_{\ell,s}^{\hV}$ by $\hx_{\ell}$, $x_{\ell,s}^{\tV}$ by $\tx_{\ell}$ and $x_{\ell,s}^{V^{(1)}}$ by $x_{\ell,s}$. We know that $f_{\tV}(x_s) = \tV_{L+1}\tx_{L}$ and $f_{\hV}(x_s) = \hV_{L+1}\hx_{L}$. Also since $\nabla_{\hV_{L+1}}f_{\hV}(x_s) = \hx_{L}$ we have 
\begin{align*}
   & f_{\hV}(x_s)-f_{\tV}(x_s)-\nabla f_{\tV}(x_s)\cdot\left(\hV-\tV \right)\\ 
   & \qquad = \hV_{L+1}\hx_{L}- \tV_{L+1}\tx_{L}
   -
   (\hV_{L+1}-\tV_{L+1})\cdot \tx_{L}  
   -
   \sum_{\ell=1}^L  \nabla_{\tV_{\ell}}f_{\tV}(x_s)\cdot (\hV_{\ell}-\tV_{\ell}) \\
   &\qquad  = \hV_{L+1}(\hx_{L}-\tx_{L})
   -
   \sum_{\ell=1}^L  \nabla_{\tV_{\ell}}f_{\tV}(x_s)\cdot (\hV_{\ell}-\tV_{\ell}) .  \numberthis \label{e:function_difference_in_terms_of_x_s}
\end{align*}

By equation~\eqref{e:ntk_x_difference_decomposition} from the proof of Lemma~\ref{l:ntk_x_perturbation_in_a_ball} above
we can decompose the difference as follows,
\begin{align*}
\hx_{L} - \tx_{L}  & = \sum_{\ell=1}^{L-1} \left[\prod_{j=L}^{\ell+1} \left( \hSigma_{j}+\cSigma_{j}\right)\hV_j\right](\hSigma_{\ell}+\cSigma_{\ell})(\hV_\ell-\tV_\ell) \tx_{\ell-1} \\ 
& \qquad +\left( \hSigma_{L}+\cSigma_{L}\right)\left(  \hV_{L}-\tV_{L}\right)\tx_{\ell-1}   +\sum_{\ell=1}^{L-1} \left[\prod_{j=L}^{\ell+1}(\hSigma_{j}+\cSigma_{j})\hV_j\right]\chi_{\ell}+\chi_{L},    \label{e:difference_of_x_s_formula} \numberthis 
\end{align*}
where the diagonal matrix $\cSigma_{j,s}$ is $O(pL^4\tau^{2/3})$-sparse and has entries in $[-1,1]$, and the $p$-dimensional vectors $\chi_{\ell}$ have infinity norm at most $h$ (see inequality~\eqref{e:ntk_chi_huberized_relu_infinity_norm_bound}).
Now when $\ell \in [L]$, the formula for the gradient given in \eqref{e:gradient_ell_inner_layers}, using this formula and because given two 
matrices
$A$ and $B$, $A\cdot B=\Tr(A^{\top}B)$ we get
\begin{align*}
    \nabla_{\tV_{\ell}} f_{\tV}(x_s) \cdot \left(\hV_{\ell}-\tV_{\ell}\right) &= \Tr\left[\nabla_{\tV_{\ell}} f_{\tV}(x_s)^{\top}\left(\hV_{\ell}-\tV_{\ell}\right)\right]\\
    &= \Tr\left[  \left(\left( \tSigma_{\ell} \prod_{j = \ell+1}^L 
              \tV_{j}^{\top} \tSigma_{j} 
           \right)
           \tV_{L+1}^{\top}
        \tx_{\ell-1}^{\top}\right)^{\top} \left(\hV_{\ell}-\tV_{\ell}\right)\right] \\
        & = \Tr\left[ \tx_{\ell-1}\tV_{L+1} \left(  \prod_{j =L}^{ \ell+1}
              \tSigma_{j} \tV_{j}
           \right)\tSigma_{\ell}
         \left(\hV_{\ell}-\tV_{\ell}\right)\right]\\
         & =  \tV_{L+1} \left(  \prod_{j =L}^{ \ell+1}
              \tSigma_{j,s} \tV_{j}
           \right)\tSigma_{\ell,s}
         \left(\hV_{\ell}-\tV_{\ell}\right)\tx_{\ell-1}. \numberthis \label{e:dot_product_gradient_times_difference}
\end{align*}
Using \eqref{e:function_difference_in_terms_of_x_s}-\eqref{e:dot_product_gradient_times_difference},
and noting that, here $\prod_{j=L}^{\ell+1} A_j$ 
denotes
$A_L A_{L-1}\ldots A_{\ell+1}$, i.e. 
the indices ``count down'',
we find
%
\begin{align*}
    &f_{\hV}(x_s)-f_{\tV}(x_s)-\nabla f_{\tV}(x_s)\cdot\left(\hV-\tV \right) \\
    &\overset{(i)}{=} \sum_{\ell=1}^{L} \left(\hV_{L+1}\left[\prod_{j=L}^{\ell+1} \left( \hSigma_{j}+\cSigma_{j}\right)\tV_j\right](\hSigma_{\ell}+\cSigma_{\ell})(\hV_\ell-\tV_\ell) \tx_{\ell-1} \right.\\&\qquad\qquad\qquad  \left.-\tV_{L+1} \left(  \prod_{j =L}^{ \ell+1}
              \tSigma_{j} \tV_{j}
           \right)\tSigma_{\ell}
         \left(\hV_{\ell}-\tV_{\ell}\right)\tx_{\ell-1}\right)+\sum_{\ell=1}^{L}
         \hV_{L+1}
         \left[\prod_{j=L}^{\ell+1}(\hSigma_{j}+\cSigma_{j})\hV_j\right]\chi_{\ell}  \\
         &= \sum_{\ell=1}^{L} \left(\hV_{L+1}\left[\prod_{j=L}^{\ell+1} \left( \hSigma_{j}+\cSigma_{j}\right)\hV_j\right](\hSigma_{\ell}+\cSigma_{\ell}) -\tV_{L+1} \left(  \prod_{j =L}^{ \ell+1}
              \tSigma_{j} \tV_{j}
           \right)\tSigma_{\ell}
         \right)(\hV_\ell-\tV_\ell) \tx_{\ell-1}\\
         & \qquad\qquad\qquad +\sum_{\ell=1}^{L}
          \hV_{L+1}
          \left[\prod_{j=L}^{\ell+1}(\hSigma_{j}+\cSigma_{j})\hV_j\right]\chi_{\ell}  \\
         &= \sum_{\ell=1}^{L} \underbrace{\left(\hV_{L+1}\left[\prod_{j=L}^{\ell+1} \left( \hSigma_{j}+\cSigma_{j}\right)\hV_j\right] -\tV_{L+1} \left(  \prod_{j =L}^{ \ell+1}
              \tSigma_{j} \tV_{j}
           \right)
         \right)\tSigma_{\ell}(\hV_\ell-\tV_\ell) \tx_{\ell-1}}_{=:\spadesuit_\ell}\\
         & \qquad\qquad\qquad + \sum_{\ell=1}^{L} \underbrace{\hV_{L+1}\left[\prod_{j=L}^{\ell+1} \left( \hSigma_{j}+\cSigma_{j}\right)\hV_j\right] 
         \left(\hSigma_{\ell}+\cSigma_{\ell}-\tSigma_{\ell}\right)(\hV_\ell-\tV_\ell) \tx_{\ell-1}}_{=: \clubsuit_\ell}\\
         & \qquad\qquad\qquad +\sum_{\ell=1}^{L} \underbrace{\hV_{L+1} \left[\prod_{j=L}^{\ell+1}(\hSigma_{j}+\cSigma_{j})\hV_j\right]\chi_{\ell}}_{=:\varheart_\ell} \numberthis \label{e:ntk_approx_error_decomp_into_suits}
\end{align*}
where in $(i)$,
we adopt the convention that when $\ell=L$,
the ``empty products'' 
$\prod_{j=L}^{\ell+1}\left( \hSigma_{j}+\cSigma_{j}\right)\hV_j$
and $\prod_{j =L}^{ \ell+1}
              \tSigma_{j} \tV_{j}$
are interpreted as $I$. Let us bound the norm of $\spadesuit_{\ell}$ in the case where $\ell \neq L$ (the bound in the boundary case when $\ell = L$ follows by exactly the same logic):
\begin{align*}
    \lv \spadesuit_{\ell}\rv & = \left\lv \left(\hV_{L+1}\left[\prod_{j=L}^{\ell+1} \left( \hSigma_{j}+\cSigma_{j}\right)\hV_j\right] -\tV_{L+1} \left(  \prod_{j =L}^{ \ell+1}
              \tSigma_{j} \tV_{j}
           \right)
        \right) \tSigma_{\ell}(\hV_\ell-\tV_\ell) \tx_{\ell-1}\right\rv \\
         &\le \left\lv \hV_{L+1}\left[\prod_{j=L}^{\ell+1} \left( \hSigma_{j}+\cSigma_{j}\right)\hV_j\right] -\tV_{L+1} \left(  \prod_{j =L}^{ \ell+1}
              \tSigma_{j} \tV_{j}
           \right)
         \right\rv_{op}\left\rv\tSigma_{\ell}(\hV_\ell-\tV_\ell) \tx_{\ell-1}\right\rv\\
         &\overset{(i)}{\le} O(\sqrt{p \log(p)}L^{4}\tau^{1/3}) \left\rv\tSigma_{\ell}\right\rv_{op}\lv \hV_\ell-\tV_\ell\rv \lv \tx_{\ell-1}\rv \\
         &\overset{(ii)}{\le} O(\sqrt{p \log(p)}L^{4}\tau^{4/3})  \left(\lv \tx_{\ell-1}-x_{\ell-1}\rv +\lv x_{\ell-1}\rv\right)\\
         &\overset{(iii)}{\le} O(\sqrt{p \log(p)}L^{4}\tau^{4/3})  \left(2+O(L^3\tau )\right)\le O(\sqrt{p \log(p)}L^{4}\tau^{4/3}) \label{e:spadesuit_bound_ntk_approx_error} \numberthis
\end{align*}
where $(i)$ follows by invoking Lemma~\ref{l:ntk_diff_between_prod_of_matrices}, $(ii)$ is because the entries of $\tSigma_{\ell}$ lie between $0$ and $1$ and because $\lv \hV_{\ell}-\tV_{\ell}\rv \le 2\tau$ since both $\hV$ and $\tV$ are in $\cB(V^{(1)},\tau)$. Inequality~$(iii)$ is because $\lv \tx_{\ell-1}-x_{\ell-1} \rv\le O(L^3 \tau)$ by Lemma~\ref{l:ntk_x_perturbation_in_a_ball} and $\lv x_{\ell-1}\rv\le 2$ by Part~(a) of Lemma~\ref{l:conc_main_lemma}.

Moving on to $\clubsuit_{\ell}$ (again consider the case where $\ell \neq L$, the bound in the boundary case when $\ell = L$ follows by exactly the same logic),
\begin{align*}
    \lv \clubsuit_{\ell}\rv & = \left\lv \hV_{L+1}\left[\prod_{j=L}^{\ell+1} \left( \hSigma_{j}+\cSigma_{j}\right)\hV_j\right] 
         \left(\hSigma_{\ell}+\cSigma_{\ell}-\tSigma_{\ell}\right)(\hV_\ell-\tV_\ell) \tx_{\ell-1}\right\rv \\
         & \le \left\lv \hV_{L+1}\left[\prod_{j=L}^{\ell+1} \left( \hSigma_{j}+\cSigma_{j}\right)\hV_j\right] \right\rv_{op}
    \left\lv     \hSigma_{\ell}+\cSigma_{\ell}-\tSigma_{\ell} \right\rv_{op}\left\lv \hV_\ell-\tV_\ell\right\rv \left\lv \tx_{\ell-1}\right\rv \\
    &\overset{(i)}{\le} O(\sqrt{p} L^2) \times \tau \times \left(2+O(L^3\tau )\right) = O(\sqrt{p}\tau  L^2)\label{e:clubsuit_bound_ntk_approx_error} \numberthis
\end{align*}
where $(i)$ follows
by invoking Lemma~\ref{l:ntk_product_matrices_in_a_ball}, since the diagonal matrix $\hSigma_{\ell}+\cSigma_{\ell}-\tSigma_{\ell}$ have entries between $-3$ and $3$ and by bounding $\left\lv \tx_{\ell-1}\right\rv$ as we did above. Finally, we bound the norm of $\varheart_{\ell}$ (again in the case where $\ell \neq L$, the bound when $\ell = L$ follows by exactly the same logic)
\begin{align*}
    \lv \varheart_{\ell}\rv & = \left\lv
    \hV_{L+1} \prod_{j=L}^{\ell+1}(\hSigma_{j}+\cSigma_{j})\hV_j\chi_{\ell}\right\rv \\
    &\le  \left\lv
    \hV_{L+1} \prod_{j=L}^{\ell+1}(\hSigma_{j}+\cSigma_{j})\hV_j\right\rv \lv\chi_{\ell}\rv \\
    &\overset{(i)}{\le} O(\sqrt{p} L^2) \lv\chi_{\ell}\rv \le O(\sqrt{p} L^2) \sqrt{p}\lv\chi_{\ell}\rv_{\infty} \overset{(ii)}{\le} O(\sqrt{p} L^2)\sqrt{p}h \overset{(iii)}{\le} O( \sqrt{p}\tau  L^2)  \numberthis \label{e:heartsuit_bound_ntk_approx_error}
\end{align*}
where $(i)$ is by invoking Lemma~\ref{l:ntk_product_matrices_in_a_ball}, $(ii)$ is due to a bound on the $\lv \chi_{\ell}\rv_{\infty}\le h$ derived in inequality~\ref{e:ntk_chi_huberized_relu_infinity_norm_bound} and $(iii)$ is by the assumption that $h <\frac{\tau}{\sqrt{p}}$. The bounds on the norms of $\spadesuit_{\ell}, \clubsuit_{\ell}$ and $\varheart_{\ell}$ along with the decomposition in \eqref{e:ntk_approx_error_decomp_into_suits} reveals that for any $s \in [n]$, $\hV,\tV \in \cB(V^{(1)},\tau)$:
\begin{align*}
    \mleft\lvert f_{\hV}(x_s)-f_{\tV}(x_s)-\nabla f_{\tV}(x_s)\cdot\left(\hV-\tV \right) \mright\rvert &\le L\left(O(\sqrt{p \log(p)}L^{4}\tau^{4/3})+O(\sqrt{p}\tau  L^2) \right)\\& \le O(\sqrt{p \log(p)}L^{5}\tau^{4/3}).
\end{align*}
This completes the proof of the first part.

\paragraph{Proof of Part~(b):}Recall the definition of $\Gamma(V^{(1)},\tau)$
\begin{align*}
    \Gamma(V^{(1)},\tau) = \sup_{s\in [n]}\sup_{\ell \in [L+1]}\sup_{V \in \cB(V^{(1)},\tau)}\lv \nabla_{V_{\ell}}f_{V}(x_s)\rv.
\end{align*}
Fix a sample $s\in [n]$. First let us bound the Frobenius norm of the gradient when $\ell \in [L]$. By the formula in \eqref{e:gradient_ell_inner_layers} we have
\begin{align*}
    \lv \nabla_{V_{\ell}} f_V(x_s)\rv & = \left\lv \left( \Sigma^V_{\ell,s} \prod_{j = \ell+1}^L \left(  V_{j}^{\top} \Sigma^V_{j,s} \right)\right)V_{L+1}^{\top} x_{\ell-1,s}^{V\top} \right\rv \\
    & \le \left\lv \left( \Sigma^V_{\ell,s} \prod_{j = \ell+1}^L \left(  V_{j}^{\top} \Sigma^V_{j,s} \right)\right)V_{L+1}^{\top}\right\rv_{op} \lv x_{\ell-1,s}^{V\top}\rv  \\
    &\le \left\lv \Sigma^V_{\ell,s} \right\rv_{op} \left\lv \prod_{j = \ell+1}^L V_{j}^{\top} \Sigma^V_{j,s}  \right\rv_{op}\lv V_{L+1}\rv \lv x_{\ell-1,s}\rv\\
    &\le  \left\lv \prod_{j = \ell+1}^L   V_{j}^{\top} \Sigma^V_{j,s}  \right\rv_{op}\lv V_{L+1}\rv \lv x_{\ell-1,s}\rv \qquad \mbox{(since $\lv \Sigma_{\ell,s}^{V}\rv_{op}\le 1$)}\\
    &\le  \left\lv \prod_{j = \ell+1}^L   V_{j}^{\top} \Sigma^V_{j,s}  \right\rv_{op}\left(\lv V_{L+1}^{(1)}\rv+\lv V_{L+1}^{(1)}-V_{L+1}\rv\right) \left(\lv x_{\ell-1,s}^{V^{(1)}}\rv+\lv x_{\ell-1,s}^V-x_{\ell-1,s}^{V^{(1)}}\rv\right)\\
    &\overset{(i)}\le  \left\lv \prod_{j = \ell+1}^L   V_{j}^{\top} \Sigma^V_{j,s}  \right\rv_{op}\left(O(\sqrt{p})+\tau\right) \left(2+O(L^3 \tau)\right)\\
    &\overset{(ii)}{\le} O(\sqrt{p})  \left\lv \prod_{j = \ell+1}^L   V_{j}^{\top} \Sigma^V_{j,s}  \right\rv_{op} \overset{(iii)}{\le} O(\sqrt{p}L^2)
     \label{e:operator_norm_gradient_bound} \numberthis
\end{align*}
where $(i)$ follows since $\lv V_{L+1}^{(1)}\rv \le O(\sqrt{p})$ by Part~(b) of Lemma~\ref{l:conc_main_lemma}, $\lv V_{L+1}^{(1)}-V_{L+1}\rv\le \tau$, $\lv x_{\ell-1,s}^{V^{(1)}}\rv \le 2$ by Part~(a) of Lemma~\ref{l:conc_main_lemma} and $\lv x_{\ell-1,s}^V-x_{\ell-1,s}^{V^{(1)}}\rv\le O(L^3 \tau)$ by Lemma~\ref{l:ntk_x_perturbation_in_a_ball}. Next $(ii)$ follows since $\tau =O(1/L^3)$. Finally, $(iii)$ follows since the matrix $\Sigma_{j,s}^V - \Sigma_{j,s}^{V^{(1)}}$ is $O(pL^2 \tau^{2/3})$ sparse by Lemma~\ref{l:ntk_x_perturbation_in_a_ball}, therefore we can apply Lemma~\ref{l:ntk_product_matrices_in_a_ball} to bound the operator norm of the product of the matrices (since $\tau = O\left(\frac{1}{L^{12}\log^{\frac{3}{2}}(p)}\right)$, that Lemma applies that this level of sparsity).

If $\ell = L+1$, then the gradient at $V$ is $x_{L,s}^{V}$, therefore
\begin{align*}
\sup_{s\in [n]}\sup_{V \in \cB(V^{(1)},\tau)}\lv \nabla_{V_{L+1}}f_{V}(x_s)\rv  &= \sup_{s\in [n]}\sup_{V \in \cB(V^{(1)},\tau)}\lv x_{L,s}^{V} \rv \\
& \le \sup_{s\in [n]}\sup_{V \in \cB(V^{(1)},\tau)}\left(\lv x_{L,s}^{V^{(1)}} \rv +\lv x_{L,s}^{V^{(1)}}-x_{L,s}^{V} \rv\right)\\
& \le \sup_{s\in [n]}\sup_{V \in \cB(V^{(1)},\tau)}\left(2 +O(L^3\tau)\right) \le O(1),
\end{align*}
where above we used the fact that $\lv x_{L,s}^{V^{(1)}} \rv \le 2$ by Part~(a) of Lemma~\ref{l:conc_main_lemma} and $\lv x_{L,s}^{V^{(1)}}-x_{L,s}^{V} \rv \le O(L^3 \tau)$ by Lemma~\ref{l:ntk_x_perturbation_in_a_ball} along with the fact that $\tau \le O(1/L^3)$. Combining the conclusions in the two cases when $\ell \in [L]$ and $\ell \in [L+1]$ establishes our second claim.
\end{proof}

Now that we have proved Lemma~\ref{l:ntk_app_error}, the reader can next jump to Appendix~\ref{ss:neural_tangent_tech_tools}.

\section{Probabilistic Tools} \label{s:probability_tools}

For an excellent reference of sub-Gaussian and sub-exponential concentration inequalities we refer the reader to \citet{vershynin2018high}. We begin by defining sub-Gaussian and sub-exponential random variables.

\begin{definition} \label{def:subgaussian}A random variable $\theta$ is sub-Gaussian if 
\begin{align*}
\lv \theta \rv_{\psi_2}:= \inf\left\{t>0: \mathbb{E}[\exp(\theta^2/t^2)]< 2\right\}
\end{align*}
is bounded. Further, $\lv \theta\rv_{\psi_2}$ is defined to be its sub-Gaussian norm.
\end{definition}

\begin{definition} \label{def:subexp}A random variable $\theta$ is 
said to be
sub-exponential if 
\begin{align*}
\lv \theta \rv_{\psi_1}:= \inf\left\{t>0: \mathbb{E}[\exp(\lvert \theta\rvert/t)< 2]\right\}
\end{align*}
is bounded. Further, $\lv \theta\rv_{\psi_1}$ is defined to be its sub-exponential norm.
\end{definition}
Next we state a few well-known facts about sub-Gaussian random variables.
\begin{lemma}(Lemma~2.7.6, Vershynin 2018)
\label{l:sub_gaussian_squared}If a random variable $\theta$ is sub-Gaussian then $\theta^2$ is sub-exponential with $\lv \theta^2\rv_{\psi_1} = \lv \theta \rv_{\psi_2}^2$.
\end{lemma}
\begin{lemma}(Lemma~5.2.2, Vershynin 2018)
\label{l:sub_gaussian_lipschitz}If a random variable $\theta\sim\cN(0,1)$ and $g$ is a $1$-Lipschitz function then $\lv g(\theta)-\E[g(\theta)]\rv_{\psi_{2}}\le c$, for some absolute positive constant $c$.
\end{lemma}
 Let us state Hoeffding's inequality \citep[see, e.g.,][Theorem~2.6.2]{vershynin2018high}, a concentration inequality for a sum of independent sub-Gaussian random variables.
\begin{theorem} \label{thm:hoeffding}
For independent mean-zero sub-Gaussian random variables $\theta_1,\ldots,\theta_m$, for every $\eta>0$, we have
\begin{align*}
\Pr\left[\Big\lvert \sum_{i=1}^m  \theta_i \Big\rvert \ge \eta\right]\le 2\exp\left(- \frac{c\eta^2}{\sum_{i=1}^m \lv \theta_i\rv_{\psi_2}^2}\right),
\end{align*}
where $c$ is
a positive absolute constant.
\end{theorem}
We shall also use Bernstein's inequality \citep[see, e.g.,][Theorem~2.8.1]{vershynin2018high} a concentration inequality for a sum of independent sub-exponential random variables.

\begin{theorem} \label{thm:bernstein}
For independent mean-zero sub-exponential random variables $\theta_1,\ldots,\theta_m$, for every $\eta>0$, we have
\begin{align*}
\Pr\left[\Big\lvert \sum_{i=1}^m  \theta_i \Big\rvert \ge \eta\right]\le 2\exp\left(-c \min\left\{\frac{\eta^2}{\sum_{i=1}^m \lv \theta_i\rv_{\psi_1}^2},\frac{\eta}{\max_i \lv \theta_i \rv_{\psi_1}}\right\}\right),
\end{align*}
where $c$ is
a positive absolute constant.
\end{theorem}

Next is the Gaussian-Lipschitz contraction inequality applied to control the squared norm of a Gaussian random vector \citep[see, e.g.,][Example~2.28]{wainwright2019high}. 
\begin{theorem} \label{thm:gaussian_concentration}
Let $\theta_1,\ldots,\theta_m$ be drawn i.i.d. from $\cN(0,\sigma^2)$ then, for every $\eta>0$, we have
\begin{align*}
\Pr\left[ \sum_{i=1}^m  \theta_i^2  \ge \sigma^2 m (1+\eta)^2\right]\le \exp\left(-c m\eta^2  \right),
\end{align*}
where $c$ is
a positive absolute constant.
\end{theorem}

Let us continue by defining an $\epsilon$-net with respect to the Euclidean distance.
\begin{definition} \label{def:epsilon_net}Let $S \subseteq \mathbb{R}^p$. A subset $K$ is called an $\epsilon$-net of $S$ if every point in $S$ is within a distance $\epsilon$ (in Euclidean distance) of some point in $K$. 
\end{definition}
The following lemma bounds the size of a $1/4$-net of unit vectors in $\mathbb{R}^p$.
\begin{lemma}\label{l:covering_numbers_unit_vectors} 
Let $S$ be the set of all unit vectors in $\mathbb{R}^p$. Then there exists a $1/4$-net of $S$ of size $9^{p}$.
\end{lemma}
\begin{proof}
Follows immediately by invoking \citep[Corollary~4.2.13]{vershynin2018high} with $\epsilon=1/4$.
\end{proof}
Here is a bound on the size of a $1/4$-net of $k$-sparse unit vectors,
along with a somewhat stronger property of the net.
\begin{lemma}\label{l:covering_numbers_sparse_vectors} 
Let $S$ be the set of all $k$-sparse unit vectors in $\mathbb{R}^p$. 
Then there exists a $1/4$-net 
$N$ of $S$ of size $\binom{p}{k}9^{k}$, and a mapping
$\zeta$ from $S$ to $N$ such that, for all $s \in S$, in
addition to $\lv s - \zeta(s) \rv \leq 1/4$, we have
$\lv s - \zeta(s) \rv_0 \leq k$.
\end{lemma}
\begin{proof}
We construct a $1/4$-net as follows. The number of distinct $k$-sparse subsets of $[p]$ are $\binom{p}{k}$. Over each of these distinct subsets build a $1/4$-net of unit vectors of size $9^{k}$, this is guaranteed by the preceding lemma. Thus by building a $1/4$-net for each of these subset and taking union of these nets we have built a $1/4$-net of $k$-sparse unit vectors of size $\binom{p}{k}9^k$ as claimed.
\end{proof}

\printbibliography

\end{document}

%% file: setting.tex
\usepackage{fullpage}
\usepackage{epsf}
\usepackage{fancyhdr}
\usepackage{graphics}
\usepackage{graphicx}
\usepackage{psfrag}
\usepackage[T1]{fontenc}
\usepackage{color}
\usepackage{amsthm}
\usepackage{amsfonts}
\usepackage{amsmath}
\usepackage{amssymb}
\usepackage{mathrsfs}
\usepackage{bm}
\usepackage{accents}
\usepackage{listings}
\usepackage{caption}
\usepackage{mleftright}
\usepackage{subcaption}
\usepackage[linesnumbered, ruled, vlined]{algorithm2e}
\usepackage[titletoc,toc,title]{appendix}
\usepackage{enumerate}
\usepackage{verbatim} 
\usepackage{csquotes} 
\usepackage{upquote} 
\usepackage[bottom]{footmisc} 
\usepackage{thmtools}
\usepackage{thm-restate}
\usepackage[dvipsnames]{xcolor}
\usepackage[english]{babel} 
\usepackage[shortlabels]{enumitem}
\usepackage{mathtools}
\usepackage[colorlinks,linkcolor = RawSienna, urlcolor  = RedViolet, citecolor = RoyalBlue, anchorcolor = ForestGreen,bookmarks=False]{hyperref}

\newlength{\widebarargwidth}
\newlength{\widebarargheight}
\newlength{\widebarargdepth}

\makeatletter
\long\def\@makecaption#1#2{
        \vskip 0.8ex
        \setbox\@tempboxa\hbox{\small {\bf #1:} #2}
        \parindent 1.5em  
        \dimen0=\hsize
        \advance\dimen0 by -3em
        \ifdim \wd\@tempboxa >\dimen0
                \hbox to \hsize{
                        \parindent 0em
                        \hfil 
                        \parbox{\dimen0}{\def\baselinestretch{0.96}\small
                                {\bf #1.} #2
                                } 
                        \hfil}
        \else \hbox to \hsize{\hfil \box\@tempboxa \hfil}
        \fi
        }
\makeatother

\makeatletter
\newcounter{manualsubequation}
\renewcommand{\themanualsubequation}{\alph{manualsubequation}}
\newcommand{\startsubequation}{%
  \setcounter{manualsubequation}{0}%
  \refstepcounter{equation}\ltx@label{manualsubeq\theequation}%
  \xdef\labelfor@subeq{manualsubeq\theequation}%
}
\newcommand{\tagsubequation}{%
  \stepcounter{manualsubequation}%
  \tag{\ref{\labelfor@subeq}\themanualsubequation}%
}
\let\subequationlabel\ltx@label
\makeatother

\renewcommand{\baselinestretch}{1.04} 
\date{}
\usepackage[style = alphabetic,citestyle=alphabetic,maxbibnames=99,backend=biber,sorting=nyt,natbib=true,backref=true]{biblatex}
\DefineBibliographyStrings{english}{%
  backrefpage = {Cited on page},
  backrefpages = {Cited on pages},
}
\newcommand{\BlackBox}{\rule{1.5ex}{1.5ex}}  

\renewenvironment{proof}{\par\noindent{\bf Proof\ }}{\hfill\BlackBox\\[2mm]}

\newtheorem{theorem}{Theorem}[section]
\newtheorem{lemma}[theorem]{Lemma}
\newtheorem{proposition}[theorem]{Proposition}

\newtheorem{definition}[theorem]{Definition}

\newcommand\numberthis{\addtocounter{equation}{1}\tag{\theequation}}

\newcommand{\sign}{\mathrm{sign}}
\renewcommand{\ss}{\subseteq}

\newcommand{\Tr}{\mathrm{Tr}}
\newcommand{\poly}{\mathrm{poly}}

\newcommand*{\eqdef}{:=}
\newcommand{\floor}[1]{{\left\lfloor #1 \right\rfloor}}
\newcommand{\ceil}[1]{{\left\lceil #1 \right\rceil}}

\newcommand{\cA}{{\cal A}}
\newcommand{\cB}{{\cal B}}

\newcommand{\cE}{{\cal E}}
\newcommand{\cF}{{\cal F}}

\newcommand{\cI}{{\cal I}}

\newcommand{\cN}{{\cal N}}

\newcommand{\cS}{{\cal S}}

\newcommand{\R}{\mathbb{R}}
\renewcommand{\S}{\mathbb{S}}
\newcommand{\N}{\mathbb{N}}

\newcommand{\diag}{\mathrm{diag}}
\newcommand{\vecrm}{\mathrm{vec}}
\newcommand{\E}{\mathbb{E}}

\renewcommand{\Pr}{\mathbb{P}}
\newcommand{\lv}{\lVert}
\newcommand{\rv}{\rVert}

\newcommand{\tx}{\tilde{x}}

\newcommand{\tW}{\widetilde{W}}
\newcommand{\tV}{\widetilde{V}}

\renewcommand{\epsilon}{\varepsilon}
\renewcommand{\ln}{\log}
\DeclareSymbolFont{extraup}{U}{zavm}{m}{n}
\DeclareMathSymbol{\varheart}{\mathalpha}{extraup}{86}
\DeclareMathSymbol{\vardiamond}{\mathalpha}{extraup}{87}

\DeclareMathOperator*{\argmin}{arg\,min}

\newcommand{\supp}{\mathrm{supp}}

\newcommand{\hV}{\hat{V}}
\newcommand{\hx}{\hat{x}}

\renewcommand{\epsilon}{\varepsilon}
\renewcommand{\ln}{\log}

\newcommand{\tSigma}{{\widetilde{\Sigma}}}
\newcommand{\hSigma}{{\hat{\Sigma}}}
\newcommand{\cSigma}{{\check{\Sigma}}}
\newcommand{\bSigma}{{\bar{\Sigma}}}

\newcommand{\Lip}{\mathrm{Lip}}